\documentclass[10pt,journal,compsoc]{IEEEtran}
\usepackage{graphicx}
\usepackage{color}
\usepackage{url}
\usepackage{amsthm}
\usepackage{amsmath,amssymb}
\usepackage{amsfonts}
\usepackage{diagbox}

\usepackage{algorithm}
\usepackage{algorithmic}
\newtheorem{assumption}{Assumption}
\newtheorem{property}{Property}
\newtheorem{problem}{Problem}
\newtheorem{theorem}{Theorem}
\newtheorem{lemma}{Lemma}
\usepackage{cleveref}
\newcommand{\tabincell}[2]{\begin{tabular}{@{}#1@{}}#2\end{tabular}}
\DeclareMathAlphabet\mathbfcal{OMS}{cmsy}{b}{n}
\begin{document}
%
\title{A Convergent ADMM Framework for Efficient Neural Network Training}

\author{Junxiang~Wang,
        Hongyi~Li,
        and Liang~Zhao~\IEEEmembership{Senior~Member,~IEEE}
\IEEEcompsocitemizethanks{

\IEEEcompsocthanksitem {Junxiang Wang (junxiang.wang@emory.edu) and Liang Zhao (liang.zhao@emory.edu) are from Emory University.}
\IEEEcompsocthanksitem {Hongyi Li (lihongyi@stu.xidian.edu.cn) is from Xidian University.}
}}

\markboth{Journal of \LaTeX\ Class Files,~Vol.~14, No.~8, August~2015}%
{Shell \MakeLowercase{\textit{et al.}}: Bare Demo of IEEEtran.cls for Computer Society Journals}

%

%

%
%

%
\IEEEtitleabstractindextext{
\begin{abstract}
As a well-known optimization framework, the Alternating Direction Method of Multipliers (ADMM) has achieved tremendous success in many classification and regression applications. Recently, it has attracted the attention of deep learning researchers and is considered to be a potential substitute to Gradient Descent (GD). However, as an emerging domain, several challenges remain unsolved, including 1) The lack of global convergence guarantees, 2) Slow convergence towards solutions, and 3) Cubic time complexity with regard to feature dimensions. In this paper, we propose a novel optimization framework to solve a general neural network training problem via ADMM (dlADMM) to address these challenges simultaneously. Specifically, the parameters in each layer are updated backward and then forward so that parameter information in each layer is exchanged efficiently. When the dlADMM is applied to specific architectures, the time complexity of subproblems is reduced from cubic to quadratic via a dedicated algorithm design utilizing quadratic approximations and backtracking techniques. Last but not least, we provide the first proof of convergence to a critical point sublinearly for an ADMM-type method (dlADMM) under mild conditions. Experiments on seven benchmark datasets demonstrate the convergence, efficiency, and effectiveness of our proposed dlADMM algorithm.
\end{abstract}

}
\maketitle
%
%

%
\maketitle
 \section{Introduction}
 \ \quad The last two decades have witnessed the rapid development of deep learning techniques. Deep learning models have numerous advantages over traditional machine learning models, one of which is rich expressiveness: deep learning methods consist of non-linear modules, and hence have a powerful capacity to express different levels of representations \cite{lecun2015deep}. Because deep learning methods have a wide range of large-scale applications ranging from computer vision to graph learning, they entail efficient and accurate optimizers to reach solutions within time limits.\\
 \indent Deep learning models are usually trained by the backpropagation algorithm, which is achieved by the Gradient Descent (GD) and many of its variants. They are state-of-the-art optimizers because of simplicity and efficiency. However,  many drawbacks of GD put up a barrier to its wide applications. For example, it suffers from the gradient vanishing problem, where the error signal diminishes as the gradient is backpropagated; As another example, it is also sensitive to the poor conditioning problem, namely, a small magnitude of input change can lead to a dramatic variation of the gradient. Recently, a well-known optimization framework, the Alternating Direction Method of Multipliers (ADMM) is being considered as an alternative to GD for training deep learning models. The principle of the ADMM is to partition a problem into multiple subproblems, each of which usually has an analytic solution. It has achieved great success in many conventional machine learning problems  \cite{boyd2011distributed}. The ADMM has many potential advantages when solving deep learning problems: it scales linearly as a model is trained in parallel across cores; it serves as a gradient-free optimizer and hence avoids the gradient vanishing and the poor conditioning problems \cite{taylor2016training}. \\
\indent Even though the ADMM seems promising, there are still several challenges at must be overcome: \textbf{1. The lack of global convergence guarantees.} Although many empirical experiments have shown that ADMM converges in deep learning applications, the underlying theory governing this convergence behavior remains mysterious. This is because a typical deep learning problem consists of a combination of linear and nonlinear mappings, causing optimization problems to be highly nonconvex. This means that traditional proof techniques cannot be directly applied. \textbf{2. Slow convergence towards solutions.} Although ADMM is a powerful optimization framework that can be applied to large-scale deep learning applications, it usually converges slowly to high accuracy, even for simple examples \cite{boyd2011distributed}. It is often the case that ADMM becomes trapped in a modest solution. \textbf{3. Cubic time complexity with regard to feature dimensions.} The implementation of the ADMM is very time-consuming for real-world datasets. Experiments conducted by Taylor et al. found that ADMM required more than 7000 cores to train a neural network with just 300 neurons \cite{taylor2016training}. This computational bottleneck mainly originates from the matrix inversion required to update the weight parameters. Computing an inverse matrix needs further subiterations, and its time complexity is approximately $O(n^3)$, where $n$ is a feature dimension \cite{boyd2011distributed}.\\
\indent In order to deal with these difficulties simultaneously,  in this paper we propose a novel optimization framework for a deep learning Alternating Direction Method of Multipliers (dlADMM) algorithm. For a general neural network training problem, our proposed dlADMM algorithm updates parameters first in a backward direction and then forwards. This update approach propagates parameter information across the whole network and accelerates the convergence empirically. Then we apply the dlADMM algorithm to two specific architectures, namely, Multi-Layer Perceptron (MLP) and Graph Convolutional Network (GCN). When updating parameters of the proposed dlADMM, we avoid the operation of matrix inversion using the quadratic approximation and backtracking techniques, reducing the time complexity from $O(n^3)$ to $O(n^2)$.  Finally,  to the best of our knowledge, we provide the first convergence proof of the ADMM-based method (dlADMM) to a critical point with a sublinear convergence rate. The assumption conditions are mild enough for many common loss functions (e.g. cross-entropy loss and square loss) and activation functions (e.g. Rectified Linear Unit (ReLU) and leaky ReLU) to satisfy. Our proposed framework and convergence proofs are highly flexible for MLP models and GCN models, as well as being easily extendable to other popular network architectures such as Convolutional Neural Networks \cite{krizhevsky2012imagenet} and Recurrent Neural Networks \cite{mikolov2010recurrent}. Our contributions in this paper include:
\begin{itemize}
\item We present a novel and efficient dlADMM algorithm to handle a general neural network training problem. The new dlADMM updates parameters in a backward-forward fashion to speed up convergence empirically. We apply the proposed dlADMM algorithm to two specific architectures, MLP and GCN models.
\item We propose the use of quadratic approximation and backtracking techniques to avoid the need for matrix inversion as well as to reduce the computational cost for large-scale datasets. The time complexity of subproblems in dlADMM is reduced from $O(n^3)$ to $O(n^2)$. 
\item  We investigate several attractive convergence properties of the proposed dlADMM. The convergence assumptions are very mild to ensure that most deep learning applications satisfy our assumptions. The proposed dlADMM is guaranteed to converge to a critical point whatever the initialization is, when the hyperparameter is sufficiently large. We also analyze the new algorithm's sublinear convergence rate.
\item We conduct extensive experiments on MLP models and GCN models on seven benchmark datasets to validate our proposed dlADMM algorithm. The results show that the proposed dlADMM algorithm not only is convergent and efficient, but also performs better than most existing state-of-the-art algorithms, including GD and its variants.
\end{itemize}
\ \quad The rest of this paper is organized as follows. In Section \ref{sec:related work}, we summarize recent research related to this topic. In Section \ref{sec:algorithm}, we present the new dlADMM algorithm, the quadratic approximation, and the backtracking techniques utilized. In Section \ref{sec:convergence}, we introduce the main convergence results for the dlADMM algorithm. The results of extensive experiments conducted to show the convergence, efficiency, and effectiveness of our proposed new dlADMM algorithm are presented in Section \ref{sec:experiment}, and Section \ref{sec:conclusion} concludes this paper by summarizing the research.
 \section{Related Work}
 \label{sec:related work}
 \ \quad Previous literature related to this research includes gradient descent and alternating minimization for deep learning models, and ADMM for nonconvex problems.\\
 \indent \textbf{Gradient descent for deep learning models:} The GD algorithm and its variants play a dominant role in the research conducted by the deep learning optimization community. The famous back-propagation algorithm was firstly introduced by Rumelhart et al. to train the neural network effectively \cite{rumelhart1986learning}. Since the superior performance exhibited by AlexNet \cite{krizhevsky2012imagenet} in 2012, deep learning has attracted a great deal of researchers' attention and many new optimizers based on GD have been proposed to accelerate the convergence process, including the use of Polyak momentum \cite{polyak1964some}, as well as research on the Nesterov momentum and initialization by Sutskever et al.  \cite{sutskever2013importance}. Adam is the most popular method because it is computationally efficient and requires little tuning \cite{kingma2014adam}. Other well-known methods that incorporate  adaptive learning rates include AdaGrad \cite{duchi2011adaptive}, RMSProp \cite{tielemandivide},  AMSGrad \cite{j.2018on}, AdaBound\cite{luo2018adaptive}, and generalized Adam-type optimizers \cite{chen2018on}. These adaptive optimizers, however, suffer from worse generalization than GD, as verified by various studies \cite{keskar2017improving, DBLP:conf/nips/WilsonRSSR17}. Many theoretical investigations have been conducted to explain such generalization gaps by exploring their convergence properties \cite{he2019asymmetric,izmailov2018averaging,keskar2017large,li2018visualizing,zhou2020towards}. Moreover, their convergence assumptions do not apply to deep learning problems, which often require non-differentiable activation functions such as the ReLU.\\
\indent\textbf{Alternating minimization for deep learning:} During the last decade, some alternating minimization methods have been studied to apply the Alternating Direction Method of Multipliers (ADMM) \cite{taylor2016training}, Block Coordinate Descent (BCD) \cite{pmlr-v97-zeng19a,choromanska2019beyond} and auxiliary coordinates (MAC) \cite{carreira2014distributed} to replace a nested neural network with a constrained problem without nesting.  These methods also avoid gradient vanishing problems and allow for non-differentiable activation functions such as ReLU as well as allowing for complex non-smooth regularization and the constraints that are increasingly important for deep neural architectures. More follow-up works have extended previous works to specific architectures \cite{lu2021training,tangADMMiRNN}, or to achieve parallel or efficient computation   \cite{guan2021pdladmm,khorram2021stochastic,qiao2021inertial,wang2020toward}.\\
\indent\textbf{ADMM for nonconvex problems}: The excellent performance achieved by ADMM over a range-wide of convex problems has attracted the attention of many researchers, who have now begun to investigate the behavior of ADMM on nonconvex problems and made significant advances. For example, Wang et al. proposed an ADMM to solve multi-convex problems with a convergence guarantee \cite{wang2019multi}, while Wang et al. presented convergence conditions for a coupled objective function that is nonconvex and nonsmooth \cite{wang2015global}. Chen et al. discussed the use of ADMM to solve problems with quadratic coupling terms \cite{chen2015extended} and Wang et al. studied the behavior of the ADMM for problems with nonlinear equality constraints \cite{wang2017nonconvex}. Other papers improved nonconvex ADMM via  Anderson acceleration \cite{ouyang2020anderson,zhang2019accelerating}, linearization \cite{liu2019linearized}, or overrelaxation \cite{themelis2020douglas}. Even though ADMM has been proposed to solve deep learning applications  \cite{taylor2016training,gao2019incomplete}, there remains a lack of theoretical convergence analysis for the application of ADMM to such problems.
 \section{The dlADMM Algorithm}
  \label{sec:algorithm}
 \quad We present our proposed dlADMM algorithm in this section. Section \ref{sec:dladmm} provides a high overview of the proposed dlADMM algorithm to handle a general neural network problem. Sections \ref{sec:MLP application} and \ref{sec:gcn problem} apply the proposed dlADMM to Multi-Layer Perceptron (MLP) and Graph Convolutional Network (GCN) problems, respectively. In particular, the quadratic approximation and backtracking techniques used to solve subproblems are discussed in detail.
  \begin{table}[]
     \centering
     \begin{tabular}{cc}
     \hline
          Notations&Descriptions  \\
          \hline
          L&Number of Layers\\
          $B_l$& Parameters for the $l$-th layer.\\
          $V_l$& The output for the $l$-th layer.\\
          $g_l(B_l,V_l)$& The nonlinear function for the $l$-th layer.\\
          $Y$& The label vector.\\
          \hline
     \end{tabular}
     \caption{Important notations of a general neural network model.}
     \label{tab:general}
 \end{table}

 \subsection{The Generic dlADMM Algorithm}
 \label{sec:dladmm}
  \begin{table}
 \centering

 \begin{tabular}{cc}
 \hline
 Notations&Descriptions\\ \hline
 $L$& Number of layers.\\
 $W_l$& The weight matrix for the $l$-th layer.\\
 $b_l$& The intercept vector for the $l$-th layer.\\
 $z_l$&  The output of the linear mapping for the $l$-th layer.\\
 $f_l(z_l)$& The nonlinear activation function for the $l$-th layer.\\
 $a_l$& The output for the $l$-th layer.\\
 $x$& The input matrix of the neural network.\\
 $y$& The predefined label vector.\\
 $R(z_L,y)$& The risk function.\\
 $\Omega_l(W_l)$& The regularization term for the $l$-th layer.\\
 $n_l$& The number of neurons for the $l$-th layer.\\
\hline
 \end{tabular}
  \caption{Important notations of the MLP model.}
 \label{tab:notation}
 \end{table}
 \begin{figure}
    \centering
    \includegraphics[width=\linewidth]{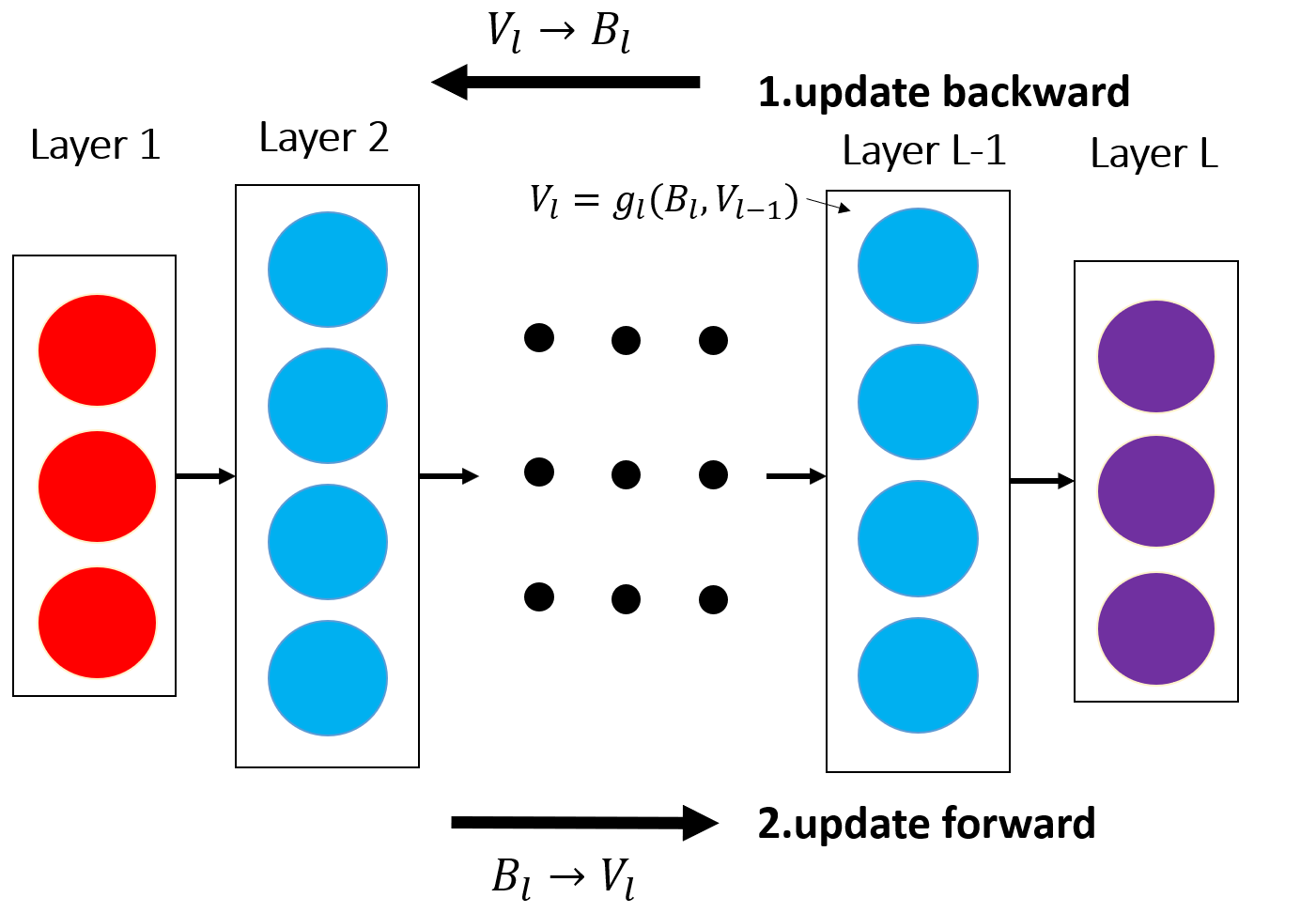}
    \caption{The dlADMM framework to train a general neural network: update parameter backward and then forward.}
    \label{fig:framwork overview}
\end{figure}
\indent In this section, we give a high-level overview of the proposed dlADMM algorithm to address a general neural network training problem. A general neural network consists of multiple layers, each of which is represented by a nonlinear function $V_l=g_l(B_l,V_{l-1})$, where $B_l$ is a weight matrix for the $l$-th layer, and $V_l$ is the output for the $l$-th layer, which is also the input of the $(l+1)$-th layer. The number of layers is $L$. For the final layer (i.e. the $L$-th layer), a fully-connected layer is represented as $V_L=B_LV_{L-1}$. $V_0$ is the input of a neural network. $Y$ is a label vector. Then a general neural network training problem is formulated as follows:\\
\begin{problem}[General Neural Network Training Problem]
\label{prob:general NN}
\begin{align*}
&\min\nolimits_{B_l,V_l} \mathbfcal{R}(V_L;Y)+\sum\nolimits_{l=1}^{L-1}\omega_l(B_l),\\
     &s.t. \ V_l=g_l(B_l,V_{l-1}) \ (l=1,\cdots,L-1), \quad V_L=B_L V_{L-1}.
\end{align*}
\end{problem}
where $\mathbfcal{R}(V_L;Y)$ is a loss function, $\omega_l(B_l)$ is a regularization term for the $l$-th layer, and $L$ is the number of layers. Solving Problem \ref{prob:general NN} directly via the ADMM algorithm is computationally intractable due to the introduction of nonlinear constraints $V_l=g_l(B_l,V_{l-1})$. By relaxing it via an $\ell_2$ penalty, we have the following problem:
\begin{problem}
\label{prob:relaxed general NN}
\begin{align*}
&\min\nolimits_{B_l,V_l} \mathbfcal{R}(V_L;Y)\!+\!\sum\nolimits_{l=1}^L\omega_l(B_l)\!+\!\Psi\sum\nolimits_{i\!=\!1}^{L\!-\!1} \Vert V_l\!-\!g_l(B_l,V_{l-1})\Vert^2_2,\\
     &\quad \quad s.t. \ V_L=B_LV_{L-1}.
\end{align*}
\end{problem}
where $\Psi>0$ is a tuning parameter. Compared with Problem \ref{prob:general NN}, Problem \ref{prob:relaxed general NN} has only a linear constraint $V_L=B_LV_{L-1}$ and hence is easier to solve.  It is straightforward to show that as $\Psi\rightarrow \infty$, the solution to Problem \ref{prob:general NN} approaches that of Problem \ref{prob:relaxed general NN}.\\
\indent Now we introduce the dlADMM algorithm to solve Problem \ref{prob:relaxed general NN}. The traditional ADMM strategy for optimizing parameters is to start from the first layer and then update parameters in the following layer sequentially \cite{taylor2016training}. In this case, the parameters in the final layer are subject to the parameter update in the first layer.  However, the parameters in the final layer contain important information that can be transmitted towards the previous layers to speed up convergence. To achieve this, we propose our novel dlADMM framework, as shown in Figure \ref{fig:framwork overview}. Specifically, the dlADMM algorithm updates parameters in two steps. In the first, the dlADMM begins updating from the $L$-th (final) layer and moves backward toward the first layer. The update order of parameters in the same layer is $V_l\rightarrow B_l$. In the second, the dlADMM reverses the update direction, beginning at the first layer and moving forward toward the $L$-th (final) layer. The update order of the parameters in the same layer is $B_l\rightarrow V_l$. The parameter information for all layers can be exchanged completely by adopting this update approach.\\
\indent The Augmented Lagrangian is formulated mathematically as follows:
\begin{align*}
    &\mathbfcal{L}_\rho(\textbf{B},\textbf{V},\mathbfcal{U})=\mathbfcal{R}(V_L;Y)+\sum\nolimits_{l=1}^L\omega_l(B_l)\\&+\Psi\sum\nolimits_{i=1}^{L-1} \Vert V_l-g_l(B_l,V_{l-1})\Vert^2_2+\mathbfcal{U}^T(V_L-B_LV_{L-1}))\\&+(\rho/2)\Vert V_L-B_LV_{L-1}\Vert^2_2.
\end{align*}
where $\rho>0$ is a hyperparameter, and $\mathbfcal{U}$ is a dual variable. Throughout the paper,  a variable with a bar over it denotes one in the backward update, and a variable without a bar denotes one in the forward update. Specifically, $\overline{B}^{k+1}_l$ and $\overline{V}^{k+1}_l$ denote backward updates of the dlADMM for the $l$-th layer in the $(k+1)$-th iteration, and $B^{k+1}_l$ and $V^{k+1}_l$ denote forward updates of the dlADMM for the $l$-th layer in the $(k+1)$-th iteration. $\textbf{B}^{k+1}=\{B^{k+1}_l\}_{l=1}^L$, $\textbf{V}^{k+1}=\{V^{k+1}_l\}_{l=1}^L$. Algorithm \ref{algo:dlADMM} shows the procedure of the proposed dlADMM. Specifically, Lines 4-5 and Lines 8-9 update variables backward and forward via solving subproblems, respectively. Lines 11 and 12 calculate the primal residual and the dual variable $U^{k+1}$, respectively. In the next two sections, we discuss how to apply the proposed dlADMM algorithm to two well-known architectures, and solve subproblems efficiently.
\begin{algorithm} 
\caption{the dlADMM Algorithm to Solve Problem \ref{prob:relaxed general NN}} 
\begin{algorithmic}[1] 
\scriptsize
\label{algo:dlADMM}
\REQUIRE $Y$, $V_0$, $\rho$, $\Psi$. 
\ENSURE $B_l(l=1,\cdots,L),V_l(l=1,\cdots,L)$. 
\STATE Initialize $k=0$.
\WHILE{$\textbf{B}^{k+1},\textbf{V}^{k+1}$ not converged}
\FOR{$l=L$ to $1$}
\STATE  Update $\overline{V}^{k\!+\!1}_l$ via \\$\arg\!\min_{V_l}\! \mathbfcal{L}(\!\{B^{k}_i\}_{i\!=\!1}^l\!,\!\{\overline{B}_i^{k\!+\!1}\}_{i\!=\!l\!+\!1}^L\!,\!\{V^{k}_i\}_{i\!=\!1}^{l\!-\!1},V_l,\{\overline{V}_i^{k\!+\!1}\}_{i\!=\!l\!+\!1}^L\!,\!\mathbfcal{U}^k\!)$.
\STATE  Update $\overline{B}^{k\!+\!1}_l$ via \\$\arg\min_{B_l} \mathbfcal{L}(\{B^{k}_i\}_{i\!=\!1}^{l\!-\!1}\!,B_l,\!\{\overline{B}_i^{k\!+\!1}\}_{i\!=\!l\!+\!1}^L\!,\!\{V^{k}_i\}_{i\!=\!1}^{l\!-\!1},\{\overline{V}_i^{k\!+\!1}\}_{i\!=\!l}^L\!,\!\mathbfcal{U}^k)$.
\ENDFOR
\FOR{$l=1$ to $L$}
\STATE  Update $B^{k\!+\!1}_l$ via \\$\arg\!\min_{B_l} \mathbfcal{L}(\{B^{k\!+\!1}_i\}_{i\!=\!1}^{l\!-\!1}\!,B_l,\!\{\overline{B}_i^{k\!+\!1}\}_{i\!=\!l\!+\!1}^L\!,\!\{V^{k\!+\!1}_i\}_{i\!=\!1}^{l\!-\!1},\{\overline{V}_i^{k\!+\!1}\}_{i\!=\!l}^L,\mathbfcal{U}^k)$.
\STATE  Update $V^{k\!+\!1}_l$ via \\$\!\arg\!\min_{V_l} \mathbfcal{L}(\{B^{k\!+\!1}_i\}_{i\!=\!1}^l\!,\!\{\overline{B}_i^{k\!+\!1}\}_{i\!=\!l\!+\!1}^L\!,\!\{V^{k\!+\!1}_i\}_{i\!=\!1}^{l\!-\!1},V_l,\{\overline{V}_i^{k\!+\!1}\}_{i\!=\!l\!+\!1}^L,\mathbfcal{U}^k)$.

\ENDFOR
\STATE{$\mathcal{E}^{k+1}\leftarrow V^{k+1}_L-B^{k+1}_LV^{k+1}_{L-1}$}.
\STATE{$\mathbfcal{U}^{k+1}\leftarrow \mathbfcal{U}^k+\rho  \mathcal{E}^{k+1}$}.
\STATE $k\leftarrow k+1$.
\ENDWHILE
\STATE Output $\textbf{B},\textbf{V}$.
\end{algorithmic}
\end{algorithm}

\subsection{The Application on the MLP  Problem}
\label{sec:MLP application}
 \ \quad In this section, we discuss how to apply the proposed dlADMM algorithm on the MLP problem. Table \ref{tab:notation} lists important notations of the MLP model. A typical MLP model is defined by multiple linear mappings and nonlinear activation functions. A linear mapping for the $l$-th layer is composed of a weight matrix $W_l\in \mathbb{R}^{n_l\times n_{l-1}}$ and an intercept vector $b_l\in \mathbb{R}^{n_l}$, where $n_l$ is the number of neurons for the $l$-th layer; a nonlinear mapping for the $l$-th layer is defined by a continuous activation function $f_l(\bullet)$. Given an input $a_{l-1}\in \mathbb{R}^{n_{l-1}}$ from the $(l-1)$-th layer, the $l$-th layer outputs $a_l=f_l(W_la_{l-1}+b_l)$. Obviously, $a_{l-1}$ is nested in $a_{l}=f_l(\bullet)$. By introducing an auxiliary variable $z_l$,  the task of training a MLP problem is formulated mathematically as follows:
\begin{problem}[MLP Training Problem]
\label{prob:problem 1}
\begin{align*}
     & \min\nolimits_{W_l,b_l,z_l,a_l} R(z_L;y)+\sum\nolimits_{l=1}^L\Omega_l(W_l),\\
     &s.t.z_l=W_la_{l-1}+b_l(l=1,\cdots,L),\\& a_l=f_l(z_l) \ (l=1,\cdots,L-1).
\end{align*}
\end{problem}
In Problem \ref{prob:problem 1}, $a_0=x\in\mathbb{R}^{n_0}$ is the input of the MLP model where $n_0$ is the number of feature dimensions, and $y$ is a predefined label vector. $R(z_L;y)$ is a risk function for the $L$-th layer, which is convex, continuous, and proper,  and $\Omega_l(W_l)$ is a regularization term for the $l$-th layer, which is also convex, continuous, and proper. Similar to the problem relaxation in the previous section,  Problem \ref{prob:problem 1} is relaxed to Problem \ref{prob:problem 2} as follows:
\begin{problem}
\label{prob:problem 2}
\begin{align*}
    &\min\nolimits_{W_l,b_l,z_l,a_l}F(\textbf{W},\textbf{b},\textbf{z},\textbf{a})=R(z_L;y)+\sum\nolimits_{l=1}^L\Omega_l(W_l)\\&+(\nu/2)\sum\nolimits_{l=1}^{L-1}(\Vert z_l-W_la_{l-1}-b_l\Vert^2_2+\Vert a_l-f_l(z_l)\Vert^2_2),\\ &s.t. \ z_L=W_La_{L-1}+b_L.
\end{align*}
\end{problem}
where $\textbf{W}=\{W_l\}_{l=1}^{L}$, $\textbf{b}=\{b_l\}_{l=1}^{L}$, $\textbf{z}=\{z_l\}_{l=1}^{L}$, $\textbf{a}=\{a_l\}_{l=1}^{L-1}$ and $\nu>0$ is a tuning parameter. The augmented Lagrangian function of Problem \ref{pro:property 2} is shown as follows:
\begin{align}
    \nonumber L_\rho(\textbf{W},\textbf{b},\textbf{z},\textbf{a},u)&=R(z_L;y)+\sum\nolimits_{l=1}^L\Omega_l(W_l)\\&+\phi(\textbf{W},\textbf{b},\textbf{z},\textbf{a},u).
    \label{eq: Lagrangian}
    \end{align}
where $\phi(\textbf{W},\textbf{b},\textbf{z},\textbf{a},u)=(\nu/2)\sum\nolimits_{l=1}^{L-1} (\Vert z_l-W_la_{l-1}-b_l\Vert^2_2+\Vert a_l-f_l(z_l)\Vert^2_2)+u^T(z_L-W_La_{L-1}-b_L)+(\rho/2)\Vert z_L-W_La_{L-1}-b_L\Vert^2_2$, $u$ is a dual variable and $\rho>0$ is a hyperparameter of the dlADMM algorithm.  $\overline{W}^{k+1}_l$, $\overline{b}^{k+1}_l$, $\overline{z}^{k+1}_l$ and $\overline{a}^{k+1}_l$ denote backward updates of the dlADMM for the $l$-th layer in the $(k+1)$-th iteration, while ${W}^{k+1}_l$, ${b}^{k+1}_l$, ${z}^{k+1}_l$ and ${a}^{k+1}_l$ denote forward updates of the dlADMM for the $l$-th layer in the $(k+1)$-th iteration. Moreover,  some notations are shown as follows for the sake of simplicity: $\overline{\textbf{W}}^{k+1}_l=\{\{{W}^{k}_i\}_{i=1}^{l-1},\{\overline{{W}}^{k+1}_i\}_{i=l}^L\}$,  $\overline{\textbf{b}}^{k+1}_l=\{\{{b}^{k}_i\}_{i=1}^{l-1},\{\overline{{b}}^{k+1}_i\}_{i=l}^L\}$,  $\overline{\textbf{z}}^{k+1}_l=\{\{{z}^{k}_i\}_{i=1}^{l-1},\{\overline{{z}}^{k+1}_i\}_{i=l}^L\}$,  $\overline{\textbf{a}}^{k+1}_l\!=\!\{\{{a}^{k}_i\}_{i=1}^{l-1},\{\overline{{a}}^{k+1}_i\}_{i=l}^{L-1}\}$, ${\textbf{W}}^{k\!+\!1}_l\!=\!\{\{{W}^{k\!+\!1}_i\}_{i\!=\!1}^{l},\{\overline{{W}}^{k\!+\!1}_i\}_{i=l\!+\!1}^{L}\}$, ${\textbf{b}}^{k\!+\!1}_l\!=\!\{\{{b}^{k\!+\!1}_i\}_{i\!=\!1}^{l},\{\overline{{b}}^{k\!+\!1}_i\}_{i\!=\!l\!+\!1}^{L}\}$, ${\textbf{z}}^{k+1}_l\!=\!\{\{{z}^{k+1}_i\}_{i=1}^{l},\{\overline{{z}}^{k+1}_i\}_{i=l+1}^{L}\}$, ${\textbf{a}}^{k+1}_l=\{\{{a}^{k+1}_i\}_{i=1}^{l},\{\overline{{a}}^{k+1}_i\}_{i=l+1}^{L-1}\}$,
$\overline{\textbf{W}}^{k+1}=\{\overline{W}_i^{k+1}\}_{i=1}^L$, $\overline{\textbf{b}}^{k+1}=\{\overline{b}_i^{k+1}\}_{i=1}^L$,
$\overline{\textbf{z}}^{k+1}=\{\overline{z}_i^{k+1}\}_{i=1}^L$,
$\overline{\textbf{a}}^{k+1}=\{\overline{a}_i^{k+1}\}_{i=1}^{L-1}$, $\textbf{W}^{k+1}=\{W_i^{k+1}\}_{i=1}^L$, $\textbf{b}^{k+1}=\{b_i^{k+1}\}_{i=1}^L$,
$\textbf{z}^{k+1}=\{z_i^{k+1}\}_{i=1}^L$, and
$\textbf{a}^{k+1}=\{a_i^{k+1}\}_{i=1}^{L-1}$. Then the dlADMM algorithm is shown in Algorithm \ref{algo:dlADMM mlp}. Specifically, Lines 5, 6, 10, 11, 14, 15, 17 and 18 solve eight subproblems, namely,  $\overline{a}^{k+1}_l$, $\overline{z}^{k+1}_l$, $\overline{b}^{k+1}_l$, $\overline{W}^{k+1}_l$, ${W}^{k+1}_l$, ${b}^{k+1}_l$, ${z}^{k+1}_l$ and ${a}^{k+1}_l$, respectively. Lines 21 and 22 update the residual $r^{k+1}$ and the dual variable $u^{k+1}$, respectively.
\begin{algorithm} 
\caption{the dlADMM Algorithm to Solve Problem \ref{prob:problem 2}} 
\begin{algorithmic}[1] 
\scriptsize
\label{algo:dlADMM mlp}
\REQUIRE $y$, $a_0=x$, $\rho$, $\nu$. 
\ENSURE $a_l(l=1,\cdots,L-1),W_l(l=1,\cdots,L),b_l(l=1,\cdots,L), z_l(l=1,\cdots,L)$. 
\STATE Initialize $k=0$.
\WHILE{$\textbf{W}^{k+1},\textbf{b}^{k+1},\textbf{z}^{k+1},\textbf{a}^{k+1}$ not converged}
\FOR{$l=L$ to $1$}
\IF{$l<L$}
\STATE Update $\overline{a}^{k+1}_l$ in Equation \eqref{eq:update overline a}.
\STATE Update $\overline{z}^{k+1}_l$ in Equation \eqref{eq:update overline z}.
\STATE Update $\overline{b}^{k+1}_l$ in Equation \eqref{eq:update overline b}.
\ELSE
\STATE{Update $\overline{z}^{k+1}_L$ in Equation \eqref{eq:update overline zl}}.
\STATE Update $\overline{b}^{k+1}_L$ in Equation \eqref{eq:update overline bl}.
\ENDIF
\STATE Update $\overline{W}^{k+1}_l$ in Equation \eqref{eq:update overline W}.
\ENDFOR
\FOR{$l=1$ to $L$}
\STATE Update ${W}^{k+1}_l$ in Equation \eqref{eq:update W}.
\IF{$l< L$}
\STATE Update $b^{k+1}_l$ in Equation \eqref{eq:update b}.
\STATE Update $z^{k+1}_l$ in Equation \eqref{eq:update z}.
\STATE Update $a^{k+1}_l$ in Equation \eqref{eq:update a}.
\ELSE
\STATE Update $b^{k+1}_L$ in Equation \eqref{eq:update bl}.
\STATE{ Update $z^{k+1}_L$ in Equation \eqref{eq:update zl}.}
\STATE{$r^{k+1}\leftarrow z_L^{k+1}-W^{k+1}_L a^{k+1}_{L-1}-b^{k+1}_L$}.
 \STATE{$u^{k+1}\leftarrow u^k+\rho  r^{k+1}$}.
\ENDIF
\ENDFOR
\STATE $k\leftarrow k+1$.
\ENDWHILE
\STATE Output $\textbf{W},\textbf{b},\textbf{z},\textbf{a}$.
\end{algorithmic}
\end{algorithm}
\\ \quad The eight subproblems in Algorithm \ref{algo:dlADMM mlp} are discussed in detail in the following. Most can be solved by the quadratic approximation and the backtracking techniques discussed below, so the matrix inversion can be avoided.\\ \\
\textbf{1. Update $\overline{a}^{k+1}_l$}\\
\indent The variables $\overline{a}^{k+1}_l(l=1,\cdots,L-1)$ are updated as follows:
\begin{align*}
    \overline{a}^{k+1}_{l}&\leftarrow \arg\min\nolimits_{a_{l}} L_\rho(\overline{\textbf{W}}^{k+1}_{l+1},\overline{\textbf{b}}^{k+1}_{l+1},\overline{\textbf{z}}^{k+1}_{l+1},\{a^k_i\}_{i=1}^{l-1},a_l,\\&\{\overline{a}^{k+1}_i\}_{i=l+1}^{L-1},u^k).
\end{align*}
\indent The subproblem is transformed into the following form after it is replaced by Equation \eqref{eq: Lagrangian}. 
\begin{align}
    \nonumber\overline{a}^{k+1}_{l}&\leftarrow \arg\min\nolimits_{a_{l}} \phi(\overline{\textbf{W}}^{k+1}_{l+1},\overline{\textbf{b}}^{k+1}_{l+1},\overline{\textbf{z}}^{k+1}_{l+1},\{a^k_i\}_{i=1}^{l-1},a_l,\\&\{\overline{a}^{k+1}_i\}_{i=l+1}^{L-1},u^k)
    \label{eq:original overline a}.
\end{align}
Because $a_l$ and $W_{l+1}$ are coupled in $\phi(\bullet)$, in order to solve this problem, we must compute the inverse matrix of $\overline{W}^{k+1}_{l+1}$, which involves subiterations and is computationally expensive \cite{taylor2016training}. In order to handle this challenge,  we define $\overline{Q}_l(a_l;\overline{\tau}^{k+1}_l)$ as a quadratic approximation of $\phi$ at $a^{k}_{l}$, which is mathematically reformulated as follows:
\begin{align*}\overline{Q}_l(a_{l};\overline{\tau}^{k+1}_l)&=\phi(\overline{\textbf{W}}^{k+1}_{l+1},\overline{\textbf{b}}^{k+1}_{l+1},\overline{\textbf{z}}^{k+1}_{l+1},\overline{\textbf{a}}^{k+1}_{l+1},u^k)\\&+(\nabla _{a^k_{l}}\phi)^T(\overline{\textbf{W}}^{k+1}_{l+1},\overline{\textbf{b}}^{k+1}_{l+1},\overline{\textbf{z}}^{k+1}_{l+1},\overline{\textbf{a}}^{k+1}_{l+1},u^k)(a_{l}-a^k_{l})\\&+\Vert\overline{\tau}^{k+1}_l\circ (a_{l}-a^k_{l})^{\circ 2}\Vert_{1}/2.
\end{align*}
where $\overline{\tau}^{k+1}_l>0$ is a  parameter vector, $\circ$ denotes the Hadamard product (the elementwise product), and $a^{\circ b}$  denotes $a$ to the Hadamard power of $b$ and $\Vert\bullet\Vert_{1}$ is the $\ell_{1}$ norm. $\nabla _{a^k_{l}}\phi$ is the gradient of $\overline{a}_l$ at $a^k_l$. Obviously, $\overline{Q}_l(a^k_l;\overline{\tau}^{k+1}_l)=\phi(\overline{\textbf{W}}^{k+1}_{l+1},\overline{\textbf{b}}^{k+1}_{l+1},\overline{\textbf{z}}^{k+1}_{l+1},\overline{\textbf{a}}^{k+1}_{l+1},u^k)$. Rather than minimizing the original problem in Equation \eqref{eq:original overline a}, we instead solve the following problem:
\begin{align} \overline{a}^{k+1}_l&\leftarrow\arg\min\nolimits_{a_l} \overline{Q}_l(a_l;\overline{\tau}^{k+1}_l). \label{eq:update overline a}
\end{align} Because $\overline{Q}_l(a_{l};\overline{\tau}^{k+1}_l)$ is a quadratic function with respect to $a_{l}$, the solution can be obtained by
\begin{align*}
\overline{a}^{k+1}_{l}\leftarrow a^k_{l}-\nabla_{a^k_{l}}\phi/\overline{\tau}^{k+1}_{l}.
\end{align*}
 given a suitable $\overline{\tau}^{k+1}_l$. Now the main focus is how to choose $\overline{\tau}^{k+1}_l$. Algorithm  \ref{algo:overline tau update} shows the backtracking algorithm utilized to find a suitable $\overline{\tau}_l^{k+1}$. Lines 2-5 implement a while loop until  the condition $\phi(\overline{\textbf{W}}^{k+1}_{l+1},\overline{\textbf{b}}^{k+1}_{l+1},\overline{\textbf{z}}^{k+1}_{l+1},\overline{\textbf{a}}^{k+1}_l,u^k)\leq\overline{Q}_l(\overline{a}^{k+1}_l;\overline{\tau}^{k+1}_l)$ is satisfied. As $\overline{\tau}^{k+1}_l$ becomes larger and larger, $\overline{a}^{k+1}_l$ is close to $a^k_l$ and $a^k_l$ satisfies the loop condition, which precludes the possibility of the infinite loop. The time complexity of Algorithm \ref{algo:overline tau update} is $O(n^2)$, where $n$ is the number of features or neurons.
 \begin{algorithm} 
\caption{The Backtracking Algorithm  to update $\overline{a}^{k+1}_{l}$ } 
\begin{algorithmic}[1]
\scriptsize
\label{algo:overline tau update}
\REQUIRE $\overline{\textbf{W}}^{k+1}_{l+1}$, $\overline{\textbf{b}}^{k+1}_{l+1}$, $\overline{\textbf{z}}^{k+1}_{l+1}$, $\overline{\textbf{a}}^{k+1}_{l+1}$, $u^k$, $\rho$, some constant $\overline{\eta}>1$. 
\ENSURE $\overline{\tau}^{k+1}_l$,$\overline{a}^{k+1}_{l}$. 
\STATE Pick up $\overline{t}$ and $\overline{\beta}=a^k_l-\nabla_{a^k_l}\phi/\overline{t}$
\WHILE{$\phi(\overline{\textbf{W}}^{k+1}_{l+1},\overline{\textbf{b}}^{k+1}_{l+1},\overline{\textbf{z}}^{k+1}_{l+1},\{a^k_i\}_{i=1}^{l-1},\overline{\beta},\{\overline{a}^{k+1}_i\}_{i=l+1}^{L-1},u^k)>\overline{Q}_l(\overline{\beta};\overline{t})$}
\STATE $\overline{t}\leftarrow \overline{t}\overline{\eta}$.\\
\STATE $\overline{\beta}\leftarrow a^k_{l}-\nabla_{a^k_{l}}\phi/\overline{t}$.\\
\ENDWHILE
\STATE Output $\overline{\tau}^{k+1}_l \leftarrow \overline{t} $.\\
\STATE Output $\overline{a}^{k+1}_{l}\leftarrow \overline{\beta}$.
\end{algorithmic}
\end{algorithm}
\\
\textbf{2. Update $\overline{z}^{k+1}_l$}\\
\indent The variables $\overline{z}^{k+1}_l(l=1,\cdots,L)$ are updated as follows:
\begin{align*}
    \overline{z}^{k+1}_{l}&\leftarrow \arg\min\nolimits_{z_{l}} L_\rho(\overline{\textbf{W}}^{k+1}_{l+1},\overline{\textbf{b}}^{k+1}_{l+1},\{z^k_i\}_{i=1}^{l-1},z_l,\{\overline{z}^{k+1}_i\}_{i=l+1}^{L},\\&\overline{\textbf{a}}^{k+1}_{l},u^k),
\end{align*}
which is equivalent to the following forms: for $\overline{z}^{k+1}_l(l=1,\cdots,L-1)$,
\begin{align}
          \nonumber\overline{z}^{k+1}_{l}&\leftarrow \arg\min\nolimits_{z_{l}} \phi(\overline{\textbf{W}}^{k+1}_{l+1},\overline{\textbf{b}}^{k+1}_{l+1},\{z^k_i\}_{i=1}^{l-1},z_l,\{\overline{z}^{k+1}_i\}_{i=l+1}^{L},\\&\overline{\textbf{a}}^{k+1}_{l},u^k),\label{eq:update overline z}
          \end{align}
          and for $\overline{z}^{k+1}_L$,
          \begin{align}
    \nonumber\overline{z}^{k+1}_L &\leftarrow\arg\min\nolimits_{z_L}      \phi(\textbf{W}^{k},\textbf{b}^{k},\{z^k_i\}_{i=1}^{L-1},z_L,\textbf{a}^{k},u^k)\\&+R(z_L;y).\label{eq:update overline zl}
\end{align}
\indent Equation \eqref{eq:update overline z} is highly nonconvex because the nonlinear activation function $f(z_l)$ is contained in $\phi(\bullet)$. For common activation functions such as the ReLU and leaky ReLU, Equation \eqref{eq:update overline z} has a closed-form solution; for other activation functions like sigmoid and hyperbolic tangent (tanh), a look-up table is recommended \cite{taylor2016training}.\\ \quad Equation \eqref{eq:update overline zl} is a convex problem because $\phi(\bullet)$ and $R(\bullet)$ are convex with regard to $z_L$. Therefore, Equation \eqref{eq:update overline zl} can be solved by Fast Iterative Soft-Thresholding Algorithm (FISTA) \cite{beck2009fast}.\\
\textbf{3. Update $\overline{b}^{k+1}_l$}\\
\indent The variables $\overline{b}^{k+1}_l(l=1,\cdots,L)$ are updated as follows:
\begin{align*}
    \overline{b}^{k+1}_l&\leftarrow \arg\min\nolimits_{b_l} 
     L_\rho(\overline{\textbf{W}}^{k+1}_{l+1},\{b^{k}_i\}_{i=1}^{l-1},b_l,\{\overline{b}^{k+1}_i\}_{i=l+1}^{L},\overline{\textbf{z}}^{k+1}_l,\\&\overline{\textbf{a}}^{k+1}_l,u^k),
\end{align*}
which is equivalent to the following form:
\begin{align*}
        \overline{b}^{k+1}_l&\leftarrow \arg\min\nolimits_{b_l} 
     \phi(\overline{\textbf{W}}^{k+1}_{l+1},\{b^{k}_i\}_{i=1}^{l-1},b_l,\{\overline{b}^{k+1}_i\}_{i=l+1}^{L},\overline{\textbf{z}}^{k+1}_l,\\&\overline{\textbf{a}}^{k+1}_l,u^k).
\end{align*}
\indent Similarly to the update of $\overline{a}^{k+1}_l$, we define $\overline{U}_l(b_l;\overline{B})$ as a quadratic approximation of $\phi(\bullet)$ at ${b}^k_l$, which is formulated mathematically as follows \cite{beck2009fast}:
\begin{align*}
 \overline{U}_l(b_l;\overline{B})&=\phi(\overline{\textbf{W}}^{k+1}_{l+1},\overline{\textbf{b}}^{k+1}_{l+1},\overline{\textbf{z}}^{k+1}_l,\overline{\textbf{a}}^{k+1}_l,u^k)\\&+(\nabla_{b^k_l}\phi)^T(\overline{\textbf{W}}^{k+1}_{l+1},\overline{\textbf{b}}^{k+1}_{l+1},\overline{\textbf{z}}^{k+1}_l,\overline{\textbf{a}}^{k+1}_l,u^k)(b_l-b^k_l)\\&+(\overline{B}/2)\Vert b_l-b^k_l\Vert^2_2,
\end{align*}
where $\overline{B}>0$ is a parameter. Here $\overline{B}\geq \nu$ for $l=1,\cdots,L-1$ and $\overline{B}\geq \rho$ for $l=L$ are required for the convergence analysis  \cite{beck2009fast}. Without loss of generality, we set $\overline{B}=\nu$, and solve the subsequent subproblem as follows:
\begin{align}
    &\overline{b}^{k+1}_l\leftarrow \arg\min\nolimits_{b_l} \overline{U}_l(b_l;\nu)(l=1,\cdots,L-1). \label{eq:update overline b}\\
    &\overline{b}^{k+1}_L\leftarrow \arg\min\nolimits_{b_L} \overline{U}_L(b_L;\rho). \label{eq:update overline bl}
\end{align}
\indent Equation \eqref{eq:update overline b} is a convex problem and has a closed-form solution as follows:
\begin{align*}
    &\overline{b}^{k+1}_{l}\leftarrow b^k_{l}-\nabla_{b^k_{l}}\phi/\nu.\quad (l=1,\cdots,L-1)\\&\overline{b}^{k+1}_{L}\leftarrow b^k_{L}-\nabla_{b^k_{L}}\phi/\rho.
\end{align*}
\textbf{4. Update $\overline{W}^{k+1}_l$}\\
\indent The variables $\overline{W}^{k+1}_l(l=1,\cdots,L)$ are updated as follows:
\begin{align*}
    \overline{W}^{k+1}_l&\leftarrow\arg\min\nolimits_{W_l} L_\rho(\{{W}^{k}_i\}_{i=1}^{l-1},W_l,\{\overline{W}^{k+1}_i\}_{i=l+1}^{L},\overline{\textbf{b}}^{k+1}_l,\\&\overline{\textbf{z}}^{k+1}_l,\overline{\textbf{a}}^{k+1}_l,u^k),
\end{align*}
which is equivalent to the following form:
\begin{align}
 \nonumber\overline{W}^{k+1}_l&\leftarrow\arg\min\nolimits_{W_l} \phi(\{{W}^{k}_i\}_{i=1}^{l-1},W_l,\{\overline{W}^{k+1}_i\}_{i=l+1}^{L},\overline{\textbf{b}}^{k+1}_l,\\&\overline{\textbf{z}}^{k+1}_l,\overline{\textbf{a}}^{k+1}_l,u^k)+\Omega(W_l).
    \label{eq:update W overline original}
\end{align}
\indent Due to the same challenge in updating $\overline{a}^{k+1}_l$,  we define $\overline{P}_l(W_l;\overline{\theta}^{k+1}_l)$ as a quadratic approximation of $\phi$ at ${W}^{k}_{l}$. The quadratic approximation is mathematically reformulated as follows \cite{beck2009fast}:
\begin{align*}\overline{P}_l(W_l;\overline{\theta}^{k+1}_l)&=\phi(\overline{\textbf{W}}^{k+1}_{l+1},\overline{\textbf{b}}^{k+1}_l,\overline{\textbf{z}}^{k+1}_l,\overline{\textbf{a}}^{k+1}_l,u^k)\\&+(\nabla_{W^k_l}\phi)^T(\overline{\textbf{W}}^{k+1}_{l+1},\overline{\textbf{b}}^{k+1}_l,\overline{\textbf{z}}^{k+1}_l,\overline{\textbf{a}}^{k+1}_l,u^k)\\&(W_l-W^k_l)+\Vert\overline{\theta}^{k+1}_l\circ (W_l-W^k_l)^{\circ 2}\Vert_{1}/2,
\end{align*}
where $\overline{\theta}^{k+1}_l>0$ is a  parameter vector, which is chosen by the Algorithm \ref{algo:overline theta update}. Instead of minimizing the Equation \eqref{eq:update W overline original}, we minimize the following:
\begin{align}
 &\overline{W}^{k+1}_l \leftarrow \arg\min\nolimits_{W_l} \overline{P}_l(W_l;\overline{\theta}^{k+1}_l)+\Omega_l (W_l).     \label{eq:update overline W}
\end{align}
\indent Equation \eqref{eq:update overline W} is convex and hence can be solved exactly. If $\Omega_l$ is either an $\ell_1$ or an $\ell_2$ regularization term, Equation \eqref{eq:update overline W} has a closed-form solution.
\begin{algorithm} 
\caption{The Backtracking Algorithm  to update $\overline{W}^{k+1}_{l}$ }
\begin{algorithmic}[1]
\scriptsize
\label{algo:overline theta update}
\REQUIRE $\overline{\textbf{W}}^{k+1}_{l+1}$,$\overline{\textbf{b}}^{k+1}_{l}$, $\overline{\textbf{z}}^{k+1}_{l}$,$\overline{\textbf{a}}^{k+1}_{l}$,$u^k$, $\rho$, some constant $\overline{\gamma}>1$. 
\ENSURE $\overline{\theta}^{k+1}_l$,$\overline{{W}}^{k+1}_{l}$. 
\STATE Pick up $\overline{\alpha}$ and $\overline{\zeta}=W^k_l-\nabla_{W^k_l}\phi/\overline{\alpha}$.
\WHILE{$\phi(\{W^k_i\}_{i=1}^{l-1},\overline{\zeta},\{\overline{W}^{k+1}_i\}_{i=l+1}^{L},\overline{\textbf{b}}^{k+1}_{l},\overline{\textbf{z}}^{k+1}_{l},\overline{\textbf{a}}^{k+1}_{l},u^k)>\overline{P}_l(\overline{\zeta};\overline{\alpha})$}
\STATE $\overline{\alpha}\leftarrow \overline{\alpha}\ \overline{\gamma}$.\\
\STATE Solve $\overline{\zeta}$ by Equation \eqref{eq:update overline W}.\\
\ENDWHILE
\STATE Output $\overline{\theta}^{k+1}_l \leftarrow \overline{\alpha} $.\\
\STATE Output $\overline{W}^{k+1}_{l}\leftarrow \overline{\zeta}$.
\end{algorithmic}
\end{algorithm}
\\
\textbf{5. Update ${W}^{k+1}_l$}\\
\indent The variables ${W}^{k+1}_l(l=1,\cdots,L)$ are updated as follows:
\begin{align*}
    W^{k+1}_l&\leftarrow\arg\min\nolimits_{W_l} L_\rho(\{{W}^{k+1}_i\}_{i=1}^{l-1},W_l,\{\overline{W}^{k+1}_i\}_{i=l+1}^{L},\textbf{b}^{k+1}_{l-1},\\&\textbf{z}^{k+1}_{l-1},\textbf{a}^{k+1}_{l-1},u^k).
\end{align*}
which is equivalent to 
\begin{align}
 \nonumber{W}^{k+1}_l&\leftarrow\arg\min\nolimits_{W_l} \phi(\{{W}^{k+1}_i\}_{i=1}^{l-1},W_l,\{\overline{W}^{k+1}_i\}_{i=l+1}^{L},\textbf{b}^{k+1}_{l-1},\\&\textbf{z}^{k+1}_{l-1},\textbf{a}^{k+1}_{l-1},u^k)+\Omega(W_l).
    \label{eq:update W original}
\end{align}
\indent Similarly,  we define $P_l(W_l;\theta^{k+1}_l)$ as a quadratic approximation of $\phi$ at $\overline{W}^{k+1}_{l}$. The quadratic approximation is then mathematically reformulated as follows \cite{beck2009fast}:
\begin{align*}{P}_l(W_l;{\theta}^{k+1}_l)&=\phi({\textbf{W}}^{k+1}_{l-1},{\textbf{b}}^{k+1}_{l-1},{\textbf{z}}^{k+1}_{l-1},{\textbf{a}}^{k+1}_{l-1},u^k)\\&+(\nabla_{\overline{W}^{k+1}_l}\phi)^T({\textbf{W}}^{k+1}_{l-1},\textbf{b}^{k+1}_{l-1},\textbf{z}^{k+1}_{l-1},{\textbf{a}}^{k+1}_{l-1},u^k)\\&(W_l-\overline{W}^{k+1}_l)+\Vert{\theta}^{k+1}_l\circ (W_l-\overline{W}^{k+1}_l)^{\circ 2}\Vert_{1}/2.
\end{align*}
where $\theta^{k+1}_l>0$ is a  parameter vector. Instead of minimizing the Equation \eqref{eq:update W original}, we minimize the following:
\begin{align}
 &{W}^{k+1}_l \leftarrow \arg\min\nolimits_{W_l} P_l(W_l;{\theta}^{k+1}_l)+\Omega_l (W_l).     \label{eq:update W}
\end{align}
The choice of $\theta^{k+1}_l$ is discussed in the supplementary materials.\\
\textbf{6. Update ${b}^{k+1}_l$}\\
\indent The variables ${b}^{k+1}_l(l=1,\cdots,L)$ are updated as follows:
\begin{align*}
    {b}^{k+1}_l&\leftarrow \arg\min\nolimits_{b_l} 
     L_\rho({\textbf{W}}^{k+1}_{l},\{b^{k+1}_i\}_{i=1}^{l-1},b_l,\{\overline{b}^{k+1}_i\}_{i=l+1}^{L},{\textbf{z}}^{k+1}_{l-1},\\&{\textbf{a}}^{k+1}_{l-1},u^k),
\end{align*}
which is equivalent to the following formulation:
\begin{align*}
        {b}^{k+1}_l&\leftarrow \arg\min\nolimits_{b_l} 
     \phi({\textbf{W}}^{k+1}_{l},\{b^{k+1}_i\}_{i=1}^{l-1},b_l,\{\overline{b}^{k+1}_i\}_{i=l+1}^{L},{\textbf{z}}^{k+1}_{l-1},\\&{\textbf{a}}^{k+1}_{l-1},u^k).
\end{align*}
$U_l(b_l;B)$ is defined as the quadratic approximation of $\phi$ at $\overline{b}^{k+1}_l$ as follows:
\begin{align*}
 U_l(b_l;B)&=\phi(\textbf{W}^{k+1}_{l},\textbf{b}^{k+1}_{l-1},{\textbf{z}}^{k+1}_{l-1},{\textbf{a}}^{k+1}_{l-1},u^k)\\&+\nabla_{\overline{b}^{k+1}_l}\phi^T({\textbf{W}}^{k+1}_{l},{\textbf{b}}^{k+1}_{l-1},{\textbf{z}}^{k+1}_{l-1},{\textbf{a}}^{k+1}_{l-1},u^k)(b_l-\overline{b}^{k+1}_l)\\&+(B/2)\Vert b_l-\overline{b}^{k+1}_l\Vert^2_2,
\end{align*}
where $B> 0$ is a parameter. We set $B=\nu$ for $l=1,\cdots,L-1$ and $B=\rho$ for $l=L$, and solve the resulting subproblems as follows:
\begin{align}
    &{b}^{k+1}_l\leftarrow \arg\min\nolimits_{b_l} {U}_l(b_l;\nu)(l=1,\cdots,L-1). \label{eq:update b}
    \\&{b}^{k+1}_L\leftarrow \arg\min\nolimits_{b_L} {U}_L(b_L;\rho). \label{eq:update bl}
\end{align}
\indent The solutions to Equations \eqref{eq:update b} and \eqref{eq:update bl} are as follows:
\begin{align*}
    &b^{k+1}_l\leftarrow \overline{b}^{k+1}_l-\nabla_{\overline{b}^{k+1}_L}\phi/\nu (l=1,\cdots,L-1).
    \\&b^{k+1}_L\leftarrow \overline{b}^{k+1}_L-\nabla_{\overline{b}^{k+1}_L}\phi/\rho.
\end{align*}
\textbf{7. Update ${z}^{k+1}_l$}\\
\indent The variables ${z}^{k+1}_l(l=1,\cdots,L)$ are updated as follows:
\begin{align*}
    {z}^{k+1}_{l}&\leftarrow \arg\min\nolimits_{z_{l}} L_\rho({\textbf{W}}^{k+1}_{l},{\textbf{b}}^{k+1}_{l},\{z^{k+1}_i\}_{i=1}^{l-1},z_l,\{\overline{z}^{k+1}_i\}_{i=l+1}^{L},\\&\textbf{a}^{k+1}_{l-1},u^k),
\end{align*}
which is equivalent to the following forms for $z_l(l=1,\cdots,L-1)$:
\begin{align}
    \nonumber{z}^{k+1}_{l}&\leftarrow \arg\min\nolimits_{z_{l}} \phi({\textbf{W}}^{k+1}_{l},{\textbf{b}}^{k+1}_{l},\{z^{k+1}_i\}_{i=1}^{l-1},z_l,\{\overline{z}^{k+1}_i\}_{i=l+1}^{L-1},\\&\textbf{a}^{k+1}_{l-1},u^k),
    \label{eq:update z}
    \end{align}
    and for $z_L$:
\begin{align}
\nonumber
    {z}^{k+1}_L &\leftarrow\arg\min\nolimits_{z_L}      \phi(\textbf{W}^{k+1}_{L},\textbf{b}^{k+1}_{L},\{z^{k+1}_i\}_{i=1}^{L-1},z_L,\textbf{a}^{k}_{L-1},u^k)\\&+R(z_L;y).\label{eq:update zl}
\end{align}
Solving Equations \eqref{eq:update z} and \eqref{eq:update zl} proceeds exactly the same as solving Equations \eqref{eq:update overline z} and \eqref{eq:update overline zl}, respectively.\\
\textbf{8. Update ${a}^{k+1}_l$}\\
\indent The variables ${a}^{k+1}_l(l=1,\cdots,L-1)$ are updated as follows:
\begin{align*}
    {a}^{k+1}_{l}&\leftarrow \arg\min\nolimits_{a_{l}} L_\rho({\textbf{W}}^{k+1}_{l},{\textbf{b}}^{k+1}_{l},{\textbf{z}}^{k+1}_{l},\{a^k_i\}_{i=1}^{l-1},a_l,\\&\{\overline{a}^{k+1}_i\}_{i=l+1}^{L-1},u^k),
\end{align*}
\indent which is equivalent to the following form:
\begin{align*}
{a}^{k+1}_{l}&\leftarrow \arg\min\nolimits_{a_{l}} \phi({\textbf{W}}^{k+1}_{l},{\textbf{b}}^{k+1}_{l},{\textbf{z}}^{k+1}_{l},\{a^k_i\}_{i=1}^{l-1},a_l,\\&\{\overline{a}^{k+1}_i\}_{i=l+1}^{L-1},u^k).
\end{align*}
$Q_l(a_l;\tau^{k+1}_l)$ is defined as the quadratic approximation of $\phi$ at $a^{k+1}_l$ as follows:
\begin{align*}{Q}_l(a_{l};{\tau}^{k+1}_l)&=\phi({\textbf{W}}^{k+1}_{l},{\textbf{b}}^{k+1}_{l},{\textbf{z}}^{k+1}_{l},{\textbf{a}}^{k+1}_{l-1},u^k)\\&+(\nabla _{\overline{a}^{k+1}_{l}}\phi)^T({\textbf{W}}^{k+1}_{l},{\textbf{b}}^{k+1}_{l},{\textbf{z}}^{k+1}_{l},{\textbf{a}}^{k+1}_{l-1},u^k)\\&(a_{l}-\overline{a}^{k+1}_{l})+\Vert{\tau}^{k+1}_l\circ (a_{l}-\overline{a}^{k+1}_{l})^{\circ 2}\Vert_{1}/2,
\end{align*}
and we can solve the following problem instead:
\begin{align} a^{k+1}_l&\leftarrow\arg\min\nolimits_{a_l} {Q}_l(a_l;{\tau}^{k+1}_l). \label{eq:update a}
\end{align}
where ${\tau}^{k+1}_l>0$ is a  parameter vector. The solution to Equation \eqref{eq:update a} can be obtained by
${a}^{k+1}_{l}\leftarrow \overline{a}^{k+1}_{l}-\nabla_{\overline{a}^{k+1}_{l}}\phi/{\tau}^{k+1}_{l}.$
To choice of an appropriate  $\tau^{k+1}_l$ is shown  in the supplementary materials.
\subsection{The Application on the GCN  Problem}
\label{sec:gcn problem}
\begin{algorithm}[H]
\caption{The dlADMM Algorithm to Solve Problem \ref{algro: dlADMM GCN}}
\label{algro: dlADMM GCN}
\scriptsize
\begin{algorithmic}[1]
 \REQUIRE ${\mathcal{Y}}$, $\mathcal{Z}_0$, $\mu, \rho$.
\ENSURE $\mathcal{Z}_l$, ${\mathcal{W}_l}, l=1,\cdots, L$.
\STATE \textbf{Initialize:} $k=0$.
\WHILE{ ${\mathcal{Z}}_l^k$, ${\mathcal{W}}_l^k$ not converged} 
\FOR{$l=L$ to $1$}
\IF{$l<L$}
\STATE Update $\overline{\mathcal{Z}}_l^{k+1}$ in Equation \eqref{eq: update overline mathcal z}. 
\ELSE
\STATE Update $\overline{\mathcal{Z}}_L^{k+1}$ in Equation \eqref{eq: original overline mathcal zl}.
\ENDIF
\STATE Update $\overline{\mathcal{W}}_l^{k+1}$ in Equation \eqref{eq:update overline mathcal w}.
\ENDFOR
\FOR{$l=1$ to $L$}
\STATE Update ${\mathcal{W}}_{l}^{k+1}$ in Equation \eqref{eq:update mathcal w}.
\IF{$l<L$}
\STATE Update ${\mathcal{Z}}_{l}^{k+1}$ in Equation
\eqref{eq: mathcal z}.
\ELSE
\STATE Update ${\mathcal{Z}}_{L}^{k+1}$ in Equation
\eqref{eq: mathcal zl}.
\ENDIF
\ENDFOR
\STATE Update $\epsilon^{k+1}\leftarrow\mathcal{Z}_{L}^{k+1}-\tilde{\mathcal{A}}\mathcal{Z}_{L-1}^{k+1}\mathcal{W}_{L}^{k+1}$.
\STATE Update $\mathcal{U}^{k+1}\leftarrow
\mathcal{U}^{k}+\rho\epsilon^{k+1}$.
\ENDWHILE
\end{algorithmic}
\end{algorithm}

\indent In this section, we apply the dlADMM algorithm to the GCN problem, which is one of the state-of-the-art architectures in Graph Neural Networks (GNNs) \cite{wu2020}.
Important notations are listed in Table \ref{tab:GCN notation}.
We consider the semi-supervised node classification task, in which the input of GCN is an undirected and unweighted graph $\mathcal{G}=(\mathcal{V}; \mathcal{E})$, with $\mathcal{V}$ a set of $\mathcal{N}$ nodes and $ \mathcal{E}$ a set of edges between them. $\mathcal{A}\in \{0, 1\}^{\mathcal{N} \times \mathcal{N}} $, $\mathcal{D}\in \mathbb{R}^{\mathcal{N} \times \mathcal{N}}$ are the adjacency matrix and the degree matrix, respectively. $\mathcal{Z}_0=\mathcal{X} \in \mathbb{R}^{\mathcal{N} \times \mathcal{C}_0}$ is the input feature matrix, each row of which corresponds to the input feature vector of one node, and  $\mathcal{Y}$ is the pre-defined node label vector. Then the task of training the GCN model is formulated in the following fashion:
\begin{problem}[GCN Training Problem]
\label{prob: problem 3}
 \begin{align*}
   &\min\nolimits_{\mathcal{W}_l, \mathcal{Z}_l} \quad \mathcal{R}(\mathcal{Z}_L,\mathcal{Y}), \\
   &s.t. \qquad\qquad \mathcal{Z}_l = f_l(\mathcal{\tilde{A}}\mathcal{Z}_{l-1}\mathcal{W}_{l}),\quad l=1, \cdots, L-1\\
   & \qquad\qquad\quad \mathcal{Z}_L = \mathcal{\tilde{A}}\mathcal{Z}_{L-1}\mathcal{W}_{L}.
   \end{align*}
   \end{problem}
In Problem \ref{prob: problem 3}, $\mathcal{\tilde{A}} = (\mathcal{D}+{I})^{-1/2}(\mathcal{A}+{I})(\mathcal{D}+{I})^{-1/2}$ is the normalized adjacency matrix, $f_l$ is a non-linear activation function for the $l$-th layer \ (e.g., ReLU), and $\mathcal{R}(\mathcal{Z}_l,\mathcal{Y})$ is some risk function that should satisfy the conditions mentioned in Section \ref{sec:MLP application}. Similar to the MLP training problem, We relax Problem \ref{prob: problem 3} to  Problem \ref{prob: problem 4} as follows:

\begin{table}
\centering
 \begin{tabular}{cc}
 \hline
 Notations&Descriptions\\ \hline
 $L$& Number of layers.\\
 $\mathcal{A}$ & The adjacency matrix of a graph.\\
 $\mathcal{D}$ & The degree matrix of a graph.\\
 $\tilde{\mathcal{A}}$ & The normalized adjacency matrix of a graph.\\
 $\mathcal{W}_l$& The weight matrix for the $l$-th layer.\\
 $f_l(\cdot)$& The nonlinear activation function for the $l$-th layer.\\
 $\mathcal{Z}_l$& The output for the $l$-th layer.\\
 $\mathcal{Z}_0$& The input feature matrix of the neural network.\\
 $\mathcal{Y}$& The predefined node label matrix.\\
 $\mathcal{R}(\mathcal{Z}_L,\mathcal{Y})$& The risk function for the $L$-th layer.\\
\hline
  \end{tabular}
    \caption{Important Notations of the GCN model.}
   \label{tab:GCN notation}
\end{table}
\begin{problem}
\label{prob: problem 4}
    \begin{align*}
   &\min\nolimits_{\mathcal{W}_l,\mathcal{Z}_l} \quad \mathcal{R}(\mathcal{Z}_L,\mathcal{Y}) +  \frac{\mu}{2}\sum_{l=1}^{L-1}\Vert\mathcal{Z}_l - f_l(\mathcal{\tilde{{A}}}\mathcal{Z}_{l-1}\mathcal{W}_{l})\Vert_F^2, \\
   &s.t.  \qquad\qquad\quad \mathcal{Z}_L = \tilde{\mathcal{A}}\mathcal{Z}_{L-1}\mathcal{W}_{L}.
    \end{align*}
\end{problem}
where $\mathbfcal{W}=\{\mathcal{W}_l\}_{l=1}^{L}, \mathbfcal{Z}=\{\mathcal{Z}_l\}_{l=1}^{L}$, and $\mu >0$ is a tuning parameter. The augmented Lagrangian is formulated as follows:
\begin{align*}
  &\nonumber \mathcal{L}_{\rho}(\mathbfcal{W},\mathbfcal{Z},\mathcal{U}) = \mathcal{R}(\mathcal{Z}_L,\mathcal{Y})  +\psi(\mathbfcal{W},\mathbfcal{Z},\mathcal{U}).
\end{align*}
where
 $\psi(\mathbfcal{W},\mathbfcal{Z},\mathcal{U}) =  (\mu/2)\sum_{l=1}^{L-1}\|\mathcal{Z}_l - f_l(\tilde{\mathcal{A}}\mathcal{Z}_{l-1}\mathcal{W}_{l})\|_F^2 
  +\langle \mathcal{U}, \mathcal{Z}_L - \tilde{\mathcal{A}}\mathcal{Z}_{L-1}\mathcal{W}_{L}\rangle + (\rho/2)\|\mathcal{Z}_L - \tilde{\mathcal{A}}\mathcal{Z}_{L-1}\mathcal{W}_{L}\|_F^2$, 
$\rho>0$ is a penalty parameter and $\mathcal{U}$ is a dual variable. $\overline{\mathcal{W}}_l^{k+1}$ and $\overline{\mathcal{Z}}_l^{k+1}$ denote backward updates of the dlADMM for the $l$-th layer in the $(k+1)$-th iteration, while $\mathcal{W}_l^{k+1}$ and $\mathcal{Z}_l^{k+1}$ denote forward updates of the dlADMM algorithm for the $l$-th layer in the $(k+1)$-th iteration. Moreover, some notations are shown as follows for the sake of simplicity: $\overline{\mathbfcal{W}}_l^{k+1}=\{\{\mathcal{W}_i^{k}\}_{i=1}^{l-1},\{\overline{\mathcal{W}}_i^{k+1}\}_{i=l}^{L}\}$, $\overline{\mathbfcal{Z}}_l^{k+1}=\{\{\mathcal{Z}_i^{k}\}_{i=1}^{l-1},\{\overline{\mathcal{Z}}_i^{k+1}\}_{i=l}^{L}\}$, ${\mathbfcal{W}}_l^{k+1}=\{\{\mathcal{W}_i^{k+1}\}_{i=1}^{l},\{\overline{\mathcal{W}}_i^{k+1}\}_{i=l+1}^{L}\}$, ${\mathbfcal{Z}}_l^{k+1}=\{\{\mathcal{Z}_i^{k+1}\}_{i=1}^{l},\{\overline{\mathcal{Z}}_i^{k+1}\}_{i=l+1}^{L}\}$, $\overline{\mathbfcal{W}}^{k+1}=\{\overline{\mathcal{W}}_i^{k+1}\}_{i=1}^{L}$, $\overline{\mathbfcal{Z}}^{k+1}=\{\overline{\mathcal{Z}}_i^{k+1}\}_{i=1}^{L}$, ${\mathbfcal{W}}^{k+1}=\{{\mathcal{W}}_i^{k+1}\}_{i=1}^{L}$, and ${\mathbfcal{Z}}^{k+1}=\{{\mathcal{Z}}_i^{k+1}\}_{i=1}^{L}$. 

Then the dlADMM algorithm to solve Problem \ref{prob: problem 4} is shown in Algorithm \ref{algro: dlADMM GCN}. Specifically, Lines 5, 9, 12 and 14 address four subproblems, namely, $\overline{\mathcal{Z}}_l^{k+1}$, $\overline{\mathcal{W}}_l^{k+1}$,
${\mathcal{W}}_l^{k+1}$ and
${\mathcal{Z}}_l^{k+1}$, respectively. Lines 19 and 20 update the residual $\epsilon^{k+1}$ and the dual variable $\mathcal{U}^{k+1}$, respectively. \\
 \indent Now we discuss four subproblems generated by the dlADMM algorithm  in detail.\\
\textbf{1. Update $\overline{\mathcal{Z}}^{k+1}_{l}$}\\
\indent The variable $\overline{\mathcal{Z}}^{k+1}_{l}$ is updated as follows:
\begin{align*}
     \nonumber\overline{\mathcal{Z}}^{k+1}_{l} &\leftarrow \arg\min\nolimits_{\mathcal{Z}_{l}}\mathcal{L}_{\rho}(\mathbfcal{\overline{W}}^{k+1}_{l+1},\{\mathcal{Z}_i^k\}_{i=1}^{l-1}, \mathcal{Z}_l,\{\overline{\mathcal{Z}}_i^{k+1}\}_{i=l+1}^L,\mathcal{U}^k)\\
     &=\psi(\mathbfcal{\overline{W}}^{k+1}_{l+1},\{\mathcal{Z}_i^k\}_{i=1}^{l-1}, \mathcal{Z}_l,\{\overline{\mathcal{Z}}_i^{k+1}\}_{i=l+1}^L,\mathcal{U}^k),
 \end{align*}
where $l=1, \cdots, L-1$. We still use the quadratic approximation as follows:
\begin{align}\label{eq: update overline mathcal z}
     \overline{\mathcal{Z}}_{l}^{k+1}\leftarrow \arg\min\nolimits_{\mathcal{Z}_{l}} \overline{\mathcal{Q}}_l(\mathcal{Z}_{l}; \overline{\delta}_l^{k+1}),
\end{align}
where $\overline{\mathcal{Q}}_l(\mathcal{Z}_{l}; \overline{\delta}_l^{k+1})
=\psi(\mathbfcal{\overline{W}}^{k+1}_{l+1},\overline{\mathbfcal{Z}}_{l+1}^{k+1},\mathcal{U}^k)+\langle \psi(\mathbfcal{\overline{W}}^{k+1}_{l+1},\overline{\mathbfcal{Z}}_{l+1}^{k+1},\mathcal{U}^k), \mathcal{Z}_l-{\mathcal{Z}}_l^{k}\rangle + (\overline{\delta}_l^{k+1}/2)\|\mathcal{Z}_l-{\mathcal{Z}}_l^{k}\|_F^2$, and $\overline{\delta}_l^{k+1}>0$ is a parameter to be chosen via the backtracking technique, it should satisfy $
 \overline{\mathcal{Q}}_l(\overline{\mathcal{Z}}_{l}^{k+1}; \overline{\delta}_l^{k+1})  \geq \psi(\mathbfcal{\overline{W}}^{k+1}_{l+1},\overline{\mathbfcal{Z}}_{l}^{k+1},\mathcal{U}^k)$.
The solution is 
\begin{align*}\overline{\mathcal{Z}}_{l}^{k+1}
\leftarrow  \mathcal{Z}^k_l-
\psi(\mathbfcal{\overline{W}}^{k+1}_{l+1},\overline{\mathbfcal{Z}}_{l+1}^{k+1},\mathcal{U}^k)/\overline{\delta}_l^{k+1}.\end{align*}
\indent For $l=L$, we have
\begin{align}
      \nonumber\overline{\mathcal{Z}}^{k+1}_{L} &\leftarrow \arg\min\nolimits_{\mathcal{Z}_{L}} \mathcal{L}_{\rho}(\mathbfcal{W}^{k},\{\mathcal{Z}_i^k\}_{i=1}^{L-1}, \mathcal{Z}_L,\mathcal{U}^k)\\
     &=\mathcal{R}(\mathcal{Z}_{L}, \mathcal{Y})+\psi(\mathbfcal{W}^{k},\{\mathcal{Z}_i^k\}_{i=1}^{L-1}, \mathcal{Z}_L,\mathcal{U}^k).
     \label{eq: original overline mathcal zl}
\end{align}
We can solve Equation \eqref{eq: original overline mathcal zl} via FISTA \cite{beck2009fast}.\\
\textbf{2. Update $\overline{\mathcal{W}}_{l}^{k+1}$}\\
\indent The variable $\overline{\mathcal{W}}_{l}^{k+1}$ is updated as follows:
\begin{align*}
\overline{\mathcal{W}}_{l}^{k+1} & \leftarrow  \arg\min\nolimits_{{\mathcal{W}}_{l}} \mathcal{L}_{\rho}(\{\mathcal{W}_{i}^{k}\}_{i=1}^{l-1},{\mathcal{W}}_{l}, \{\overline{\mathcal{W}}_{i}\}_{i=l+1}^{L} ,\overline{\mathbfcal{Z}}_l^{k+1},\mathcal{U}^k)\\
&=\psi(\{\mathcal{W}_{i}^{k}\}_{i=1}^{l-1},{\mathcal{W}}_{l}, \{\overline{\mathcal{W}}_{i}\}_{i=l+1}^{L} ,\overline{\mathbfcal{Z}}_l^{k+1},\mathcal{U}^k).
 \end{align*}
We apply similar quadratic approximation techniques as follows:
\begin{align}\label{eq:update overline mathcal w}
\overline{\mathcal{W}}_{l}^{k+1}\leftarrow \arg\min\nolimits_{\mathcal{W}_{l}} \overline{\mathcal{P}}_{l}(\mathcal{W}_{l};\overline{\sigma}_{l}^{k+1}),
\end{align}
where 
\begin{align*}
&\overline{\mathcal{P}}_l(\mathcal{W}_{l};\overline{\sigma}_l^{k+1})=\psi( \overline{\mathbfcal{W}}_{l+1}^{k+1} ,\overline{\mathbfcal{Z}}_l^{k+1},\mathcal{U}^k)\\
&+\langle\nabla_{\mathcal{W}_{l}^k}\psi( \overline{\mathbfcal{W}}_{l+1}^{k+1} ,\overline{\mathbfcal{Z}}_l^{k+1},\mathcal{U}^k), \mathcal{W}_{l}-\mathcal{W}_{l}^k\rangle\\
&+(\overline{\sigma}_l^{k+1}/2)\|\mathcal{W}_{l}-\mathcal{W}_{l}^k\|_F^2.
 \end{align*}
\\
 and $\overline{\sigma}_l^{k+1}>0$ is a parameter to be chosen via the backtracking technique. It should satisfy:
\begin{align*}
\overline{\mathcal{P}}_l(\mathcal{W}_{l}^{k+1};\overline{\sigma}_l^{k+1}) \geq \psi( \overline{\mathbfcal{W}}_{l+1}^{k+1} ,\overline{\mathbfcal{Z}}_l^{k+1},\mathcal{U}^k).
 \end{align*}
 
The solution to Equation \eqref{eq:update overline mathcal w} is:
\begin{align*}
\overline{\mathcal{W}}_{l}^{k+1}
\leftarrow
 \mathcal{W}_{l}^{k}-\nabla_{\mathcal{W}_{l}^k}\psi( \overline{\mathbfcal{W}}_{l+1}^{k+1} ,\overline{\mathbfcal{Z}}_l^{k+1},\mathcal{U}^k)/\overline{\sigma}_l^{k+1}.
 \end{align*}
\textbf{3. Update ${\mathcal{W}}_{l}^{k+1}$}\\
\indent The variable ${\mathcal{W}}_{l}^{k+1}$ is updated as follows:
\begin{align*}
{\mathcal{W}}_{l}^{k+1} & \leftarrow  \arg\min\nolimits_{{\mathcal{W}}_{l}} \mathcal{L}_{\rho}(\{\mathcal{W}_{i}^{k\!+\!1}\}_{i\!=\!1}^{l\!-\!1},{\mathcal{W}}_{l}, \{\overline{\mathcal{W}}_{i}\}_{i\!=\!l\!+\!1}^{L} ,\mathbfcal{Z}_{l\!-\!1}^{k\!+\!1},\mathcal{U}^k)\\
&=\psi(\{\mathcal{W}_{i}^{k+1}\}_{i=1}^{l-1},{\mathcal{W}}_{l}, \{\overline{\mathcal{W}}_{i}\}_{i=l+1}^{L} ,\mathbfcal{Z}_{l-1}^{k\!+\!1},\mathcal{U}^k).
 \end{align*}
 
Similarly, we define ${\mathcal{P}}_{l}(\mathcal{W}_{l};{\sigma}_{l}^{k+1})$ as the quadratic approximation of $\psi$ at $\mathcal{W}_{l}^k$ that is formulated as below:
\begin{align*}
&{\mathcal{P}}_l(\mathcal{W}_{l};{\sigma}_l^{k+1})=\psi( {\mathbfcal{W}}_{l-1}^{k+1} ,\mathbfcal{Z}_{l-1}^{k\!+\!1},\mathcal{U}^k)\\
&+\langle\nabla_{\overline{\mathcal{W}}_{l}^{k+1}}\psi({\mathbfcal{W}}_{l-1}^{k+1} ,\mathbfcal{Z}_{l-1}^{k\!+\!1},\mathcal{U}^k), \mathcal{W}_{l}-\overline{\mathcal{W}}_{l}^{k+1}\rangle\\
&+({\sigma}_l^{k+1}/2)\|\mathcal{W}_{l}-\overline{\mathcal{W}}_{l}^{k+1}\|_F^2,
 \end{align*}
 where ${\sigma}_l^{k+1}>0$ is a parameter to be chosen via the backtracking technique. It should satisfy ${\mathcal{P}}_l(\mathcal{W}_{l}^{k+1};{\sigma}_l^{k+1}) \geq \psi( {\mathbfcal{W}}_{l-1}^{k+1} ,\mathbfcal{Z}_{l-1}^{k\!+\!1},\mathcal{U}^k)$. Then we solve the following subproblem instead: 
\begin{align}
{\mathcal{W}}_{l}^{k+1}\leftarrow \arg\min\nolimits_{\mathcal{W}_{l}} {\mathcal{P}}_{l}(\mathcal{W}_{l};{\sigma}_{l}^{k+1}).
\label{eq:update mathcal w}
\end{align}
The solution to Equation \eqref{eq:update mathcal w} is:
\begin{align*}
\mathcal{W}_{l}^{k+1}
\leftarrow
 \overline{\mathcal{W}}_{l}^{k+1}-\nabla_{\overline{\mathcal{W}}_{l}^{k+1}}\psi( {\mathbfcal{W}}_{l-1}^{k+1} ,\mathbfcal{Z}_{l-1}^{k\!+\!1},\mathcal{U}^k)/{\sigma}_l^{k+1}.
 \end{align*}
\textbf{4. Update ${\mathcal{Z}}^{k+1}_{l}$}\\
\indent The variable ${\mathcal{Z}}^{k+1}_{l}$ is updated as follows:
\begin{align*}
     &\nonumber{\mathcal{Z}}^{k+1}_{l} \leftarrow \arg\min\nolimits_{\mathcal{Z}_{l}}\mathcal{L}_{\rho}({\mathbfcal{W}}_l^{k+1},\{\mathcal{Z}_i^{k+1}\}_{i=1}^{l-1}, \mathcal{Z}_l,\\\nonumber &\{\overline{\mathcal{Z}}_i^{k+1}\}_{i=l+1}^L,\mathcal{U}^k)\\
     &=\psi({\mathbfcal{W}}_l^{k+1},\{\mathcal{Z}_i^{k+1}\}_{i=1}^{l-1}, \mathcal{Z}_l,\{\overline{\mathcal{Z}}_i^{k+1}\}_{i\!=\!l\!+\!1}^L,\mathcal{U}^k),
 \end{align*}
where $l=1, \cdots, L-1$, and
\begin{align}
      \nonumber{\mathcal{Z}}^{k+1}_{L} &\leftarrow \arg\min\nolimits_{\mathcal{Z}_{L}} \mathcal{L}_{\rho}({\mathbfcal{W}}^{k+1},\{\mathcal{Z}_i^{k+1}\}_{i=1}^{L-1}, \mathcal{Z}_L,\mathcal{U}^k)\\
     &=\mathcal{R}(\mathcal{Z}_{L}, \mathcal{Y})+\psi({\mathbfcal{W}}^{k+1},\{\mathcal{Z}_i^{k+1}\}_{i=1}^{L-1}, \mathcal{Z}_L,\mathcal{U}^k).
     \label{eq: mathcal zl}
\end{align}
Similarly, we define
$\mathcal{Q}_l(\mathcal{Z}_{l}; {\delta}_l^{k+1})$ as the quadratic approximation of $\psi$ at $\mathcal{Z}_{l}^{k+1}$ as follows:
\begin{align*}
&\mathcal{Q}_l(\mathcal{Z}_{l}; {\delta}_l^{k+1})
=\psi(\mathbfcal{W}_l^{k+1},{\mathbfcal{Z}}_{l-1}^{k+1},\mathcal{U}^k)\\
&\!+\!\langle \nabla_{\overline{\mathcal{Z}}_l^{k+1}}\psi(\mathbfcal{W}_l^{k+1},{\mathbfcal{Z}}_{l-1}^{k+1},\mathcal{U}^k), \mathcal{Z}_l\!-\!\overline{\mathcal{Z}}_l^{k+1}\rangle \!+\! ({\delta}_l^{k+1}/2)\|\mathcal{Z}_l-\overline{\mathcal{Z}}_l^{k+1}\|_F^2,
\end{align*}
where $\delta_l^{k+1}>0$ is a parameter to be chosen via the backtracking technique. It should satisfy $
 \mathcal{Q}_l(\mathcal{Z}_{l}^{k+1}; {\delta}_l^{k+1})  \geq \psi(\mathbfcal{W}_l^{k+1},{\mathbfcal{Z}}_{l-1}^{k+1},\mathcal{U}^k)$. We solve the following subproblem instead:
\begin{align}
     \mathcal{Z}_{l}^{k+1}\leftarrow \arg\min\nolimits_{\mathcal{Z}_{l}} \mathcal{Q}_l(\mathcal{Z}_{l}; {\delta}_l^{k+1}).
     \label{eq: mathcal z}
\end{align}
The solution is 
\begin{align*}
\mathcal{Z}_{l}^{k+1}\leftarrow \overline{\mathcal{Z}}_{l}^{k+1}-
 \nabla_{\overline{\mathcal{Z}}_l^{k+1}}\psi(\mathbfcal{W}_l^{k+1},{\mathbfcal{Z}}_{l-1}^{k+1},\mathcal{U}^k)/{\delta}_l^{k+1}.
\end{align*}
When $l=L$,  we can solve Equation \eqref{eq: mathcal zl} by FISTA \cite{beck2009fast}.
\section{Convergence Analysis}
\label{sec:convergence}
\indent In this section, the theoretical convergence of the proposed dlADMM algorithm is analyzed. For the sake of simplicity, we only prove its convergence in the MLP problem. Similar convergence results can be derived for the GCN problem. Before we formally present the convergence results of the dlADMM algorithms, Section \ref{sec:assumption} presents necessary assumptions to guarantee the convergence of dlADMM. In Section \ref{sec:key properties}, we prove the convergence of the dlADMM algorithm to a critical point with a sublinear convergence rate $o(1/k)$.
\subsection{Assumptions}
\label{sec:assumption}
\begin{assumption}[Closed-form Solutions] There exist activation functions $a_l=f_l(z_l)$ such that Equations \eqref{eq:update overline z} and \eqref{eq:update z} have closed form solutions  $\overline{z}^{k+1}_l=\overline{h}(\overline{\textbf{W}}^{k+1}_{l
+1},\overline{\textbf{b}}^{k+1}_{l+1},\overline{\textbf{a}}^{k+1}_{l})$ and $z^{k+1}_l=h(\textbf{W}^{k+1}_l,\textbf{b}^{k+1}_l,\textbf{a}^{k+1}_{l-1})$, respectively, where $\overline{h}(\bullet)$ and $h(\bullet)$ are continuous functions. \label{ass:assumption 1}
\end{assumption}
\indent This assumption can be satisfied by commonly used activation functions such as ReLU and leaky ReLU. For example, for the ReLU function $a_l=\max(z_l,0)$, Equation \eqref{eq:update z} has the following solution: 
\begin{align*}
    z^{k+1}_l=
    \begin{cases} \min(W^{k+1}_{l}a^{k+1}_{l-1}+b^{k+1}_l,0)&z^{k+1}_l\leq0\\
    \max((W^{k+1}_{l}a^{k+1}_{l-1}+b^{k+1}_l+\overline{a}^{k+1}_l)/2,0)&z^{k+1}_l\geq0\\
    \end{cases}
    .
\end{align*}
\begin{assumption}[Objective Function] $F (\textbf{W}, \textbf{b},\textbf{z}, \textbf{a})$ is coercive over the nonempty set $G = \{(\textbf{W}, \textbf{b},\textbf{z}, \textbf{a}) : z_L - W_La_{L-1} -
b_L = 0\}$. In other words, $F (\textbf{W}, \textbf{b},\textbf{z}, \textbf{a}) \rightarrow \infty$ if $(\textbf{W}, \textbf{b},\textbf{z}, \textbf{a})\in G$ and
$\Vert(\textbf{W}, \textbf{b},\textbf{z},\textbf{ a})\Vert \rightarrow \infty$. Moreover, $R(z_l;y)$ is Lipschitz differentiable with Lipschitz constant $H\geq 0$.\label{ass:assumption 2}
\end{assumption}
The Assumption \ref{ass:assumption 2} is mild enough for most common loss functions to satisfy. For example, the cross-entropy and square loss are Lipschitz differentiable.
\subsection{Key Properties}
\label{sec:key properties}
\indent We present the main convergence result of the proposed dlADMM algorithm in this section.  Specifically, as long as Assumptions \ref{ass:assumption 1}-\ref{ass:assumption 2} hold, then Properties \ref{pro:property 1}-\ref{pro:property 3} are satisfied, which are important to prove the global convergence of the proposed dlADMM algorithm. The proof details are included in the supplementary materials.
\begin{property}[Boundness]
\label{pro:property 1}
If $\rho > 2H$, then $\{\textbf{W}^k,\textbf{b}^k
,\textbf{z}^k, \textbf{a}^k,u^k\}$ is bounded, and $L_\rho(\textbf{W}^k,\textbf{b}^k
,\textbf{z}^k, \textbf{a}^k,u^k)$ is lower bounded.
\end{property}
Property \ref{pro:property 1} concludes that all variables and the value of $L_\rho$ have
lower bounds. It is proven under Assumptions \ref{ass:assumption 1} and \ref{ass:assumption 2}, and its proof
can be found in the supplementary materials.
\begin{property}[Sufficient Descent]
\label{pro:property 2}
If $\rho>2H$ so that $C_1=\rho/2-H/2-H^2/\rho>0$, then there exists $C_2=\min(\nu/2,C_1,\{\overline{\theta}^{k+1}_l\}_{l=1}^L,\{{\theta}^{k+1}_l\}_{l=1}^L,\{\overline{\tau}^{k+1}_l\}_{l=1}^{L-1},\{\tau^{k+1}_l\}_{l=1}^{L-1})$ such that
\begin{align}
\nonumber &L_\rho(\textbf{W}^k,\textbf{b}^k,\textbf{z}^k,\textbf{a}^k,u^k)-L_\rho(\textbf{W}^{k+1},\textbf{b}^{k+1},\textbf{z}^{k+1},\textbf{a}^{k+1},u^{k+1})\\\nonumber &\geq C_2(\sum\nolimits_{l=1}^L(\Vert \overline{W}_l^{k+1}-W_l^k\Vert^2_2+\Vert {W}_l^{k+1}-\overline{W}_l^{k+1}\Vert^2_2\\\nonumber&+\Vert \overline{b}_l^{k+1}-b_l^k\Vert^2_2+\Vert {b}_l^{k+1}-\overline{b}_l^{k+1}\Vert^2_2)\\&\nonumber +\sum\nolimits_{l=1}^{L-1}(\Vert \overline{a}_l^{k+1}-a_l^k\Vert^2_2+\Vert {a}_l^{k+1}-\overline{a}_l^{k+1}\Vert^2_2)\\&+\Vert \overline{z}^{k+1}_L-{z}^{k}_L\Vert^2_2+\Vert z^{k+1}_L-\overline{z}^{k+1}_L\Vert^2_2)\label{eq: property2}.
\end{align}
\end{property}
\indent Property \ref{pro:property 2} depicts the  monotonic decrease of the objective value during iterations. The proof of Property \ref{pro:property 2} is detailed in the supplementary materials. 
\begin{property}[Subgradient Bound]
\label{pro:property 3}
There exist a constant $C>0$ and $g\in \partial L_\rho(\textbf{W}^{k+1},\textbf{b}^{k+1},\textbf{z}^{k+1},\textbf{a}^{k+1})$ such that
\begin{align}
    \nonumber\Vert g\Vert &\leq C(\Vert \textbf{W}^{k+1}\!-\!\overline{\textbf{W}}^{k+1}\Vert+\Vert\textbf{b}^{k+1}\!-\!\overline{\textbf{b}}^{k+1}\Vert\\&+\Vert\textbf{z}^{k+1}\!-\!\overline{\textbf{z}}^{k+1}\Vert+\Vert\textbf{a}^{k+1}\!-\!\overline{\textbf{a}}^{k+1}\Vert+\Vert \textbf{z}^{k+1}\!-\!\textbf{z}^k\Vert)\label{eq: property3}.
\end{align}
\end{property}
\indent Property \ref{pro:property 3} ensures that the subgradient of the objective function is bounded by variables. The proof of Property \ref{pro:property 3} requires Property \ref{pro:property 1} and the proof is elaborated in the supplementary materials. Now the global convergence of the dlADMM algorithm is presented. The following theorem states that Properties \ref{pro:property 1}-\ref{pro:property 3} are guaranteed.
\begin{theorem}
\label{thero: theorem 1}
For any $\rho>2H$, if Assumptions \ref{ass:assumption 1} and \ref{ass:assumption 2} are satisfied, then Properties \ref{pro:property 1}-\ref{pro:property 3} hold. 
\end{theorem}
\begin{proof}
This theorem can be concluded by the proofs in the supplementary materials.
\end{proof}
\indent The next theorem presents the global convergence of the dlADMM algorithm.
\begin{theorem} [Global Convergence]
\label{thero: theorem 2}
If $\rho>2H$, then for the variables $(\textbf{W},\textbf{b},\textbf{z},\textbf{a},u)$ in Problem \ref{prob:problem 2}, starting from any $(\textbf{W}^0,\textbf{b}^0,\textbf{z}^0,\textbf{a}^0,u^0)$, it has at least a limit point $(\textbf{W}^*,\textbf{b}^*,\textbf{z}^*,\textbf{a}^*,u^*)$, and any limit point $(\textbf{W}^*,\textbf{b}^*,\textbf{z}^*,\textbf{a}^*,u^*)$ is a critical point of Problem \ref{prob:problem 2}. That is, $0\in \partial L_\rho(\textbf{W}^*,\textbf{b}^*,\textbf{z}^*,\textbf{a}^*,u^*)$. Or equivalently, 
\begin{align*}
    & z^*_L=W^*_La^*_{L-1}+b^*_L,\\
    & 0\in\partial_{\textbf{W}^*} L_\rho(\textbf{W}^*,\textbf{b}^*,\textbf{z}^*,\textbf{a}^*,u^*),
    & \nabla_{\textbf{b}^*} L_\rho(\textbf{W}^*,\textbf{b}^*,\textbf{z}^*,\textbf{a}^*,u^*)=0,\\
    & 0\in\partial_{\textbf{z}^*} L_\rho(\textbf{W}^*,\textbf{b}^*,\textbf{z}^*,\textbf{a}^*,u^*),
    & \nabla_{\textbf{a}^*} L_\rho(\textbf{W}^*,\textbf{b}^*,\textbf{z}^*,\textbf{a}^*,u^*)=0.
\end{align*}
\end{theorem}
\begin{proof}
Because $(\textbf{W}^k,\textbf{b}^k,\textbf{z}^k,\textbf{a}^k,u^k)$  is bounded, there exists a subsequence  $(\textbf{W}^s,\textbf{b}^s,\textbf{z}^s,\textbf{a}^s,\textbf{u}^s)$ such that $(\textbf{W}^s,\textbf{b}^s,\textbf{z}^s,\textbf{a}^s,u^s)\rightarrow (\textbf{W}^*,\textbf{b}^*,\textbf{z}^*,\textbf{a}^*,u^*)$ where $(\textbf{W}^*,\textbf{b}^*,\textbf{z}^*,\textbf{a}^*,u^*)$ is a limit point.
By Properties \ref{pro:property 1} and \ref{pro:property 2}, $L_\rho(\textbf{W}^k,\textbf{b}^k,\textbf{z}^k,\textbf{a}^k,u^k)$ is non-increasing and  lower bounded and hence converges.  By Property \ref{pro:property 2}, we prove that $\Vert \overline{\textbf{W}}^{k+1}-\textbf{W}^{k}\Vert\rightarrow 0$, $\Vert \overline{\textbf{b}}^{k+1}-\textbf{b}^{k}\Vert\rightarrow 0$, $\Vert \overline{\textbf{a}}^{k+1}-\textbf{a}^{k}\Vert\rightarrow 0$, $\Vert \textbf{W}^{k+1}-\overline{\textbf{W}}^{k+1}\Vert\rightarrow 0$, $\Vert \textbf{b}^{k+1}-\overline{\textbf{b}}^{k+1}\Vert\rightarrow 0$, and $\Vert \textbf{a}^{k+1}-\overline{\textbf{a}}^{k+1}\Vert\rightarrow 0$, as $k\rightarrow \infty$ . Therefore $\Vert \textbf{W}^{k+1}-{\textbf{W}}^{k}\Vert\rightarrow 0$, $\Vert \textbf{b}^{k+1}-{\textbf{b}}^{k}\Vert\rightarrow 0$, and $\Vert \textbf{a}^{k+1}-{\textbf{a}}^{k}\Vert\rightarrow 0$, as $k\rightarrow \infty$. Moreover, from Assumption \ref{ass:assumption 1}, we know that $\overline{\textbf{z}}^{k+1}\rightarrow \textbf{z}^{k}$ and $\textbf{z}^{k+1}\rightarrow \overline{\textbf{z}}^{k+1}$  as $k\rightarrow \infty$. Therefore, $\textbf{z}^{k+1}\rightarrow \textbf{z}^{k}$.  We infer there exists $g^k\in \partial L_\rho(\textbf{W}^k,\textbf{b}^k,\textbf{z}^k,\textbf{a}^k,u^k)$ such that $\Vert g^k\Vert \rightarrow 0$ as $k\rightarrow \infty$ based on Property \ref{pro:property 3}. Specifically, $\Vert g^s\Vert \rightarrow 0$ as $s\rightarrow \infty$. According to the definition of general subgradient (Defintion 8.3 in \cite{rockafellar2009variational}), we have $0\in \partial L_\rho(\textbf{W}^*,\textbf{b}^*,\textbf{z}^*,\textbf{a}^*,u^*)$. In other words, the limit point $(\textbf{W}^*,\textbf{b}^*,\textbf{z}^*,\textbf{a}^*,u^*)$ is a critical point of $L_\rho$ defined in Equation \eqref{eq: Lagrangian}.
\end{proof}
\indent Theorem \ref{thero: theorem 2} shows that our dlADMM algorithm converges globally for sufficiently large $\rho$, which is consistent with previous literature \cite{wang2015global,kiaee2016alternating}. The next theorem shows that the dlADMM converges globally with a sublinear convergence rate $o(1/k)$.
\begin{theorem}[Convergence Rate]
For a sequence\\ $(\textbf{W}^k,\textbf{b}^k,\textbf{z}^k,\textbf{a}^k,u^k)$,  define $c_k=\min\nolimits_{0\leq i\leq k}(\sum\nolimits_{l=1}^L(\Vert \overline{W}_l^{i+1}-W_l^i\Vert^2_2+\Vert {W}_l^{i+1}-\overline{W}_l^{i+1}\Vert^2_2+\Vert \overline{b}_l^{i+1}-b_l^i\Vert^2_2+\Vert {b}_l^{i+1}-\overline{b}_l^{i+1}\Vert^2_2) +\sum\nolimits_{l=1}^{L-1}(\Vert \overline{a}_l^{i+1}-a_l^i\Vert^2_2+\Vert {a}_l^{i+1}-\overline{a}_l^{i+1}\Vert^2_2)+\Vert \overline{z}^{i+1}_L-{z}^{i}_L\Vert^2_2+\Vert z^{i+1}_L-\overline{z}^{i+1}_L\Vert^2_2)$, then the convergence rate of $c_k$ is $o(1/k)$.
\label{thero: theorem 3}
\end{theorem}
\begin{proof}
The proof of this theorem is included in the supplementary materials. 
\end{proof}

\section{Experiments}
\label{sec:experiment}
In this section, the proposed dlADMM is evaluated on MLP models and GCN models using seven benchmark datasets. Effectiveness, efficiency, and convergence of the proposed dlADMM algorithm are compared with state-of-the-art optimizers. The experiments on MLP models were conducted on a 64-bit machine with Intel(R) Xeon processor and GTX1080Ti GPU, and the experiments on GCN models were conducted on a 64-bit machine with Xeon processor and NVIDIA Quadro RTX 5000 GPU.
\subsection{Experiments on MLP Models}
\begin{table}[]
    \centering
    \begin{tabular}{c|c|c|c|c}
    \hline\hline
         Dataset&Feature\#&Class\#& \tabincell{c}{Training\\ Sample\#}& \tabincell{c}{Test\\ Sample\#} \\
         \hline
         MNIST&784&10&55,000&10,000\\\hline Fashion MNIST&784&10&60,000&10,000\\\hline  \hline
    \end{tabular}
    \caption{Statistics of two benchmark datasets on MLP models.}
    \label{tab:MLP dataset}
\end{table}
\subsubsection{Experiment Setup}
\indent For MLP models, two benchmark datasets were used for performance evaluation:  MNIST \cite{lecun1998gradient} and  Fashion MNIST \cite{xiao2017fashion}. The MNIST dataset has ten classes of handwritten-digit images, which was firstly introduced by LeCun et al. in 1998 \cite{lecun1998gradient}. Unlike the MNIST dataset, the Fashion MNIST dataset has ten classes of assortment images on the website of Zalando, which is Europe's largest online fashion platform \cite{xiao2017fashion}. Their statistics are available in Table \ref{tab:MLP dataset}. We set up a network architecture which contained two hidden layers with $1,000$ neurons each. The ReLU was used for the activation function for both network structures. The loss function was set as the deterministic cross-entropy loss. $\nu$ and $\rho$ were both set as $10^{-6}$. The number of iteration (i.e. epoch) was set to $200$.\\
\indent GD and its variants and ADMM serve as comparison methods. GD-based methods trained models in a full-batch fashion. All parameters were chosen based on the optimal training accuracy. All comparison methods are outlined as follows: \\
\indent 1. Gradient Descent (GD) \cite{bottou2010large}. The GD updates parameters based on their gradients. The learning rate of GD was set to $10^{-6}$.\\
\indent 2. Adaptive gradient (Adagrad) \cite{duchi2011adaptive}. Rather than fixing the learning rate, Adagrad adapts the learning rate to some hyperparameter. The learning rate was set to $10^{-3}$.\\
\indent 3. Adaptive learning rate (Adadelta) \cite{zeiler2012adadelta}. The Adadelta is proposed to overcome the sensitivity to hyperparameter selection. The learning rate was set to $0.1$.\\
\indent 4. Adaptive momentum estimation (Adam) \cite{kingma2014adam}. The Adam estimates the first and second momentums in order to correct the biased gradient. The learning rate of Adam was set to $10^{-3}$.\\
\indent 5. Alternating Direction Method of Multipliers (ADMM) \cite{taylor2016training}. ADMM is a powerful convex optimization method because it can split an objective function into a series of subproblems, which are coordinated to get global solutions. It is scalable to large-scale datasets and supports parallel computations. The $\rho$ of ADMM was set to $1$.
\subsubsection{Experimental Results}
\begin{figure}[h]
  \centering
\begin{minipage}
{0.49\linewidth}
\centerline{\includegraphics[width=\columnwidth]
{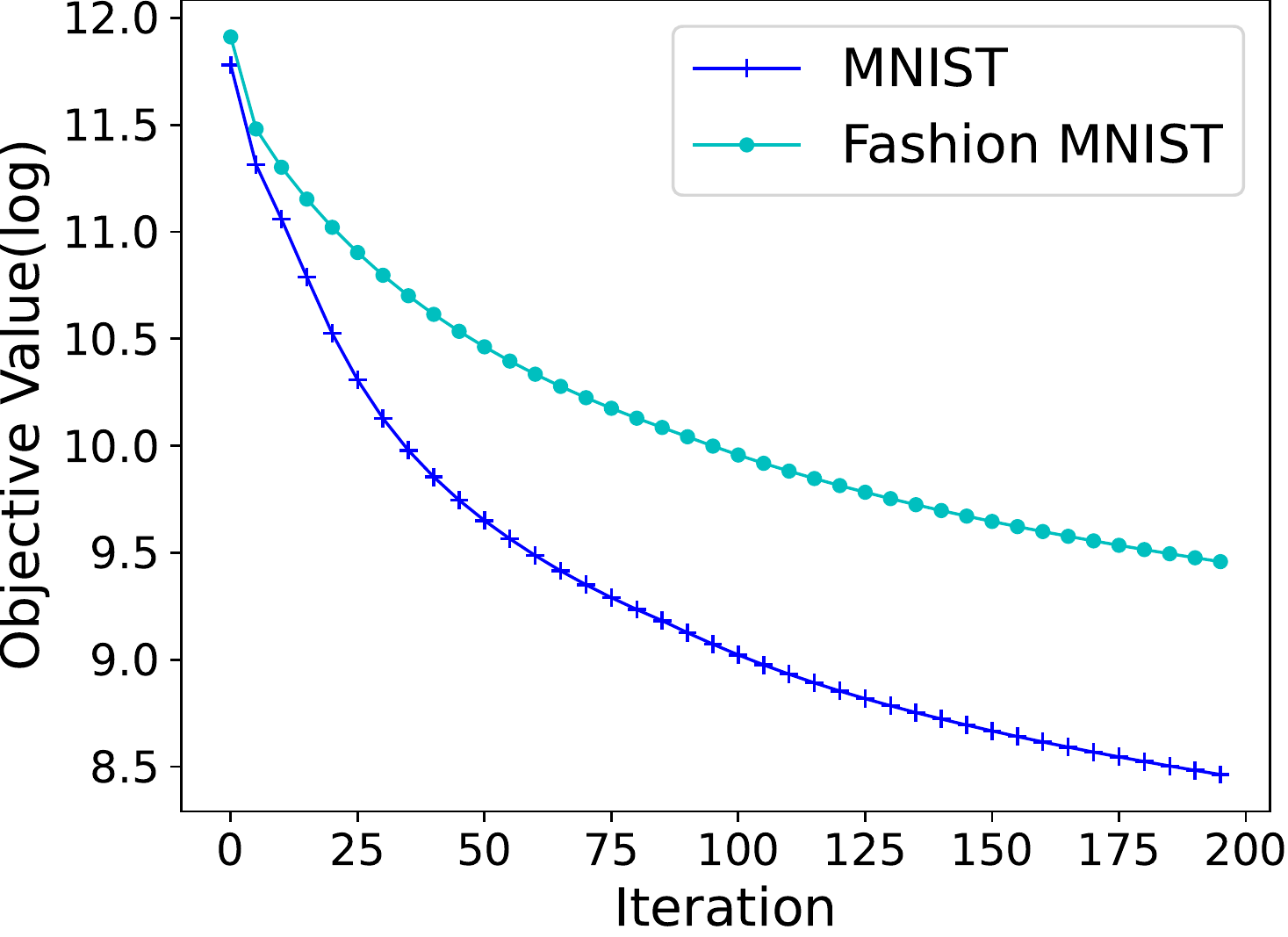}}
\centerline{(a). Objective value}
\end{minipage}
\hfill
\begin{minipage}
{0.49\linewidth}
\centerline{\includegraphics[width=\columnwidth]
{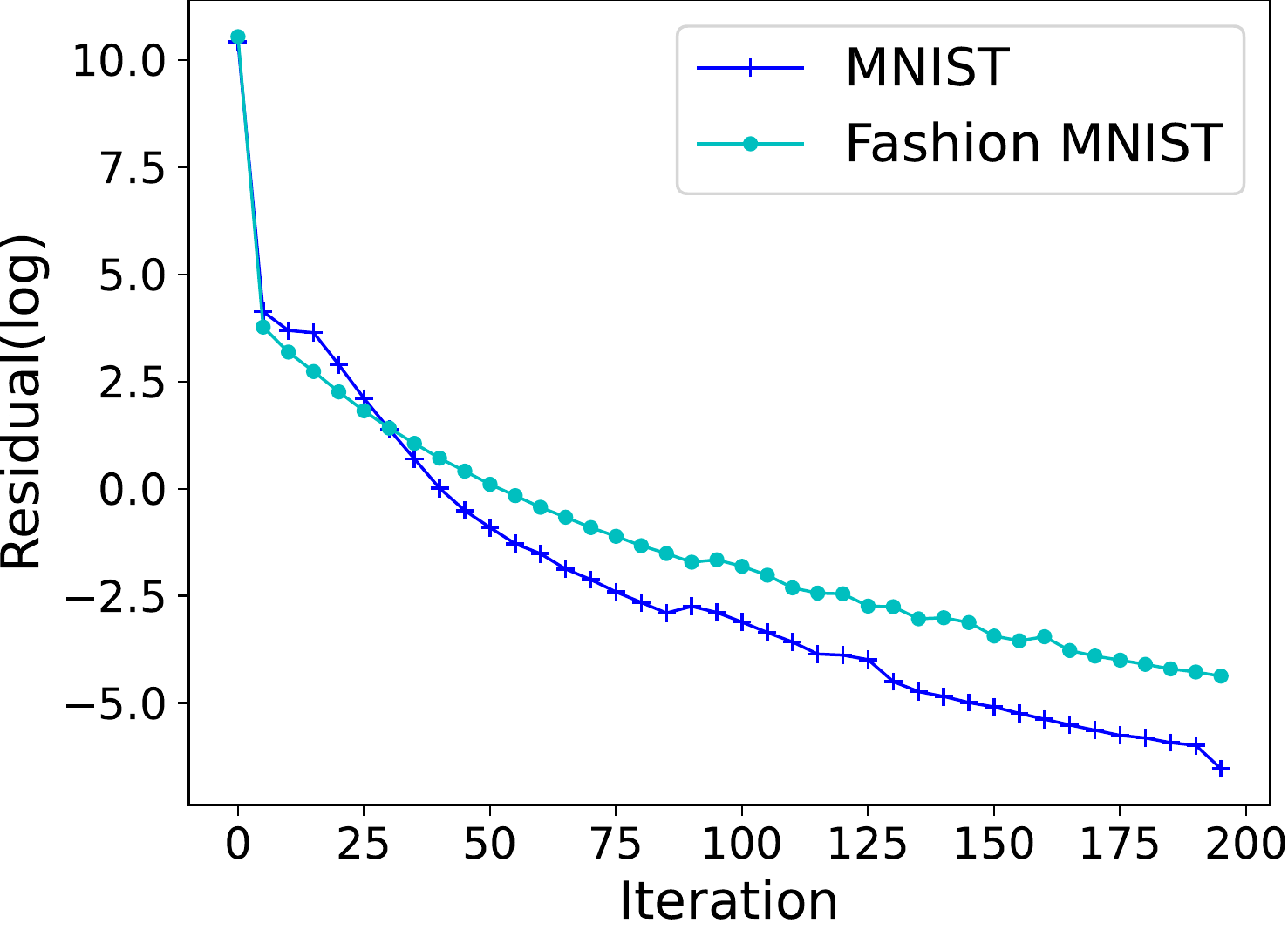}}
\centerline{(b).  Residual}
\end{minipage}
  \caption{Convergence curves of the proposed dlADMM on two datasets when $\rho=1$.}
  \label{fig:convergence}
\end{figure}
\begin{figure}[h]
  \centering
\begin{minipage}
{0.49\linewidth}
\centerline{\includegraphics[width=\textwidth]
{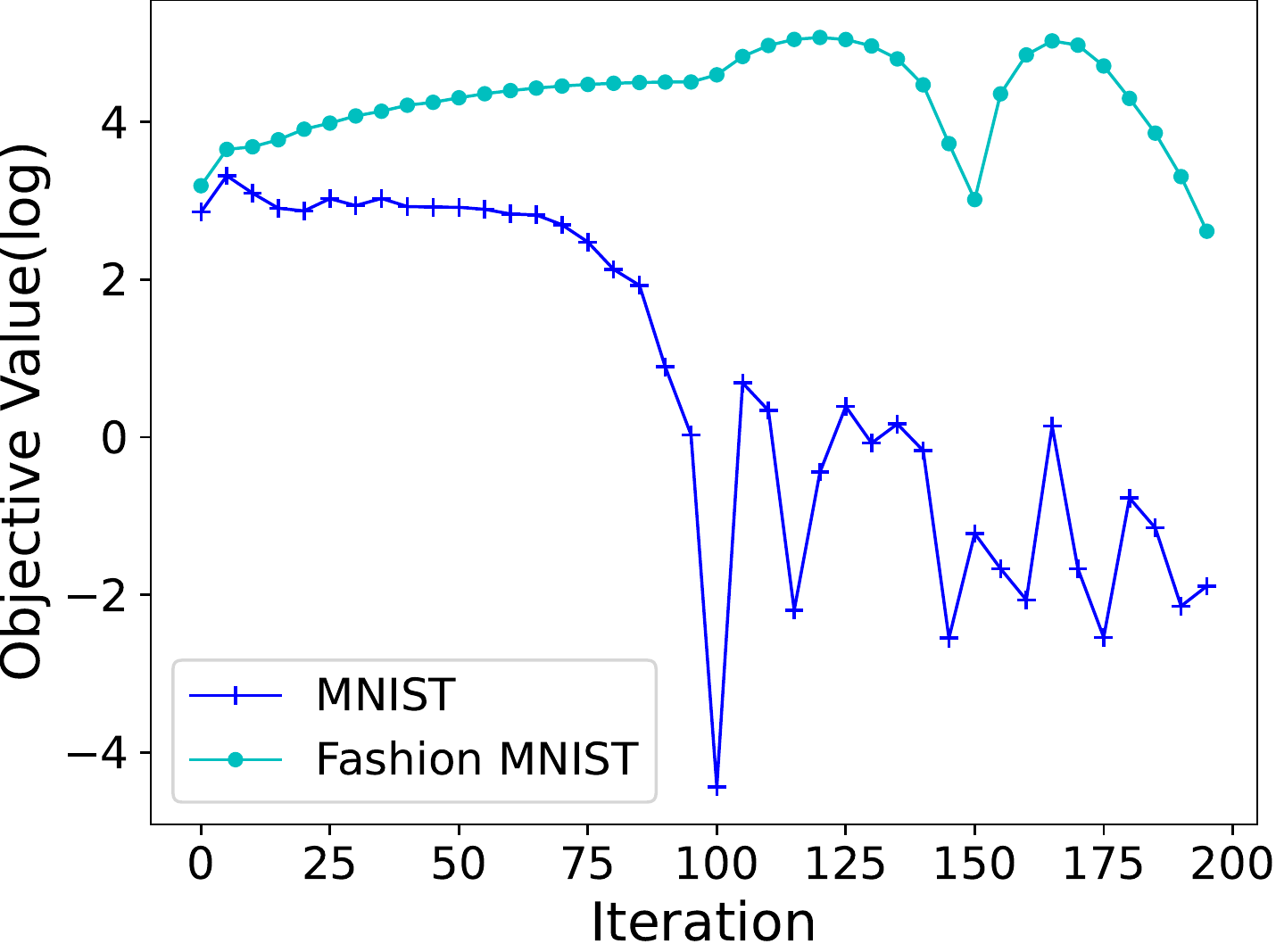}}
\centerline{(a). Objective value}
\end{minipage}
\hfill
\begin{minipage}
{0.49\linewidth}
\centerline{\includegraphics[width=\textwidth]
{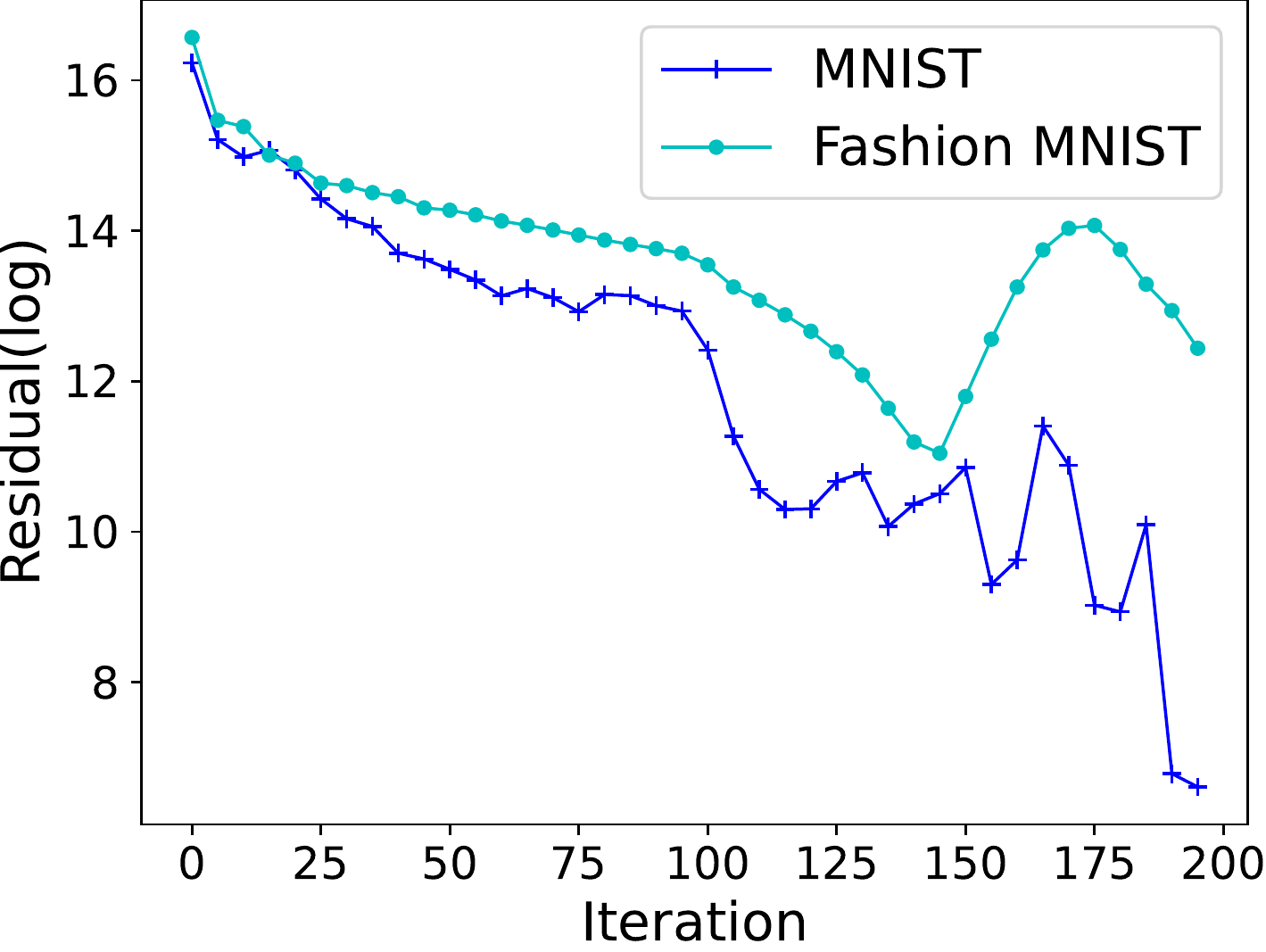}}
\centerline{(b). Residual}
\end{minipage}
  \caption{Divergence curves of the proposed dlADMM on two datasets when $\rho=10^{-6}$.}
  \label{fig:divergence}
\end{figure}
\begin{figure}[h]
  \centering
\begin{minipage}
{0.49\linewidth}
\centerline{\includegraphics[width=\textwidth]
{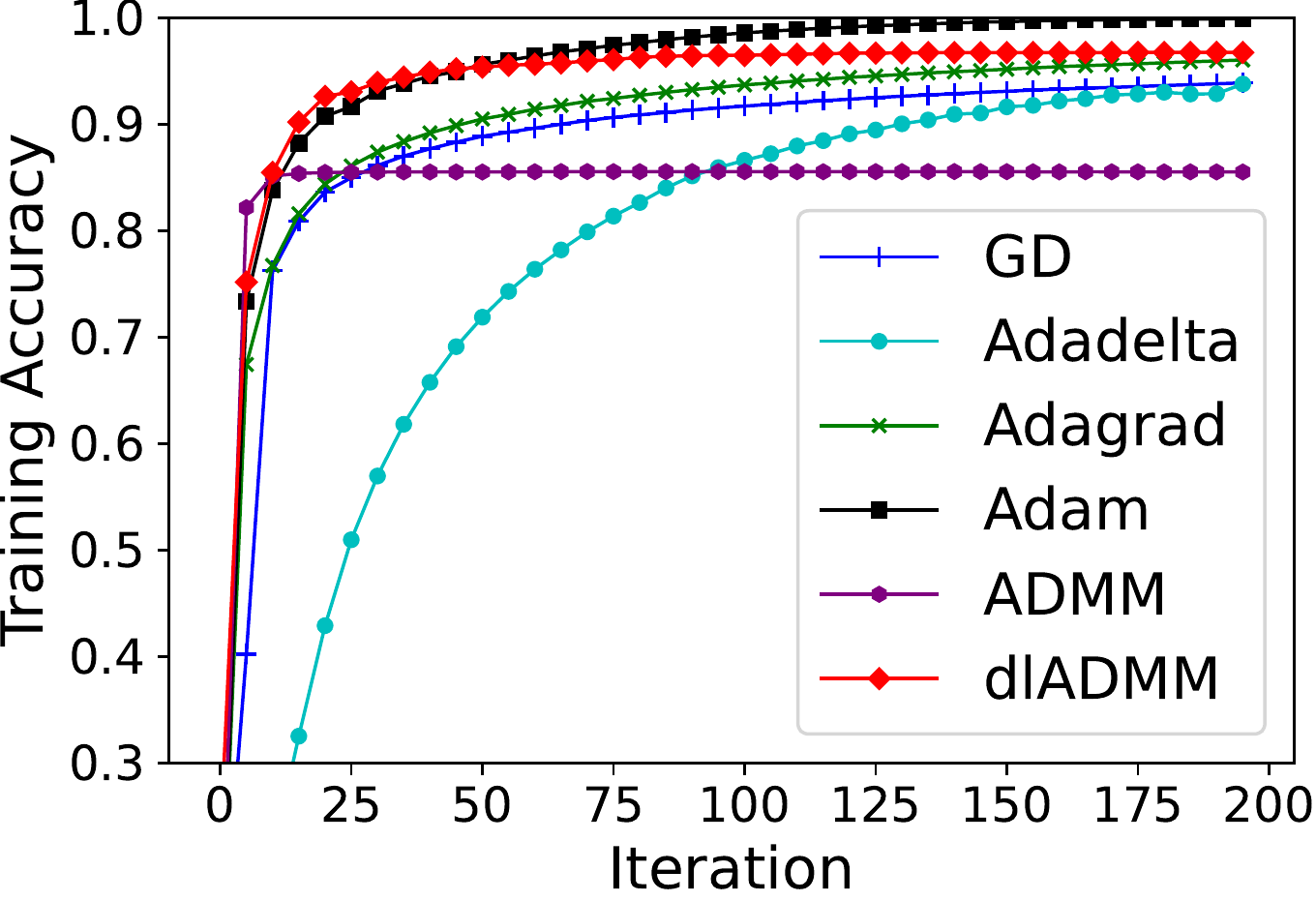}}
\centerline{(a). MNIST}
\end{minipage}
\hfill
\begin{minipage}
{0.49\linewidth}
\centerline{\includegraphics[width=\textwidth]
{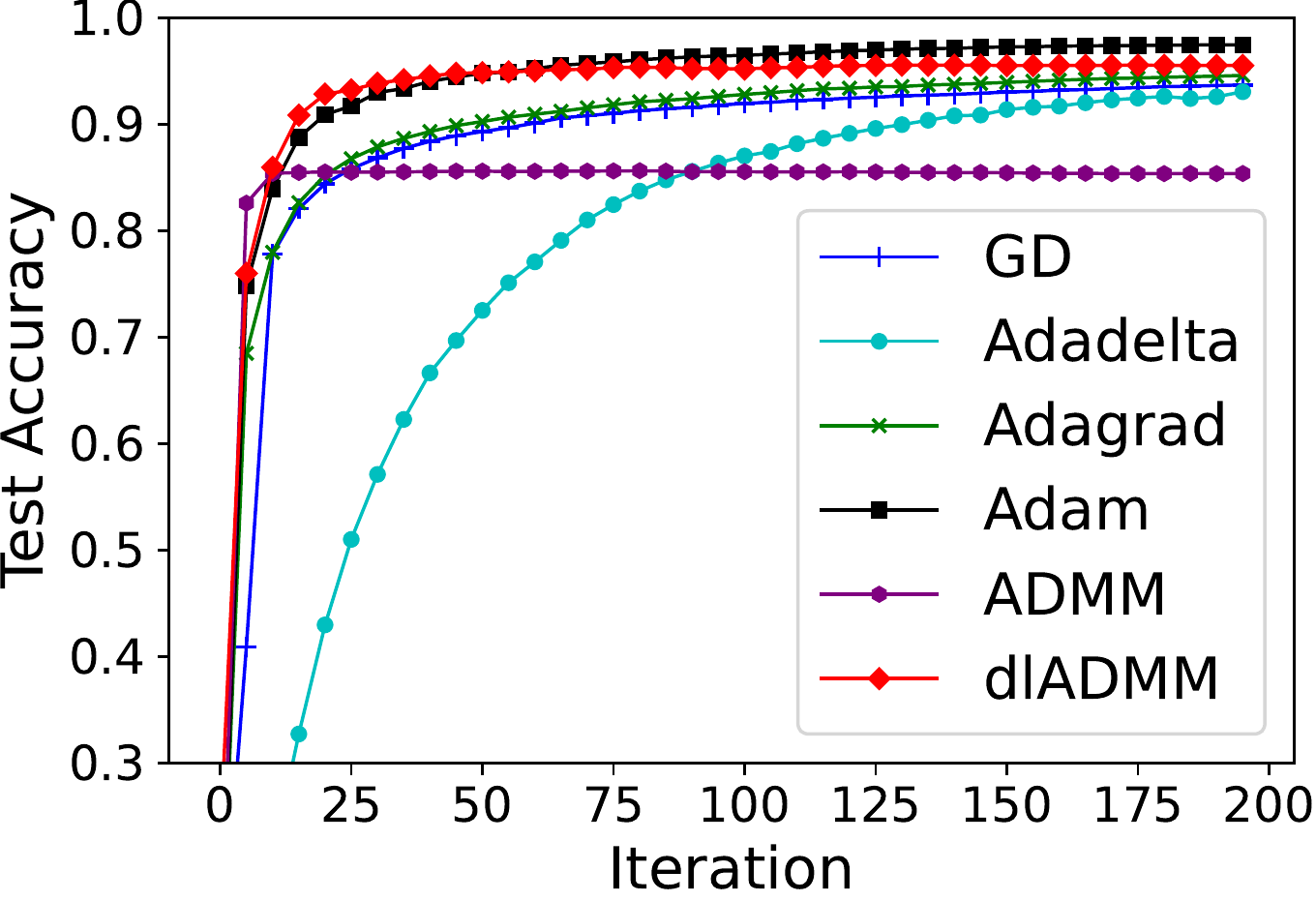}}
\centerline{(b). Fashion MNIST}
\end{minipage}
  \caption{Test Performance of all methods on two datasets: the dlADMM outperformed most of the comparison methods.}
  \label{fig:MLP test}
\end{figure}
\begin{table}[!hbp]

\centering
\begin{tabular}{r|r|r|r|r|r}
\hline\hline
\multicolumn{6}{c}{MNIST dataset: From 200 to 1,000 neurons}\\
\hline
\diagbox{$\rho$}{neuron} &200&400 &600  &800 &1000 \\\hline
$10^{-6}$&1.9025&2.7750&3.6615&4.5709&5.7988\\\hline
 $10^{-5}$&2.8778&4.6197&6.3620&8.2563&10.0323\\\hline
 $10^{-4}$&2.2761&3.9745&5.8645&7.6656&9.9221\\\hline
 $10^{-3}$&2.4361&4.3284&6.5651&8.7357&11.3736\\\hline
 $10^{-2}$&2.7912&5.1383&7.8249&10.0300&13.4485\\
\hline\hline
\multicolumn{6}{c}{Fashion MNIST dataset: From 200 to 1,000 neurons }\\
\hline
\diagbox{$\rho$}{neuron} &200&400&600&800&1000\\
\hline
$10^{-6}$&2.0069&2.8694&4.0506&5.1438&6.7406\\\hline
$10^{-5}$&3.3445&5.4190&7.3785&9.0813&11.0531\\\hline
$10^{-4}$&2.4974&4.3729&6.4257&8.3520&10.0728\\\hline
$10^{-3}$&2.7108&4.7236&7.1507&9.4534&12.3326\\\hline
$10^{-2}$&2.9577&5.4173&8.2518&10.0945&14.3465
\\\hline\hline
\end{tabular}
\caption{The relationship between running time per epoch (in second) and the number of neurons for each layer as well as value of $\rho$ when the training size was fixed: generally, the running time increased as the number of neurons and the value of $\rho$ became larger.}
\label{tab:running time 1}
\end{table}
\begin{table}[!hbp]
\centering
\begin{tabular}{r|r|r|r|r|r}
\hline\hline
\multicolumn{6}{c}{MNIST dataset: From 11,000  to 55,000 training samples}\\
\hline
\diagbox{$\rho$}{size} &11,000&22,000 &33,000  &44,000 &55,000 \\\hline
$10^{-6}$&1.0670&2.0682&3.3089&4.6546&5.7709\\\hline
$10^{-5}$&2.3981&3.9086&6.2175&7.9188&10.2741\\\hline
 $10^{-4}$&2.1290&3.7891&5.6843&7.7625&9.8843\\\hline
 $10^{-3}$&2.1295&4.1939&6.5039&8.8835&11.3368\\\hline
 $10^{-2}$&2.5154&4.9638&7.6606&10.4580&13.4021\\
\hline\hline
\multicolumn{6}{c}{Fashion MNIST dataset: From 12,000  to 60,000 training samples }\\
\hline
\diagbox{$\rho$}{size} &12,000&24,000&36,000&48,000&60,000\\
\hline
$10^{-6}$&1.2163&2.3376&3.7053&5.1491&6.7298\\\hline
$10^{-5}$&2.5772&4.3417&6.6681&8.3763&11.0292\\\hline
$10^{-4}$&2.3216&4.1163&6.2355&8.3819&10.7120\\\hline
$10^{-3}$&2.3149&4.5250&6.9834&9.5853&12.3232\\\hline
$10^{-2}$&2.7381&5.3373&8.1585&11.1992&14.2487
\\\hline\hline
\end{tabular}
\caption{The relationship between running time per epoch (in second) and the size of  training samples as well as value of $\rho$ when the number of neurons is fixed: generally, the running time increased as the training sample and the value of $\rho$ became larger.}
\label{tab:running time 2}
\end{table}
\indent Experimental results of the proposed dlADMM are analyzed against comparison methods as follows:\\
\textbf{Convergence}: Firstly, to demonstrate Theorem \ref{thero: theorem 2}, we show that our proposed dlADMM converges when $\rho$ is sufficiently large and diverges when $\rho$ is small. The convergence and divergence of dlADMM are shown in Figures \ref{fig:convergence} and \ref{fig:divergence} when $\rho=1$ and $\rho=10^{-6}$, respectively. In Figure \ref{fig:convergence}, both objective values and residuals decrease monotonically on two datasets. Moreover, Figure \ref{fig:divergence} illustrates that  both objective values and residuals diverge when $\rho=10^{-6}$. Their curves fluctuate drastically on the objective values. Even though there is a decreasing trend for residuals, they still fluctuate irregularly.
\\
\textbf{Performance}: Figure \ref{fig:MLP test} shows the test accuracy of our proposed dlADMM and all comparison methods on two datasets. Overall,  our proposed dlADMM outperforms most of them on both training accuracy and test accuracy on two datasets. Specifically, the curves of our proposed dlADMM soar to $0.8$ at the early stage and then rise steadily towards more than $0.9$. The curves of the most GD-related methods, such as GD, Adadelta, and Adagrad, climb more slowly than our proposed dlADMM. The curves of the proposed ADMM also rocket to around $0.8$, but decrease slightly since then. Only the  Adam performs better than the proposed dlADMM marginally by around $4\%$.\\
\textbf{Efficiency Analysis}: Finally, the relationship between running time per epoch of our proposed dlADMM and three potential factors, namely, the value of $\rho$, the size of training samples, and the number of neurons was explored. The running time was calculated by the average of 200 iterations. \\
\indent Firstly, when the training size was fixed, the running time per epoch on two datasets is shown in Table \ref{tab:running time 1}. The number of neurons for each layer ranged from 200 to 1,000, with an increase of 200 each time. The value of $\rho$ ranged from $10^{-6}$ to $10^{-2}$, with being multiplied by 10  each time. Generally, the running time increases with the increase of the number of neurons and the value of $\rho$. However, there are a few exceptions: for example, when there are 200 neurons on the MNIST dataset, and $\rho$ increases from $10^{-5}$ to $10^{-4}$, the running time per epoch drops from  $2.8778$ seconds to $2.2761$ seconds.\\
\indent Secondly, we fixed the number of neurons for each layer as $1,000$. The relationship between running time per epoch, the training size, and the value of $\rho$ is shown in Table \ref{tab:running time 2}. The value of $\rho$ ranged from $10^{-6}$ to $10^{-2}$, with being multiplied by 10  each time. The training size of the MNIST dataset ranged from $11,000$ to $55,000$, with an increase of $11,000$ each time. The training size of the Fashion MNIST dataset ranged from $12,000$ to $60,000$, with an increase of $12,000$ each time. Similar to Table \ref{tab:running time 2}, the running time increases generally as the training sample and the value of $\rho$ become larger.
\subsection{Experiments on GCN Models}
\begin{table}
    \scriptsize
    \centering
    \begin{tabular}{c|c|c|c|c|c }
    \hline\hline
         Dataset&\tabincell{c}{ Node\#}&\tabincell{c}{ Edge\#}&\tabincell{c}{ Class\#}&
{Feature\#}&\tabincell{c}{ Label\\Rate}\\\hline 
\tabincell{c}{Cora} & 2708 & 10556 & 7 & 1433 &5.17\%\\\hline 
\tabincell{c}{PubMed} & 19717  & 88648 &  3  & 500 &0.30\%\\\hline 
\tabincell{c}{Citeseer}& 3327 & 9104  & 6 & 3703 &3.61\%\\\hline 
\tabincell{c}{Coauthor  \\ CS}& 18333 & 163788 & 15 & 6805 &8.18\%\\\hline 
\tabincell{c}{Coauthor  \\ Physics}& 34493 & 495924 & 5 & 8415 &1.45\%\\\hline\hline
    \end{tabular}
    \caption{Statistics of five benchmark datasets on GCN models.}
    \label{tab:gcn dataset}
        \vspace{-1cm}
\end{table}
\subsubsection{Experiment Setup}
\indent For GCN models, five benchmark datasets were used for performance evaluation, which are shown in Table \ref{tab:gcn dataset}. All of them are outlined as follows:\\
1. Cora \cite{sen2008collective}. The Cora dataset consists of 2708 scientific publications classified into one of seven classes.\\
2. PubMed \cite{sen2008collective}. PubMed comprises 30M+ citations for biomedical literature that have been collected from sources such as MEDLINE, life science journals, and published online e-books.\\
3. Citeseer \cite{sen2008collective}. The Citeseer dataset was collected from the Tagged.com social network website. The original task on the dataset is to identify (i.e., classify) the spammer users based on their relational and non-relational features. \\
4. Coauthor CS and Coauthor Physics \cite{shchur2018pitfalls}. They are co-authorship graphs based on the Microsoft Academic Graph from the KDD Cup 2016 challenge 3. Here, nodes are authors, that are connected by an edge if they co-authored a paper.\\
\indent We set up a GCN model which contained one hidden layer with $128$ neurons. We choose this shallow model because the performance of GCN models deteriorates as they go deeper due to the over-smoothing problem \cite{Oono2020Graph}. The ReLU was used for the activation function. The loss function was set as the cross-entropy loss. $\nu$ and $\rho$ were both set to $10^{-3}$ based on the optimal training accuracy. The number of epoch was set to $500$.\\
\indent GD \cite{bottou2010large}, Adagrad \cite{duchi2011adaptive}, Adadelta \cite{zeiler2012adadelta}, and Adam \cite{kingma2014adam} are utilized to compare performance. The full batch dataset was used for training models. All parameters were chosen by the accuracy of the training dataset, and the learning rates for GD, Adagrad, Adadelta and Adam were set to $0.1$, $0.001$, $0.001$ and $0.01$, respectively.
\subsubsection{Experimental Results}
\begin{figure}[h]
  \centering
\begin{minipage}
{0.49\linewidth}
\centerline{\includegraphics[width=\columnwidth]
{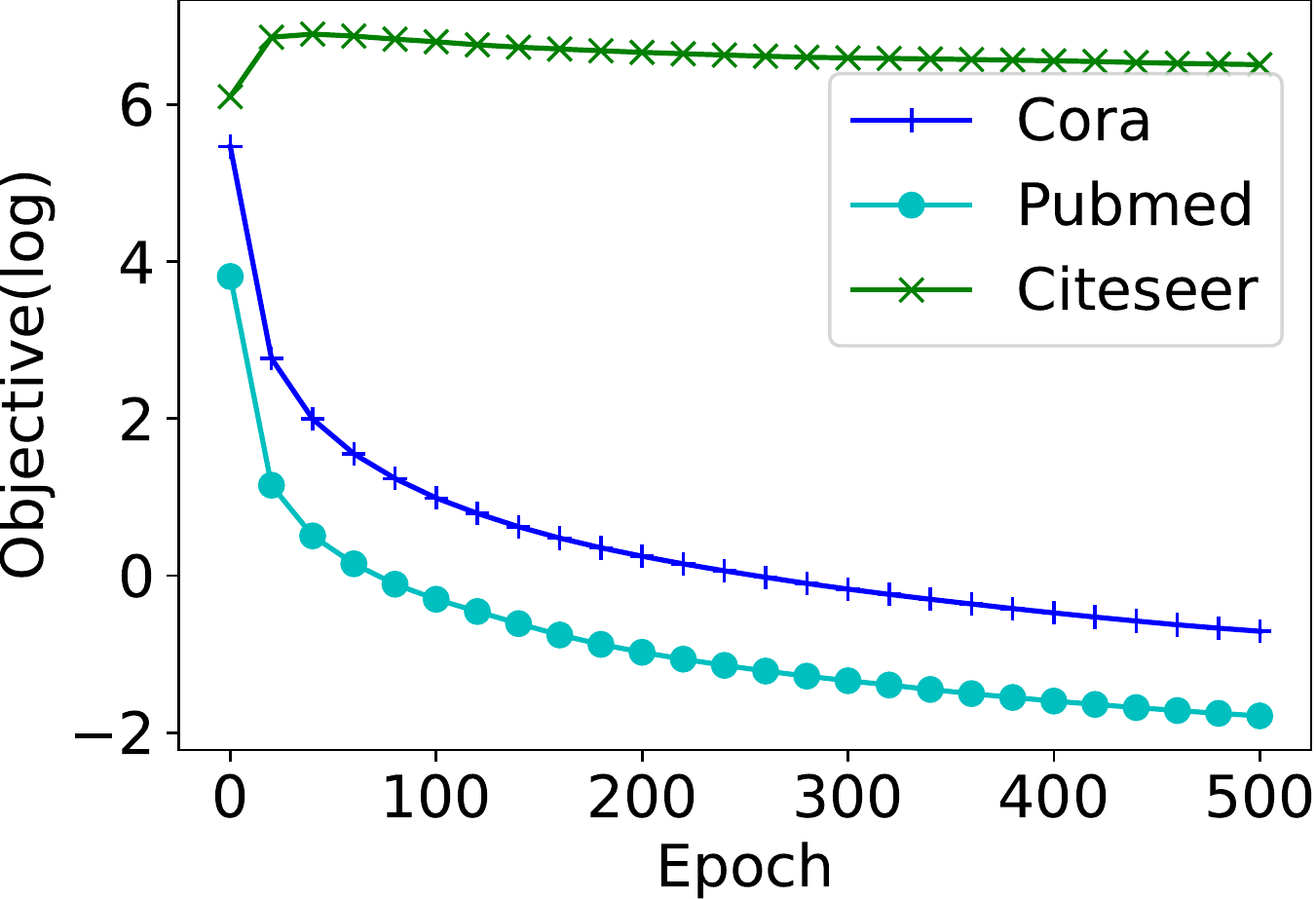}}
\centerline{(a). Objective value}
\end{minipage}
\hfill
\begin{minipage}
{0.49\linewidth}
\centerline{\includegraphics[width=\columnwidth]
{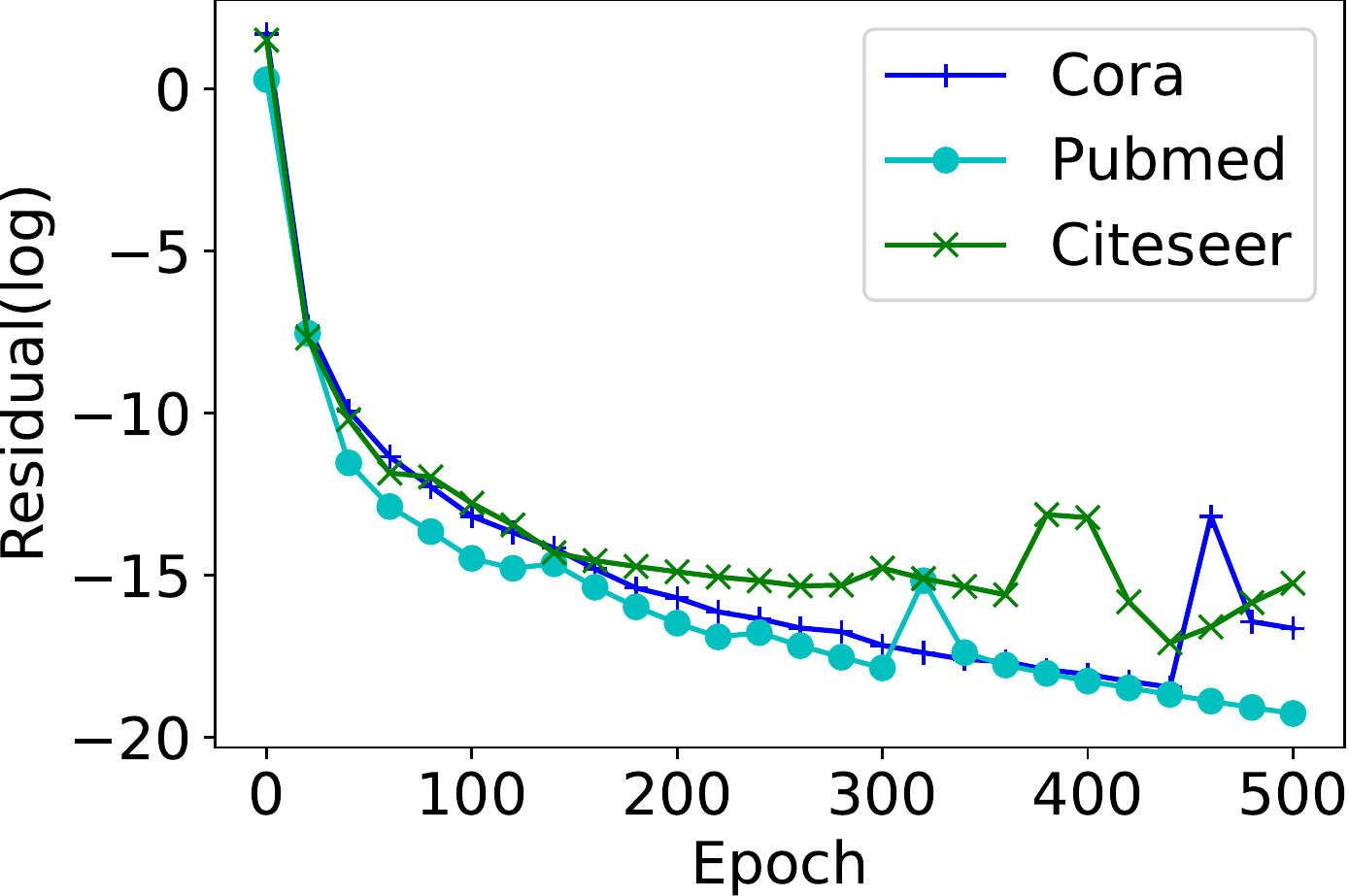}}
\centerline{(b).  Residual}
\end{minipage}
  \caption{Convergence curves of the proposed dlADMM on three datasets when $\rho=1$.}
      \vspace{-0.5cm}

  \label{fig:gcn convergence}
\end{figure}
\begin{figure}[h]
  \centering
\begin{minipage}
{0.49\linewidth}
\centerline{\includegraphics[width=\columnwidth]
{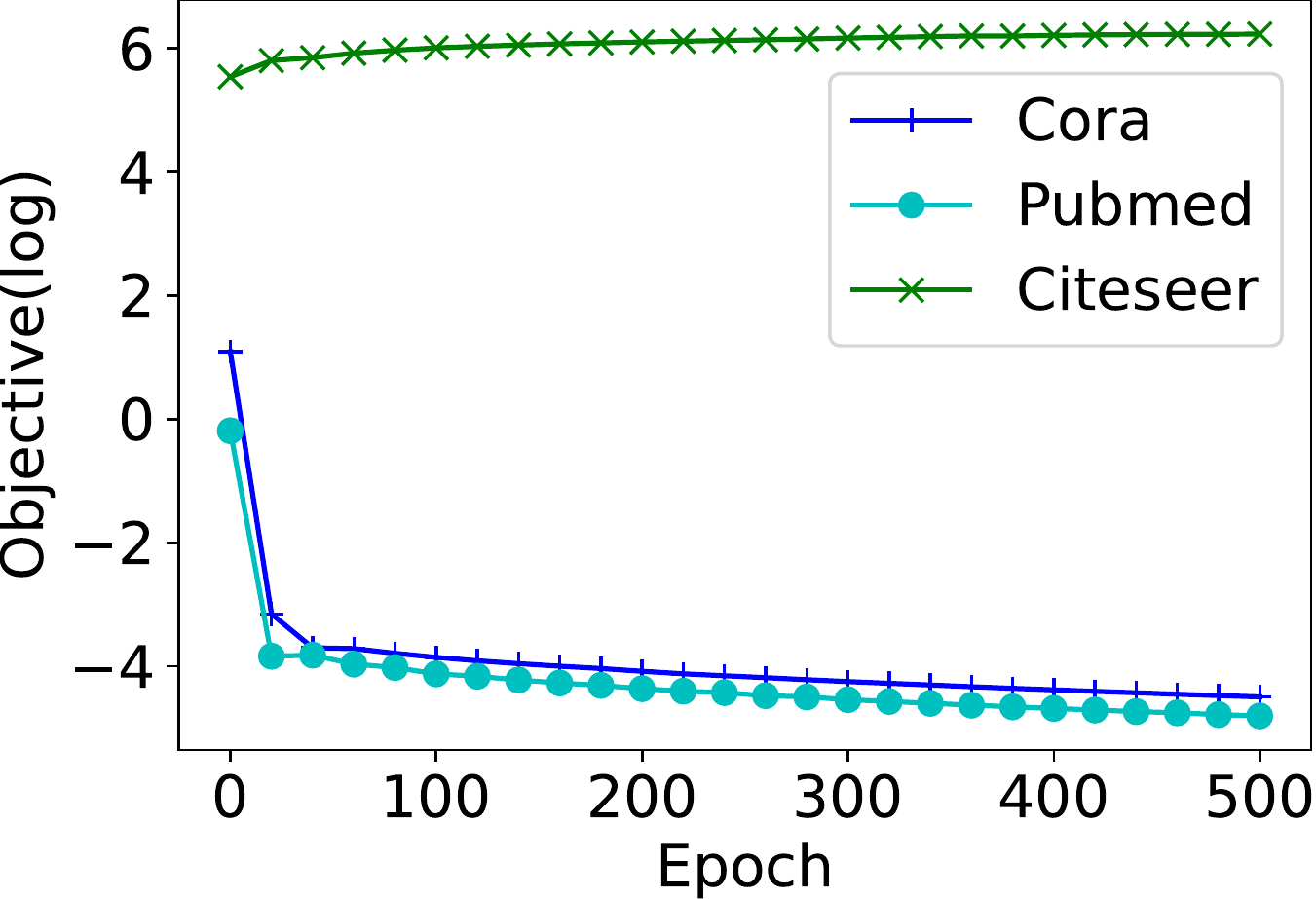}}
\centerline{(a). Objective value}
\end{minipage}
\hfill
\begin{minipage}
{0.49\linewidth}
\centerline{\includegraphics[width=\columnwidth]
{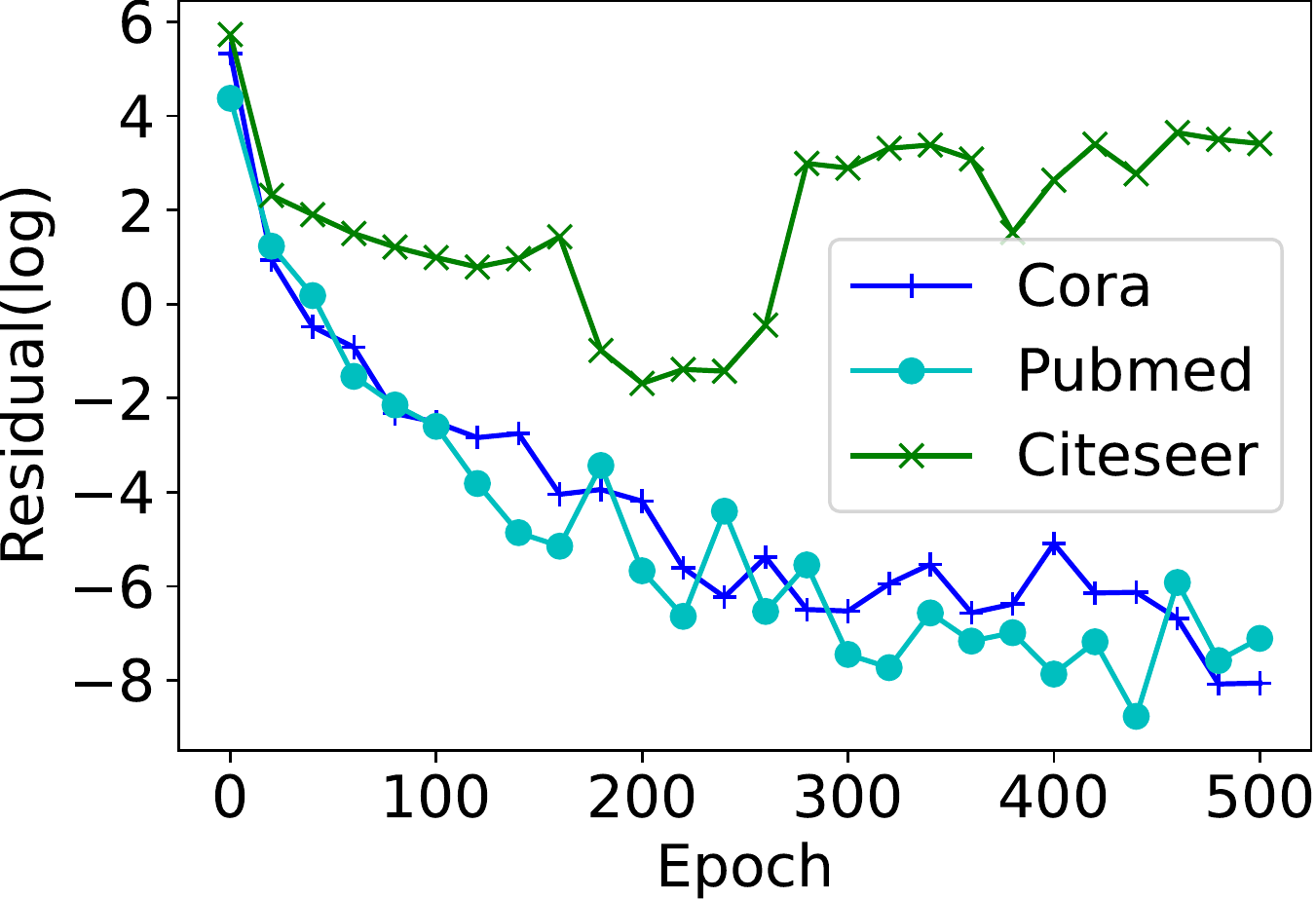}}
\centerline{(b).  Residual}
\end{minipage}
  \caption{ Curves of the dlADMM algorithm on three datasets when $\rho=10^{-3}$: it converges on the PubMed and Citeseer datasets, but diverges on the Cora dataset.}
  \label{fig:gcn divergence}
\end{figure}
\begin{table*}[]
    \centering
    \begin{tabular}{c|c|c|c|c|c}
    \hline\hline
     \diagbox{Method}{Dataset}&Cora&PubMed&Citeseer&Coauthor CS&Coauthor Physics  \\
     \hline
         Adam&$0.7814\pm 0.0034$&$0.7533\pm	0.0027$&$0.6364	\pm 0.0130$&$0.9160\pm 0.0019$&$0.9261\pm	0.0024$
\\
         
   \hline
         Adagrad&$0.7929\pm 0.0064$&$0.7429	\pm 0.0055$&$0.6498	\pm 0.0093$&$0.9169\pm 0.0026
$&$0.9352\pm 0.0009$

\\
         \hline
         GD&$0.8103\pm 0.0031$&$0.6108\pm0.0513$&$\boldsymbol{0.6986 \pm	0.0034}$&$0.9220\pm 0.0016$&$0.9336\pm	0.0011$
 \\
         \hline
   Adadelta&$0.7236\pm0.0555$&$0.6987	\pm 0.0272$&$0.6416\pm 0.0173$&$0.9138\pm 0.0028$&$0.9335\pm	0.0032$

\\\hline
         dlADMM&$\boldsymbol {0.8130\pm0.0106}$&$\boldsymbol{0.7631\pm	0.0060}$&$0.6972 \pm 0.0070$&$\boldsymbol{0.9261\pm 0.0016}$&$\boldsymbol{0.9376	\pm 0.0011}$ \\\hline\hline
    \end{tabular}
    \caption{Performance on five datasets averaged by 10 initializations: the proposed dlADMM performs the best on most of them.}
    \label{tab:gcn performance}
    \vspace{-1cm}
\end{table*}

\begin{table}[!hbp]

\centering
\begin{tabular}{r|r|r|r|r|r}
\hline\hline
\multicolumn{6}{c}{Cora: from 16 to 256 neurons}\\
\hline
\diagbox{$\rho$}{neuron} &16&32 &64  &128 &256 \\\hline
$10^{-6}$&0.1848&0.1756&0.1710&0.1806&0.1860\\\hline
 $10^{-5}$&0.2082&0.2056&0.2081&0.2121&0.2191\\\hline
 $10^{-4}$&0.2483&0.2439&0.2374&0.2495&0.2538\\\hline
 $10^{-3}$&0.2836&0.2791&0.2780&0.2877&0.2944\\\hline
 $10^{-2}$&0.2964&0.2964&0.3058&0.3210&0.3335\\\hline\hline
\multicolumn{6}{c}{PubMed: From 16 to 256 neurons}\\
\hline
\diagbox{$\rho$}{neuron} &16&32 &64  &128 &256\\
\hline
$10^{-6}$&0.1690&0.1711&0.1733&0.1740&0.1790\\\hline
$10^{-5}$&0.1894&0.1990&0.2011&0.2042&0.2087\\\hline
$10^{-4}$&0.2206&0.2280&0.2224&0.2289&0.2347\\\hline
$10^{-3}$&0.2546&0.2668&0.2707&0.2671&0.2761\\\hline
$10^{-2}$&0.2794&0.2897&0.2887&0.2948&0.3030\\\hline\hline
\multicolumn{6}{c}{Citeseer: from 16 to 256 neurons}\\
\hline
\diagbox{$\rho$}{neuron} &16&32 &64  &128 &256\\
\hline
$10^{-6}$&0.8083&0.8157&0.8247&0.8457&0.8448\\\hline
$10^{-5}$&0.9120&0.9846&0.9291&0.9331&0.9547\\\hline
$10^{-4}$&1.0560&1.0862&1.0894&1.0832&1.0784\\\hline
$10^{-3}$&1.1179&1.1366&1.1564&1.1760&1.2041\\\hline
$10^{-2}$&1.1749&1.2078&1.1609&1.2149&1.1972\\\hline\hline
\multicolumn{6}{c}{Coauthor CS: from 16 to 256 neurons}\\
\hline
\diagbox{$\rho$}{neuron} &16&32 &64  &128 &256\\
\hline
$10^{-6}$&1.1163&1.0684&1.0826&1.1041&1.1182\\\hline
$10^{-5}$&1.2663&1.2495&1.1517&1.1334&1.1328\\\hline
$10^{-4}$&1.2480&1.2380&1.1963&1.2013&1.2165\\\hline
$10^{-3}$&1.2401&1.2445&1.3253&1.2728&1.3261\\\hline
$10^{-2}$&1.2106&1.1952&1.2381&1.2728&1.3147\\\hline\hline
\multicolumn{6}{c}{Coauthor Physics: from 16 to 256 neurons }\\
\hline
\diagbox{$\rho$}{neuron} &16&32 &64  &128 &256\\
\hline
$10^{-6}$&0.8890&0.8926&0.9154&0.8955&0.9585\\\hline
$10^{-5}$&0.9929&0.9945&1.0649&1.0766&1.1329\\\hline
$10^{-4}$&1.1891&1.1159&1.2319&1.1663&1.2309\\\hline
$10^{-3}$&1.2210&1.1574&1.2737&1.2491&1.3408\\\hline
$10^{-2}$&1.2052&1.1653&1.2782&1.2573&1.5391\\\hline\hline
\end{tabular}
\caption{The relationship between running time per epoch (in second) and the number of neurons for each layer as well as value of $\rho$: generally, the running time increases with the increase of the value of $\rho$.}
\label{tab:running time 3}
\end{table}
 \begin{figure}
     \centering
    \includegraphics[width=0.8\linewidth]{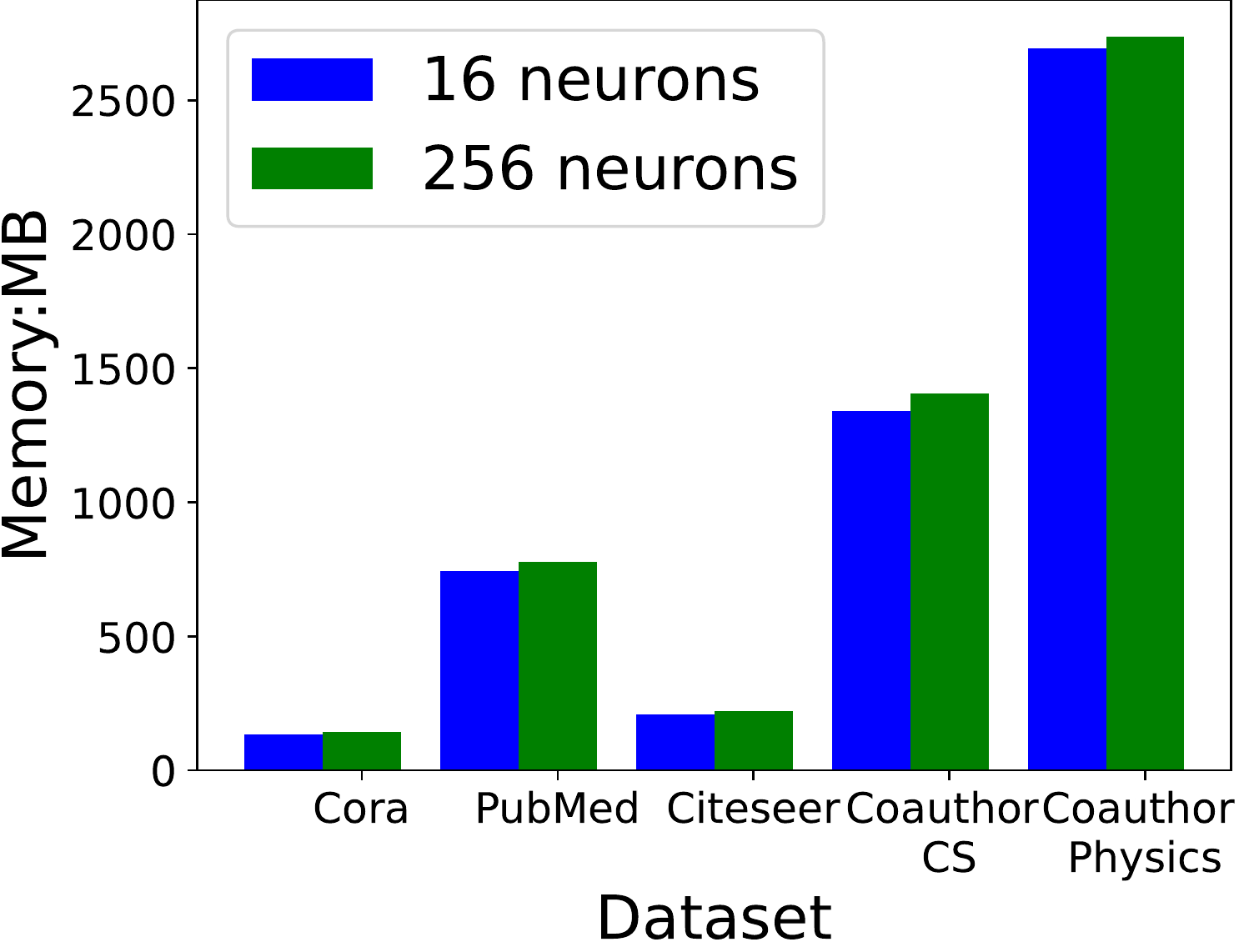}
     \centering
    \caption{Memory usages of the proposed dlADMM on all datasets: the number of neurons has little effects on the memory usages. }
    \label{fig:gcn memory}
 \vspace{-1cm}

\end{figure}

\indent The experimental results on GCN models are analyzed in this section.\\
\textbf{Convergence}: Figures \ref{fig:gcn convergence} and \ref{fig:gcn divergence} show convergence curves and divergence curves of the proposed dlADMM on three datasets when $\rho=1$ and $\rho=10^{-3}$, respectively. when $\rho=1$ (i.e. in Figure \ref{fig:gcn convergence}), the objectives tumble down monotonously and smoothly, and the residuals converge to 0, while the objective and the residual on the Cora dataset fluctuate when $\rho=10^{-3}$ (i.e. in Figure \ref{fig:gcn divergence}). Specifically, in Figure \ref{fig:gcn convergence}(a), the objective on the Cora dataset keeps decreasing, whereas those on the PubMed and Citeseer datasets converge within 100 epochs. The residuals on three datasets converge to 0 before 100 epochs, as demonstrated in Figure \ref{fig:gcn convergence}(b). In Figure \ref{fig:gcn divergence}(a), the objective on the Cora dataset increases with the increase of epochs, while it decreases on the other two datasets. For the residuals, Figure \ref{fig:gcn divergence}(b) shows that the fluctuation on the Cora dataset is more drastic than that on the PubMed and Citeseer datasets.\\
\textbf{Performance:} Table \ref{tab:gcn performance} shows the test accuracy of the proposed dlADMM and all comparison methods on five datasets. They were averaged by 10 random initializations. Overall, our proposed dlADMM outperforms others marginally on four out of five datasets. Specifically, it performs $0.76$ on the PubMed dataset, $1\%$ better than Adam, which is the best of all comparison methods. For other datasets such as Coauthor CS and Coauthor Physics, while all methods perform well, our proposed dlADMM achieves marginally better performance by about $0.4\%$. For the Citeseer dataset, GD is the best optimizer, followed by the proposed dlADMM, which is $0.4\%$ inferior in performance. The GD performs differently on five datasets: it is inferior to others on the PubMed dataset, but performs competitively on the Cora dataset. Compared with other optimizers, the Adadelta performs poorly in general: its performance on the Cora, PubMed, and Citeseer datasets is $9\%$, $7\%$ and $5\%$ inferior to that of the proposed dlADMM, respectively. Last but not least, the standard deviations of all methods are small, which suggests that they are resistant to random noises.\\
\textbf{Efficiency Analysis:} Next, we explore the relationship between running time per epoch of our proposed dlADMM and two potential factors, namely, the value of $\rho$, and the number of neurons. The running time was calculated by the average of 500 epochs. \\
\indent  The running time on five datasets is shown in Table \ref{tab:running time 3}. The number of neurons for each layer ranged from 16 to 256. The value of $\rho$ ranged from $10^{-6}$ to $10^{-2}$, with being multiplied by 10  each time. Generally, the running time increases in proportion to the value of $\rho$. For example, the running time per epoch rises from 0.8 to 1.17 when the number of neurons is 16 on the Citeseer dataset, and it climbs from 0.167 to 0.28 on the Cora dataset correspondingly. Moreover, we also find that the running time increases with the increase of the number of neurons on some datasets. As an instance, when $\rho=10^{-2}$, the running time grows from 1.21 to 1.54 as the number of neurons increases from 16 to 256. However, the running time per epoch remains constant on other datasets. As an example, the running time keeps around 0.18 when $\rho=10^{-6}$ on the Cora and PubMed datasets.\\
\indent \textbf{Memory Analysis:} Finally, we investigate the memory usages of our proposed dlADMM on all datasets, which is shown in Figure \ref{fig:gcn memory}. A blue bar and a green bar denote 16 neurons and 256 neurons, respectively. The Coauthor-Physics consumes the most memory of all datasets. It occupies more than 2500MB, whereas the Cora dataset uses less than 100MB. Compare Figure \ref{fig:gcn memory} with Table \ref{tab:gcn dataset}, we find that the used memories are approximately in proportion to the number of nodes on five datasets. Moreover, the effect of the number of neurons is slim on the memory usages: the memory used on the same dataset is almost the same for 16 neurons and 256 neurons. Therefore, we infer datasets, instead of model architectures, are the main component of memory usages.
\section{Conclusion}
\label{sec:conclusion}
\indent Alternating Direction Method of Multipliers (ADMM) is a good alternative to Gradient Descent (GD) for deep learning problems. In this paper, we propose a novel deep learning Alternating Direction Method of Multipliers (dlADMM) to address some previously mentioned challenges. Firstly, the dlADMM updates parameters from backward to forward in order to transmit parameter information more efficiently. The time complexity is successfully reduced from $O(n^3)$ to $O(n^2)$  by iterative quadratic approximations and backtracking. Finally, the dlADMM is guaranteed to converge to a critical point under mild conditions with a sublinear convergence rate $o(1/k)$. Experiments on MLP and GCN models on seven benchmark datasets demonstrate not only the convergence and efficiency of our proposed dlADMM algorithm, but also its outstanding performance against comparison methods.
\bibliography{example_paper}

\begin{thebibliography}{10}
\bibitem{tielemandivide}
T.~Tieleman and G.~Hinton, ``Divide the gradient by a running average of its
  recent magnitude. coursera: Neural networks for machine learning,'' Technical
  Report., Tech. Rep.
\end{thebibliography}
\bibliographystyle{IEEEtran}
\onecolumn
\newpage
\large{Supplementary Materials}
\begin{appendix}
\small
\section*{Algorithms to update $W^{k+1}_l$ and $a^{k+1}_l$}
\label{sec:same algorithm}
\indent The algorithms to update $W^{k+1}_l$ and $a^{k+1}_l$ are described in the Algorithms \ref{algo:theta update} and \ref{algo:tau update}, respectively.
\begin{algorithm} 
\caption{The Backtracking Algorithm  to update ${W}^{k+1}_{l}$ } 
\begin{algorithmic}[1]
\label{algo:theta update}
\REQUIRE ${\textbf{W}}^{k+1}_{l-1}$,${\textbf{b}}^{k+1}_{l-1}$, ${\textbf{z}}^{k+1}_{l-1}$,${\textbf{a}}^{k+1}_{l-1}$,$u^k$, $\rho$, some constant ${\gamma}>1$. 
\ENSURE ${\theta}^{k+1}_l$,${{W}}^{k+1}_{l}$. 
\STATE Pick up ${\alpha}$ and ${\zeta}=\overline{W}^{k+1}_l-\nabla_{\overline{W}^{k+1}_l}\phi/{\alpha}$.
\WHILE{$\phi(\{W^{k+1}_i\}_{i=1}^{l-1},{\zeta},\{\overline{W}^{k+1}_i\}_{i=l+1}^{L},{\textbf{b}}^{k+1}_{l},{\textbf{z}}^{k+1}_{l},{\textbf{a}}^{k+1}_{l},u^k)>{P}_l({\zeta};{\alpha})$}
\STATE ${\alpha}\leftarrow {\alpha}\ {\gamma}$.\\
\STATE Solve ${\zeta}$ by Equation \eqref{eq:update W}.\\
\ENDWHILE
\STATE Output ${\theta}^{k+1}_l \leftarrow {\alpha} $.\\
\STATE Output ${W}^{k+1}_{l}\leftarrow {\zeta}$.
\end{algorithmic}
\end{algorithm}
\begin{algorithm} 
\caption{The Backtracking Algorithm  to update ${a}^{k+1}_{l}$ } 
\begin{algorithmic}[1]
\label{algo:tau update}
\REQUIRE ${\textbf{W}}^{k+1}_{l}$,${\textbf{b}}^{k+1}_{l}$, ${\textbf{z}}^{k+1}_{l}$,${\textbf{a}}^{k+1}_{l-1}$,$u^k$, $\rho$, some constant ${\eta}>1$. 
\ENSURE ${\tau}^{k+1}_l$,${{a}}^{k+1}_{l}$. 
\STATE Pick up ${t}$ and ${\beta}=\overline{a}^{k+1}_l-\nabla_{\overline{a}^{k+1}_l}\phi/{t}$
\WHILE{$\phi({\textbf{W}}^{k+1}_{l+1},{\textbf{b}}^{k+1}_{l+1},{\textbf{z}}^{k+1}_{l+1},\{a^{k+1}_i\}_{i=1}^{l-1},{\beta},\{\overline{a}^{k+1}_i\}_{i=l+1}^{L-1},u^k)>{Q}_l({\beta};{t})$}
\STATE ${t}\leftarrow {t}{\eta}$.\\
\STATE ${\beta}\leftarrow \overline{a}^{k+1}_{l}-\nabla_{\overline{a}^{k+1}_{l}}\phi/{t}$.\\
\ENDWHILE
\STATE Output ${\tau}^{k+1}_l \leftarrow {t} $.\\
\STATE Output ${a}^{k+1}_{l}\leftarrow {\beta}$.
\end{algorithmic}
\end{algorithm}
\section*{Lemmas for the Proofs of Properties}
\label{sec:proofs}
The following several lemmas are preliminary results.
\begin{lemma}
\label{lemma:lemma 1}
Equation \eqref{eq:update overline W} holds if and only if there exists $\overline{s}\in \partial\Omega_l(\overline{W}^{k+1}_l)$, the subgradient of $\Omega_l(\overline{W}^{k+1}_l)$  such that
\begin{align*}
    &\nabla_{{W}^{k}_l}\phi({\overline{\textbf{W}}}^{k+1}_{l+1},\overline{\textbf{b}}^{k+1}_{l},\overline{\textbf{z}}^{k+1}_{l},\overline{\textbf{a}}^{k+1}_{l},u^k) +\overline{\theta}^{k+1}_l \circ(\overline{W}^{k+1}_l-{W}^{k}_l)+\overline{s}=0
\end{align*}
Likewise, Equation \eqref{eq:update W} holds if and only if there exists $s\in \partial\Omega_l(W^{k+1}_l)$, the subgradient of $\Omega_l(W^{k+1}_l)$  such that
\begin{align*}
    &\nabla_{\overline{W}^{k+1}_l}\phi({\textbf{W}}^{k+1}_{l-1},\textbf{b}^{k+1}_{l-1},\textbf{z}^{k+1}_{l-1},{\textbf{a}}^{k+1}_{l-1},u^k) +\theta^{k+1}_l \circ(W^{k+1}_l-\overline{W}^{k+1}_l)+s=0
\end{align*}
\end{lemma}
\begin{proof}
    These can be obtained by directly applying the optimality conditions of Equation \eqref{eq:update overline W} and Equation \eqref{eq:update W}, respectively.
\end{proof}
\begin{lemma}
$\nabla_{z_L^k} R(z_L^k;y)+u^k=0$ for all $k\in \mathbb{N}$.
\label{lemma:z_l optimality}
\end{lemma}
\begin{proof}
The optimality condition of $z^k_L$ in Equation \eqref{eq:update zl} gives rise to 
\begin{align*}
    \nabla_{z^k_L} R(z^{k}_L;y)+\rho(z^{k}_L-W^{k}_La^{k}_{L-1}-b^{k}_L)+u^{k-1}=0
\end{align*}
Because $u^k=u^{k-1}+\rho(z^{k}_L-W^{k}_La^{k}_{L-1}-b^{k}_L)$, then we have $\nabla_{z_L^k} R(z_L^k;y)+u^k=0$.
\end{proof}
\begin{lemma}
\label{lemma: R(z_l) lipschitz}
It holds that $\forall z_{L,1},z_{L,2}\in\mathbb{R}^{n_L}$,
\begin{align*}
    &R(z_{L,1};y)\leq R(z_{L,2};y)+\nabla_{z_{L,2}} R^T(z_{L,2};y)(z_{L,1}-z_{L,2})+(H/2)\Vert z_{L,1}-z_{L,2}\Vert^2\\&-R(z_{L,1};y)\leq -R(z_{L,2};y)-\nabla_{z_{L,2}} R^T(z_{L,2};y)(z_{L,1}-z_{L,2})+(H/2)\Vert z_{L,1}-z_{L,2}\Vert^2
\end{align*}
\end{lemma}
\begin{proof}
Because $R(z_L;y)$ is Lipschitz differentiable by Assumption \ref{ass:assumption 2}, so is $-R(z_L;y)$. Therefore, this lemma is proven exactly as same as Lemma 2.1 in \cite{beck2009fast}.
\end{proof}
\begin{lemma}
For Equations \eqref{eq:update overline b} and \eqref{eq:update b}, if $\overline{B},B\geq \nu$,then the following inequalities hold:
\begin{align}
&\overline{U}_{l}(\overline{\textbf{b}}^{k+1}_{l};\overline{B})\geqslant \phi(\overline{\textbf{W}}^{k+1}_{ l+1},\overline{\textbf{b}}^{k+1}_{l},\overline{\textbf{z}}^{k+1}_{l},\overline{\textbf{a}}^{k+1}_{l},u^k)  \label{eq:lipschitz b back}\\&
    U_{l}(\textbf{b}^{k+1}_{l};B)\geqslant \phi(\textbf{W}^{k+1}_{l},\textbf{b}^{k+1}_{l},\textbf{z}^{k+1}_{l-1},\textbf{a}^{k+1}_{l-1},u^k)  \label{eq:lipschitz b forward}
\end{align} 
\label{lemma:lemma 2}

\end{lemma}
\begin{proof}
Because $\phi(\textbf{W},\textbf{b},\textbf{z},\textbf{a},u)$ is Lipschitz differentiable with respect to $\textbf{b}$ with Lipschitz
coefficient $\nu$ (the definition of Lipschitz differentiablity can be found in \cite{beck2009fast}), we directly apply Lemma 2.1 in \cite{beck2009fast} to $\phi$ to obtain Equations \eqref{eq:lipschitz b back} and \eqref{eq:lipschitz b forward}, respectively.
\end{proof}
\begin{lemma}
\label{lemma:lemma 3}
It holds that for $\forall k\in \mathbb{N}$,
\begin{align}
& L_\rho(\overline{\textbf{W}}^{k+1}_{l+1},\overline{\textbf{b}}^{k+1}_{l+1},\overline{\textbf{z}}^{k+1}_{l+1},\overline{\textbf{a}}^{k+1}_{l+1},u^k)-L_\rho(\overline{\textbf{W}}^{k+1}_{l+1},\overline{\textbf{b}}^{k+1}_{l+1},\overline{\textbf{z}}^{k+1}_{l+1},\overline{\textbf{a}}^{k+1}_l,u^k)\geq \Vert\overline{\tau}_l^{k+1}\circ (\overline{a}^{k+1}_l-a_l^k)^{\circ 2}\Vert_{1}/2(l=1,\cdots,L-1)
\label{eq:overline a optimality}\\
&L_\rho(\overline{\textbf{W}}^{k+1}_{l+1},\overline{\textbf{b}}^{k+1}_{l+1},\overline{\textbf{z}}^{k+1}_{l+1},\overline{\textbf{a}}^{k+1}_l,u^k)\geq L_\rho(\overline{\textbf{W}}^{k+1}_{l+1},\overline{\textbf{b}}^{k+1}_{l+1},\overline{\textbf{z}}^{k+1}_l,\overline{\textbf{a}}^{k+1}_l,u^k)(l=1,\cdots,L-1)\label{eq:overline z optimality}\\
&L_\rho({\textbf{W}}^{k},{\textbf{b}}^{k},{\textbf{z}}^{k},{\textbf{a}}^{k},u^k)- L_\rho({\textbf{W}}^{k},{\textbf{b}}^{k},\overline{\textbf{z}}^{k+1}_L,{\textbf{a}}^k,u^k)\geq (\rho/2)\Vert \overline{z}^{k+1}_L-z^k_L\Vert^2_2 \label{eq:overline zl optimality}\\
& L_\rho(\overline{\textbf{W}}^{k+1}_{l+1},\overline{\textbf{b}}^{k+1}_{l+1},\overline{\textbf{z}}^{k+1}_l,\overline{\textbf{a}}^{k+1}_l,u^k)-L_\rho(\overline{\textbf{W}}^{k+1}_{l+1},\overline{\textbf{b}}^{k+1}_l,\overline{\textbf{z}}^{k+1}_l,\overline{\textbf{a}}^{k+1}_l,u^k)\geq (\nu/2)\Vert \overline{b}^{k+1}_l-b_l^k\Vert^2_2(l=1,\cdots,L-1)\label{eq:overline b optimality}\\
 &L_\rho({\textbf{W}}^{k},{\textbf{b}}^{k},\overline{\textbf{z}}^{k+1}_L,{\textbf{a}}^{k},u^k)- L_\rho({\textbf{W}}^{k},\overline{\textbf{b}}^{k+1}_L,\overline{\textbf{z}}^{k+1}_L,{\textbf{a}}^k,u^k)\geq (\rho/2)\Vert \overline{b}^{k+1}_L-b^k_L\Vert^2_2\label{eq:overline bl optimality}\\
& L_\rho(\overline{\textbf{W}}^{k+1}_{l+1},\overline{\textbf{b}}^{k+1}_l,\overline{\textbf{z}}^{k+1}_l,\overline{\textbf{a}}^{k+1}_l,u^k)-L_\rho(\overline{\textbf{W}}^{k+1}_l,\overline{\textbf{b}}^{k+1}_l,\overline{\textbf{z}}^{k+1}_l,\overline{\textbf{a}}^{k+1}_l,u^k)\geq \Vert\overline{\theta}_l^{k+1}\circ (\overline{W}^{k+1}_l-W_l^k)^{\circ 2}\Vert_{1}/2(l=1,\cdots,L)\label{eq:overline W optimality}\\
& L_\rho({\textbf{W}}^{k+1}_{l-1},{\textbf{b}}^{k+1}_{l-1},{\textbf{z}}^{k+1}_{l-1},{\textbf{a}}^{k+1}_{l-1},u^k)-L_\rho({\textbf{W}}^{k+1}_{l},{\textbf{b}}^{k+1}_{l-1},{\textbf{z}}^{k+1}_{l-1},{\textbf{a}}^{k+1}_{l-1},u^k)\geq \Vert{\theta}_l^{k+1}\circ ({W}^{k+1}_l-\overline{W}_l^{k+1})^{\circ 2}\Vert_{1}/2(l=1,\cdots,L)\label{eq:W optimality}\\
& L_\rho({\textbf{W}}^{k+1}_{l},{\textbf{b}}^{k+1}_{l-1},{\textbf{z}}^{k+1}_{l-1},{\textbf{a}}^{k+1}_{l-1},u^k)-L_\rho({\textbf{W}}^{k+1}_{l},{\textbf{b}}^{k+1}_{l},{\textbf{z}}^{k+1}_{l-1},{\textbf{a}}^{k+1}_{l-1},u^k)\geq (\nu/2)\Vert {b}^{k+1}_l-\overline{b}_l^{k+1}\Vert^2_2(l=1,\cdots,L-1)\label{eq:b optimality}\\
& L_\rho({\textbf{W}}^{k+1},{\textbf{b}}^{k+1}_{L-1},{\textbf{z}}^{k+1}_{L-1},{\textbf{a}}^{k+1},u^k)-L_\rho({\textbf{W}}^{k+1},{\textbf{b}}^{k+1},{\textbf{z}}^{k+1}_{L-1},{\textbf{a}}^{k+1},u^k)\geq (\rho/2)\Vert {b}^{k+1}_L-\overline{b}_L^{k+1}\Vert^2_2\label{eq:bl optimality}\\
 &L_\rho({\textbf{W}}^{k+1}_{l},{\textbf{b}}^{k+1}_{l},{\textbf{z}}^{k+1}_{l-1},{\textbf{a}}^{k+1}_{l-1},u^k)\geq L_\rho({\textbf{W}}^{k+1}_{l},{\textbf{b}}^{k+1}_{l},{\textbf{z}}^{k+1}_l,{\textbf{a}}^{k+1}_{l-1},u^k)(l=1,\cdots,L-1)\label{eq:z optimality}\\
 & L_\rho({\textbf{W}}^{k+1}_{l},{\textbf{b}}^{k+1}_{l},{\textbf{z}}^{k+1}_{l},{\textbf{a}}^{k+1}_{l-1},u^k)-L_\rho({\textbf{W}}^{k+1}_{l},{\textbf{b}}^{k+1}_{l},{\textbf{z}}^{k+1}_{l},{\textbf{a}}^{k+1}_l,u^k)\geq \Vert{\tau}_l^{k+1}\circ ({a}^{k+1}_l-\overline{a}_{l}^{k+1})^{\circ 2}\Vert_{1}/2(l=1,\cdots,L-1)
\label{eq:a optimality}
\end{align}\end{lemma}
\begin{proof}
Essentially, all inequalities can be obtained by applying optimality conditions of updating $\overline{a}^{k+1}_l$, $\overline{z}^{k+1}_l$, $\overline{b}^{k+1}_l$, $\overline{W}^{k+1}_l$, ${W}^{k+1}_l$, ${b}^{k+1}_l$, ${z}^{k+1}_l$ and ${a}^{k+1}_l$, respectively. We only prove Inequality \eqref{eq:overline zl optimality}, \eqref{eq:W optimality}, \eqref{eq:b optimality} and \eqref{eq:z optimality} .This is because Inequalities \eqref{eq:overline a optimality}, \eqref{eq:overline W optimality} and \eqref{eq:a optimality} follow the routine of Inequality \eqref{eq:W optimality}, Inequalities \eqref{eq:overline b optimality}, \eqref{eq:overline bl optimality} and \eqref{eq:bl optimality} follow the routine of Inequality \eqref{eq:b optimality}, and  Inequality \eqref{eq:overline z optimality} follows the routine of Inequality \eqref{eq:z optimality}.\\
\indent Firstly, we focus on Inequality \eqref{eq:overline zl optimality}.
\begin{align}
    \nonumber &L_\rho({\textbf{W}}^{k},{\textbf{b}}^{k},{\textbf{z}}^{k},{\textbf{a}}^{k},u^k)- L_\rho({\textbf{W}}^{k},{\textbf{b}}^{k},\overline{\textbf{z}}^{k+1}_L,{\textbf{a}}^k,u^k)\\\nonumber&=R(z^k_L;y)+(u^k)^T(z^k_L-W^k_La^k_{L-1}-b^k_L)+(\rho/2)\Vert z^k_L-W^k_La^k_{L-1}-b^k_L\Vert^2_2\\\nonumber &-R(\overline{z}^{k+1}_L;y)-(u^k)^T(\overline{z}^{k+1}_L-W^k_La^k_{L-1}-b^k_L)-(\rho/2)\Vert \overline{z}^{k+1}_L-W^k_La^k_{L-1}-b^k_L\Vert^2_2\\&\nonumber=R(z^k_L;y)-R(\overline{z}^{k+1}_L;y)+(u^k)^T(z^k_L-\overline{z}^{k+1}_L)+(\rho/2)\Vert \overline{z}^{k+1}_L-z^k_L\Vert^2_2\\&+ \rho(\overline{z}^{k+1}_L-W^k_La^k_{L-1}-b^k_L)^T(z^k_L-\overline{z}^{k+1}_L)\label{eq:cosine rule}
\end{align}
where the second equality follows from the cosine rule $\Vert z^k_L-W^k_La^k_{L-1}-b^k_L\Vert^2_2-\Vert \overline{z}^{k+1}_L-W^k_La^k_{L-1}-b^k_L\Vert^2_2=\Vert \overline{z}^{k+1}_L-z^k_L\Vert^2_2+(\overline{z}^{k+1}_L-W^k_La^k_{L-1}-b^k_L)^T(z^k_L-\overline{z}^{k+1}_L)$.\\
\indent According to the optimality condition of Equation \eqref{eq:update overline zl},
we have 
$\nabla_{\overline{z}^{k+1}_L}R(\overline{z}^{k+1}_L;y)+u^k+\rho(\overline{z}^{k+1}_L-W^k_La^k_{L-1}-b^k_L)=0$. Because $R(z_L;y)$ is convex and differentiable with regard to $z_L$, its subgradient is also its gradient. According to the definition of subgradient, we have
\begin{align}
    \nonumber R(z^k_L;y)&\geq R(\overline{z}^{k+1}_L;y)+\nabla_{\overline{z}^{k+1}_L}R^T(\overline{z}^{k+1}_L;y)(z^k_L-\overline{z}^{k+1}_L)\\&=R(\overline{z}^{k+1}_L;y)-(u^k+\rho(\overline{z}^{k+1}_L-W^k_La^k_{L-1}-b^k_L))^T(z^k_L-\overline{z}^{k+1}_L)
    \label{eq:zl subgradient}
\end{align}
We introduce Equation \eqref{eq:zl subgradient} into Equation \eqref{eq:cosine rule} to obtain Equation \eqref{eq:overline zl optimality}.\\
\indent Secondly, we focus on Inequality \eqref{eq:W optimality}. The stopping criterion of Algorithm \ref{algo:theta update} shows that
\begin{align}
    \phi(\textbf{W}^{k+1}_l,\textbf{b}^{k+1}_{l-1},\textbf{z}^{k+1}_{l-1},\textbf{a}^{k+1}_{l-1},u^k)\leq P_l(W^{k+1}_{l};\theta^{k+1}_{l}).\label{ineq:ineq1}
\end{align}
Because $\Omega_{W_l}(W_l)$ is convex, according to the definition of subgradient, we have
\begin{align}
  \Omega_l(\overline{W}^{k+1}_{l})\geq  \Omega_l({W}^{k+1}_l)+s^T(\overline{W}_l^{k+1}-W_l^{k+1}). \label{ineq:ineq2}
\end{align}
where $s$ is defined in the premise of Lemma \ref{lemma:lemma 1}. Therefore, we have
\begin{align*}
    & L_\rho({\textbf{W}}^{k+1}_{l-1},{\textbf{b}}^{k+1}_{l-1},{\textbf{z}}^{k+1}_{l-1},{\textbf{a}}^{k+1}_{l-1},u^k)-L_\rho({\textbf{W}}^{k+1}_{l},{\textbf{b}}^{k+1}_{l-1},{\textbf{z}}^{k+1}_{l-1},{\textbf{a}}^{k+1}_{l-1},u^k)\\&=\phi({\textbf{W}}^{k+1}_{l-1},{\textbf{b}}^{k+1}_{l-1},{\textbf{z}}^{k+1}_{l-1},{\textbf{a}}^{k+1}_{l-1},u^k)+\Omega_l(\overline{W}^{k+1}_l)-\phi({\textbf{W}}^{k+1}_{l},{\textbf{b}}^{k+1}_{l-1},{\textbf{z}}^{k+1}_{l-1},{\textbf{a}}^{k+1}_{l-1},u^k)-\Omega_l(W^{k+1}_{l})\ \text{(Definition of $L_\rho$)}\\&\geq \Omega_l(\overline{W}^{k+1}_l)-\Omega_l(W^{k+1}_{l})-\nabla \phi^T_{\overline{W}^{k+1}_l}(W^{k+1}_l-\overline{W}^{k+1}_l)-\Vert\theta_l^{k+1}\circ (W^{k+1}_l-\overline{W}_l^{k+1})^{\circ 2}\Vert_{1}/2(\text{Equation \eqref{ineq:ineq1}})\\&\geq s^T(\overline{W}^{k+1}_l-{W}^{k+1}_l)-\nabla \phi^T_{\overline{W}^{k+1}_l}(W^{k+1}_l-\overline{W}^{k+1}_l)-\Vert\theta_l^{k+1}\circ (W^{k+1}_l-\overline{W}_l^{k+1})^{\circ 2}\Vert_{1}/2(\text{Equation \eqref{ineq:ineq2}})\\ &=(s^T+\nabla \phi^T_{\overline{W}^{k+1}_l})(\overline{W}^{k+1}_l-W^{k+1}_l)-\Vert\theta_l^{k+1}\circ (W^{k+1}_l-\overline{W}_l^{k+1})^{\circ 2}\Vert_{1}/2\\&=\Vert\theta_l^{k+1}\circ (W^{k+1}_l-\overline{W}_l^{k+1})^{\circ 2}\Vert_{1}/2\ (\text{Lemma \ref{lemma:lemma 1}}).
\end{align*}
Thirdly, we focus on Inequality \eqref{eq:b optimality}.
\begin{align*}
    &L_\rho({\textbf{W}}^{k+1}_{l},{\textbf{b}}^{k+1}_{l-1},{\textbf{z}}^{k+1}_{l-1},{\textbf{a}}^{k+1}_{l-1},u^k)-L_\rho({\textbf{W}}^{k+1}_{l},{\textbf{b}}^{k+1}_{l},{\textbf{z}}^{k+1}_{l-1},{\textbf{a}}^{k+1}_{l-1},u^k)\\&=\phi({\textbf{W}}^{k+1}_{l},{\textbf{b}}^{k+1}_{l-1},{\textbf{z}}^{k+1}_{l-1},{\textbf{a}}^{k+1}_{l-1},u^k)-\phi({\textbf{W}}^{k+1}_{l},{\textbf{b}}^{k+1}_{l},{\textbf{z}}^{k+1}_{l-1},{\textbf{a}}^{k+1}_{l-1},u^k)\\&\geq (\nu/2)\Vert {b}^{k+1}_l-\overline{b}_l^{k+1}\Vert^2_2. (\text{Lemma \ref{lemma:lemma 2}})
\end{align*}
Finally, we focus on Equation \eqref{eq:z optimality}. This follows directly  from the optimality of $z^{k+1}_l$ in Equation \eqref{eq:update z}.
\end{proof}
\begin{lemma}
If $\rho>2H$  so that $C_1=\rho/2-H/2-H^2/\rho>0$, then it holds that
\begin{align}
\nonumber &L_\rho(\textbf{W}^{k+1},\textbf{b}^{k+1},\textbf{z}^{k+1}_{L-1},\textbf{a}^{k+1},u^k)-L_\rho(\textbf{W}^{k+1},\textbf{b}^{k+1},\textbf{z}^{k+1},\textbf{a}^{k+1},u^{k+1})\\&\geq  C_1\Vert z_L^{k+1}-\overline{z}_L^{k+1}\Vert^2_2-(H^2/\rho)\Vert\overline{z}_L^{k+1}-{z}_L^{k}\Vert^2_2. \label{eq: lemma4}  
\end{align}
\label{lemma:lemma 4}
\end{lemma}
\begin{proof}
\begin{align*}
    &L_\rho(\textbf{W}^{k+1},\textbf{b}^{k+1},\textbf{z}^{k+1}_{L-1},\textbf{a}^{k+1},u^k)-L_\rho(\textbf{W}^{k+1},\textbf{b}^{k+1},\textbf{z}^{k+1},\textbf{a}^{k+1},u^{k+1})\\&=R(\overline{z}_L^{k+1};y)-R({z}_L^{k+1};y)+(u^{k+1})^T(\overline{z}_L^{k+1}-z_L^{k+1})+(\rho/2)\Vert z_L^{k+1}-\overline{z}_L^{k+1}\Vert^2_2-(1/\rho)\Vert u^{k+1}-u^k\Vert^2_2\\&=R(\overline{z}_L^{k+1};y)-R(z_L^{k+1};y)+\nabla_{z_L^{k+1}} R(z_L^{k+1};y)^T(z_L^{k+1}-\overline{z}_L^{k+1})+(\rho/2)\Vert z_L^{k+1}-\overline{z}_L^{k+1}\Vert^2_2-(1/\rho)\Vert u^{k+1}-u^k\Vert^2_2\text{( Lemma \ref{lemma:z_l optimality})}\\&\geq (-H/2)\Vert z_L^{k+1}-\overline{z}_L^{k+1}\Vert^2_2+(\rho/2)\Vert z_L^{k+1}-\overline{z}_L^{k+1}\Vert^2_2-(1/\rho)\Vert \nabla_{z_L^{k+1}} R(z_L^{k+1};y)-\nabla_{z_L^{k}} R(z_L^{k};y)\Vert^2_2\\&\text{($-R(z_L;y)$ is Lipschitz differentiable, Lemmas \ref{lemma:z_l optimality} and \ref{lemma: R(z_l) lipschitz})}\\&\geq (-H/2)\Vert z_L^{k+1}-\overline{z}_L^{k+1}\Vert^2_2+(\rho/2)\Vert z_L^{k+1}-\overline{z}_L^{k+1}\Vert^2_2-(H^2/\rho)\Vert z_L^{k+1}-{z}_L^{k}\Vert^2_2\text{(Assumption \ref{ass:assumption 2})}\\&\geq(-H/2)\Vert z_L^{k+1}-\overline{z}_L^{k+1}\Vert^2_2+(\rho/2)\Vert z_L^{k+1}-\overline{z}_L^{k+1}\Vert^2_2-(H^2/\rho)\Vert z_L^{k+1}-\overline{z}_L^{k+1}\Vert^2_2-(H^2/\rho)\Vert \overline{z}_L^{k+1}-{z}_L^{k}\Vert^2_2\text{(triangle inequality)}\\&=C_1\Vert z_L^{k+1}-\overline{z}_L^{k+1}\Vert^2_2-(H^2/\rho)\Vert\overline{z}_L^{k+1}-{z}_L^{k}\Vert^2_2.
\end{align*}
We choose $\rho>2H$ to make $C_1>0$.
\end{proof}
\section*{Proof of Theorem \ref{thero: theorem 1}}
\label{sec: proof of theorem 1}
Proving Theorem \ref{thero: theorem 1} is equal to proving jointly Theorem \ref{thero: property 1}, \ref{thero: property 2}, and \ref{thero: property 3}, which are elaborated in the following.
\begin{theorem}
Given that Assumptions \ref{ass:assumption 1} and \ref{ass:assumption 2} hold, the dlADMM satisfies Property \ref{pro:property 1}.
\label{thero: property 1}
\end{theorem}
\begin{proof}
There exists $z_L^{'}$ such that $z_L^{'}-W^k_La^k_{L-1}-b^k_L=0$. By Assumption \ref{ass:assumption 2}, we have
\begin{align*}
    &F(\textbf{W}^k,\textbf{b}^k,\{z_l^k\}_{l=1}^{L-1},z^{'}_L,\textbf{a}^k)\geq \min S  > -\infty
\end{align*}
where $S=\{F(\textbf{W},\textbf{b},\textbf{z},\textbf{a}): z_L-W_La_{L-1}-b_L=0\}$. Then we have
\begin{align*}
    &L_\rho(\textbf{W}^k,\textbf{b}^k,\textbf{z}^k,\textbf{a}^k,u^k)\\&=F(\textbf{W}^k,\textbf{b}^k,\textbf{z}^k,\textbf{a}^k)+(u^k)^T(z^k_L-W^k_La^k_{L-1}-b^k_{L})+(\rho/2)\Vert z^k_L-W^k_La^k_{L-1}-b^k_{L}\Vert^2_2\\&=F(\textbf{W}^k,\textbf{b}^k,\textbf{z}^k,\textbf{a}^k)+(u^k)^T(z^k_L-z^{'}_L)+(\rho/2)\Vert z^k_L-W^k_La^k_{L-1}-b^k_{L}\Vert^2_2 \  \text{$(z_L^{'}-W^k_La^k_{L-1}-b^k_L=0)$}\\
    &=F(\textbf{W}^k,\textbf{b}^k,\textbf{z}^k,\textbf{a}^k)+\nabla_{z^k_L} R^T(z^k_L;y)(z^{'}_L-z^k_L)+(\rho/2)\Vert z^k_L-W^k_La^k_{L-1}-b^k_{L}\Vert^2_2\text{(Lemma \ref{lemma:z_l optimality})}\\&=\sum\nolimits_{l=1}^L \Omega_l(W^{k}_l)+(\nu/2)\sum\nolimits_{l=1}^{L-1}(\Vert z^k_l-W^k_la^k_{l-1}-b^k_l\Vert^2_2+\Vert a^k_l-f(z^k_l)\Vert^2_2)+R(z^k_L;y)+\nabla_{z^k_L} R^T(z^k_L;y)(z^{'}_L-z^k_L)\\&+(\rho/2)\Vert z^k_L-W^k_La^k_{L-1}-b^k_{L}\Vert^2_2\text{(The definition of $F$)}\\&\geq\sum\nolimits_{l=1}^L \Omega_l(W^{k}_l)+(\nu/2)\sum\nolimits_{l=1}^{L-1}(\Vert z^k_l-W^k_la^k_{l-1}-b^k_l\Vert^2_2+\Vert a^k_l-f(z^k_l)\Vert^2_2)+R(z^{'}_L;y)+(\rho-H/2)\Vert z^k_L-W^k_La^k_{L-1}-b^k_{L}\Vert^2_2\\&\text{(Lemmas \ref{lemma:z_l optimality} and \ref{lemma: R(z_l) lipschitz}, $R(z_L;y)$ is Lipschitz differentiable)}\\&> -\infty 
\end{align*}
It concludes from Lemma \ref{lemma:lemma 3} and Lemma \ref{lemma:lemma 4} that $L_\rho(\textbf{W}^k,\textbf{b}^k,\textbf{z}^k,\textbf{a}^k,u^k)$ is upper bounded by $L_\rho(\textbf{W}^0,\textbf{b}^0,\textbf{z}^0,\textbf{a}^0,\textbf{u}^0)$ and so are $\sum\nolimits_{l=1}^L \Omega_l(W^{k}_l)+(\nu/2)\sum\nolimits_{l=1}^{L-1}(\Vert z^k_l-W^k_la^k_{l-1}-b^k_l\Vert^2_2+\Vert a^k_l-f(z^k_l)\Vert^2_2)$ and $\Vert z^k_L-W^k_La^k_{L-1}-b^k_L\Vert^2_2$. By Assumption \ref{ass:assumption 2}, $(\textbf{W}^k,\textbf{b}^k,\textbf{z}^k,\textbf{a}^k)$ is bounded. By Lemma \ref{lemma:z_l optimality}, it is obvious that $u^k$ is bounded as well.
\end{proof}
\begin{theorem}
Given that Assumptions \ref{ass:assumption 1} and \ref{ass:assumption 2} hold, the dlADMM satisfies Property \ref{pro:property 2}.
\label{thero: property 2}
\end{theorem}
\begin{proof}
This follows directly from Lemma \ref{lemma:lemma 3} and Lemma \ref{lemma:lemma 4}.
\end{proof}
\begin{theorem}
Given that Assumptions \ref{ass:assumption 1} and \ref{ass:assumption 2} hold, the dlADMM satisfies Property \ref{pro:property 3}.
\label{thero: property 3}
\end{theorem}
\begin{proof}
We know that 
\begin{align*}
&\partial L_\rho(\textbf{W}^{k+1},\textbf{b}^{k+1},\textbf{z}^{k+1},\textbf{a}^{k+1},u^{k+1})\\&=(\partial_{\textbf{W}^{k+1}} L_\rho,\nabla_{\textbf{b}^{k+1}} L_\rho,\partial_{\textbf{z}^{k+1}} L_\rho,\nabla_{\textbf{a}^{k+1}}L_\rho,\nabla_{u^{k+1}} L_\rho)
\end{align*}
where $\partial_{\textbf{W}^{k+1}} L_\rho=\{\partial_{{W_l}^{k+1}} L_\rho\}_{l=1}^L$, $\nabla_{\textbf{b}^{k+1}} L_\rho=\{\nabla_{{b_l}^{k+1}}L_\rho\}_{l=1}^L$, $\partial_{\textbf{z}^{k+1}} L_\rho=\{\partial_{{z_l}^{k+1}} L_\rho\}_{l=1}^L$, and $\nabla_{\textbf{a}^{k+1}} L_\rho=\{\nabla_{{a_l}^{k+1}} L_\rho\}_{l=1}^{L-1}$ . To prove Property \ref{pro:property 3}, we need to give an upper bound of $\partial_{\textbf{W}^{k+1}} L_\rho,\nabla_{\textbf{b}^{k+1}} L_\rho,\partial_{\textbf{z}^{k+1}} L_\rho,\nabla_{\textbf{a}^{k+1}}L_\rho$ and $\nabla_{u^{k+1}} L_\rho$ by a linear combination of $\Vert \textbf{W}^{k+1}-\overline{\textbf{W}}^{k+1}\Vert$,  $\Vert\textbf{b}^{k+1}-\overline{\textbf{b}}^{k+1}\Vert$, $\Vert\textbf{z}^{k+1}-\overline{\textbf{z}}^{k+1}\Vert$, $\Vert\textbf{a}^{k+1}-\overline{\textbf{a}}^{k+1}\Vert$ and $\Vert \textbf{z}_l^{k+1}-\textbf{z}_l^{k}\Vert$.\\
\indent For $W^{k+1}_l(l<L)$,
\begin{align*}
    \partial_{W^{k+1}_l} L_\rho&=\partial\Omega_l(W^{k+1}_l)+\nabla _{W_l^{k+1}}\phi(\textbf{W}^{k+1},\textbf{b}^{k+1},\textbf{z}^{k+1},\textbf{a}^{k+1},u^{k+1})\\&=\partial\Omega_l(W^{k+1}_l)+\nu(W^{k+1}_la^{k+1}_{l-1}+b^{k+1}_l-z^{k+1}_l)(a^{k+1}_{l-1})^T\\&=\partial\Omega_l(W^{k+1}_l)+\nabla _{W_l^{k+1}}\phi(\textbf{W}^{k+1}_l,\textbf{b}^{k+1}_{l},\textbf{z}^{k+1}_{l},\textbf{a}^{k+1}_{l-1},u^{k})\\&=\partial\Omega_l(W^{k+1}_l)+\nabla _{\overline{W}_l^{k+1}}\phi(\textbf{W}^{k+1}_{l-1},\textbf{b}^{k+1}_{l-1},\textbf{z}^{k+1}_{l-1},\textbf{a}^{k+1}_{l-1},u^{k})+\theta^{k+1}_l\circ(W^{k+1}_l-\overline{W}^{k+1}_l)- \theta^{k+1}_l\circ(W^{k+1}_l-\overline{W}^{k+1}_l)\\&-\nabla _{\overline{W}_l^{k+1}}\phi(\textbf{W}^{k+1}_{l-1},\textbf{b}^{k+1}_{l-1},\textbf{z}^{k+1}_{l-1},\textbf{a}^{k+1}_{l-1},u^{k})+\nabla_{W_l^{k+1}}\phi(\textbf{W}^{k+1}_l,\textbf{b}^{k+1}_{l},\textbf{z}^{k+1}_{l},\textbf{a}^{k+1}_{l-1},u^{k})
    \\&=\partial\Omega_l(W^{k+1}_l)+\nabla_{\overline{W}_l^{k+1}}\phi(\textbf{W}^{k+1}_{l-1},\textbf{b}^{k+1}_{l-1},\textbf{z}^{k+1}_{l-1},\textbf{a}^{k+1}_{l-1},u^{k})+\theta^{k+1}_l\circ(W^{k+1}_l-\overline{W}^{k+1}_l)- \theta^{k+1}_l\circ(W^{k+1}_l-\overline{W}^{k+1}_l)\\&+\nu(W^{k+1}_la^{k+1}_{l-1}+b^{k+1}_l-z^{k+1}_l)(a^{k+1}_{l-1})^T-\nu(\overline{W}^{k+1}_{l}a^{k+1}_{l-1}+\overline{b}^{k+1}_l-\overline{z}^{k+1}_l)(a^{k+1}_{l-1})^T
\end{align*}
Because
\begin{align*}
&\Vert - \theta^{k+1}_l\circ(W^{k+1}_l-\overline{W}^{k+1}_l)+\nu(W^{k+1}_la^{k+1}_{l-1}+b^{k+1}_l-z^{k+1}_l)(a^{k+1}_{l-1})^T-\nu(\overline{W}^{k+1}_{l}a^{k+1}_{l-1}+\overline{b}^{k+1}_l-\overline{z}^{k+1}_l)(a^{k+1}_{l-1})^T\Vert\\&=\Vert - \theta^{k+1}_l\circ(W^{k+1}_l-\overline{W}^{k+1}_l)+\nu(W^{k+1}_l-\overline{W}^{k+1}_l)a^{k+1}_{l-1}(a^{k+1}_{l-1})^T+\nu(b^{k+1}_l-\overline{b}^{k+1}_l)(a^{k+1}_{l-1})^T-\nu(z^{k+1}_l-\overline{z}^{k+1}_l)(a^{k+1}_{l-1})^T\Vert\\&\leq\Vert \theta^{k+1}_l\circ(W^{k+1}_l-\overline{W}^{k+1}_l)\Vert+\nu\Vert(W^{k+1}_l-\overline{W}^{k+1}_l)a^{k+1}_{l-1}(a^{k+1}_{l-1})^T\Vert+\nu\Vert(b^{k+1}_l-\overline{b}^{k+1}_l)(a^{k+1}_{l-1})^T\Vert+\nu\Vert(z^{k+1}_l-\overline{z}^{k+1}_l)(a^{k+1}_{l-1})^T\Vert\\&\text{(triangle inequality)}\\&\leq\Vert \theta^{k+1}_l\circ (W^{k+1}_l-\overline{W}^{k+1}_l)\Vert+\nu\Vert W^{k+1}_l-\overline{W}^{k+1}_l\Vert\Vert a^{k+1}_{l-1}\Vert\Vert a^{k+1}_{l-1}\Vert+\nu\Vert b^{k+1}_l-\overline{b}^{k+1}_l\Vert\Vert a^{k+1}_{l-1}\Vert+\nu\Vert z^{k+1}_l-\overline{z}^{k+1}_l\Vert\Vert a^{k+1}_{l-1}\Vert\\& \text{(Cauchy-Schwarz inequality)}
\end{align*}
and the optimality condition of Equation \eqref{eq:update W} yields
\begin{align*}
    0&\in\partial\Omega_l(W^{k+1}_l)+\nabla_{\overline{W}_l^{k+1}}\phi(\textbf{W}^{k+1}_{l-1},\textbf{b}^{k+1}_{l-1},\textbf{z}^{k+1}_{l-1},\textbf{a}^{k+1}_{l-1},u^{k})+\theta^{k+1}_l\circ(W^{k+1}_l-\overline{W}^{k+1}_l)
\end{align*}
Because  $a^{k+1}_{l-1}$ is bounded by Property \ref{pro:property 1}, $\Vert\partial _{W_l^{k+1} }L_\rho\Vert$ can be upper bounded by a linear combination of  $\Vert W^{k+1}_l-\overline{W}^{k+1}_l\Vert$, $\Vert b^{k+1}_l-\overline{b}^{k+1}_l\Vert$ and $\Vert z^{k+1}_l-\overline{z}^{k+1}_l\Vert$.\\
\indent For $W^{k+1}_L$,
\begin{align*}
    \partial_{W^{k+1}_L} L_\rho&=\partial\Omega_L(W^{k+1}_L)+\nabla_{W_L^{k+1}}\phi(\textbf{W}^{k+1},\textbf{b}^{k+1},\textbf{z}^{k+1},\textbf{a}^{k+1},u^{k+1})\\&=\partial\Omega_L(W^{k+1}_L)+\nabla _{\overline{W}_L^{k+1}}\phi(\textbf{W}^{k+1}_{L-1},\textbf{b}^{k+1}_{L-1},\textbf{z}^{k+1}_{L-1},\textbf{a}^{k+1},u^{k})+\theta^{k+1}_L\circ(W^{k+1}_L-\overline{W}^{k+1}_L)- \theta^{k+1}_L\circ(W^{k+1}_L-\overline{W}^{k+1}_L)\\&-\nabla _{\overline{W}_L^{k+1}}\phi(\textbf{W}^{k+1}_{L-1},\textbf{b}^{k+1}_{L-1},\textbf{z}^{k+1}_{L-1},\textbf{a}^{k+1},u^{k})+\nabla_{W_L^{k+1}}\phi(\textbf{W}^{k+1},\textbf{b}^{k+1},\textbf{z}^{k+1},\textbf{a}^{k+1},u^{k+1})
    \\&=\partial\Omega_L(W^{k+1}_L)+\nabla_{W_L^{k+1}}\phi(\textbf{W}^{k+1}_{L-1},\textbf{b}^{k+1}_{L-1},\textbf{z}^{k+1}_{L-1},\textbf{a}^{k+1},u^{k})+\theta^{k+1}_L\circ(W^{k+1}_L-\overline{W}^{k+1}_L)- \theta^{k+1}_L\circ(W^{k+1}_L-\overline{W}^{k+1}_L)\\&+\rho(W^{k+1}_La^{k+1}_{L-1}+b^{k+1}_L-z^{k+1}_L-u^{k+1}/\rho)(a^{k+1}_{L-1})^T-\rho(\overline{W}^{k+1}_{L}a^{k+1}_{L-1}+\overline{b}^{k+1}_L-\overline{z}^{k+1}_L-u^k/\rho)(a^{k+1}_{L-1})^T
\end{align*}
Because
\begin{align*}
&\Vert - \theta^{k+1}_L\circ(W^{k+1}_L-\overline{W}^{k+1}_L)+\rho(W^{k+1}_La^{k+1}_{L-1}+b^{k+1}_L-z^{k+1}_L-u^{k+1}/\rho)(a^{k+1}_{L-1})^T-\rho(\overline{W}^{k+1}_{L}a^{k+1}_{L-1}+\overline{b}^{k+1}_L-\overline{z}^{k+1}_L-u^k/\rho)(a^{k+1}_{L-1})^T\Vert\\&=\Vert - \theta^{k+1}_L\circ(W^{k+1}_L-\overline{W}^{k+1}_L)+\rho(W^{k+1}_L-\overline{W}^{k+1}_L)a^{k+1}_{L-1}(a^{k+1}_{L-1})^T+\rho(b^{k+1}_L-\overline{b}^{k+1}_L)(a^{k+1}_{L-1})^T-\rho(z^{k+1}_L-\overline{z}^{k+1}_L)(a^{k+1}_{L-1})^T-(u^{k+1}-{u}^{k})(a^{k+1}_{L-1})^T\Vert\\&\leq\Vert \theta^{k+1}_L\circ(W^{k+1}_L-\overline{W}^{k+1}_L)\Vert+\rho\Vert(W^{k+1}_L-\overline{W}^{k+1}_L)a^{k+1}_{L-1}(a^{k+1}_{L-1})^T\Vert+\rho\Vert(b^{k+1}_L-\overline{b}^{k+1}_L)(a^{k+1}_{L-1})^T\Vert+\rho\Vert(z^{k+1}_L-\overline{z}^{k+1}_L)(a^{k+1}_{L-1})^T\Vert\\&+\Vert(u^{k+1}-{u}^{k})(a^{k+1}_{L-1})^T\Vert\text{(triangle inequality)}\\&\leq\Vert \theta^{k+1}_L\circ (W^{k+1}_L-\overline{W}^{k+1}_L)\Vert+\rho\Vert W^{k+1}_L-\overline{W}^{k+1}_L\Vert\Vert a^{k+1}_{L-1}\Vert\Vert a^{k+1}_{L-1}\Vert+\rho\Vert b^{k+1}_L-\overline{b}^{k+1}_L\Vert\Vert a^{k+1}_{L-1}\Vert+\rho\Vert z^{k+1}_L-\overline{z}^{k+1}_L\Vert\Vert a^{k+1}_{L-1}\Vert+H\Vert z_L^{k+1}-z_L^k\Vert \Vert a^{k+1}_{L-1}\Vert\\& \text{(Cauchy-Schwarz inequality, Lemma \ref{lemma:z_l optimality}, $R(z_L;y)$ is Lipschitz differentiable)}
\end{align*}
and the optimality condition of Equation \eqref{eq:update W} yields
\begin{align*}
    0&\in\partial\Omega_L(W^{k+1}_L)+\nabla \phi_{\overline{W}_L^{k+1}}(\textbf{W}^{k+1}_{L-1},\textbf{b}^{k+1}_{L-1},\textbf{z}^{k+1}_{L-1},\textbf{a}^{k+1},u^{k})+\theta^{k+1}_L\circ(W^{k+1}_L-\overline{W}^{k+1}_L)
\end{align*}
Because  $a^{k+1}_{L-1}$ is bounded by Property \ref{pro:property 1}, $\Vert\partial _{W_L^{k+1} }L_\rho\Vert$ can be upper bounded by a linear combination of  $\Vert W^{k+1}_L-\overline{W}^{k+1}_L\Vert$, $\Vert b^{k+1}_L-\overline{b}^{k+1}_L\Vert$, $\Vert z^{k+1}_L-\overline{z}^{k+1}_L\Vert$ and $\Vert z^{k+1}_L-{z}^{k}_L\Vert$.\\
\indent For $b^{k+1}_l(l<L)$,
\begin{align*}
    \nabla _{b^{k+1}_l}L_\rho&=\nabla _{b^{k+1}_l}\phi(\textbf{W}^{k+1},\textbf{b}^{k+1},\textbf{z}^{k+1},\textbf{a}^{k+1},u^{k+1})\\&=\nu(W^{k+1}_{l}a^{k+1}_{l-1}+b^{k+1}_l-z^{k+1}_l)\\&=\nabla_{b^{k+1}_l}\phi(\textbf{W}^{k+1}_l,\textbf{b}^{k+1}_{l},\textbf{z}^{k+1}_{l},\textbf{a}^{k+1}_{l},u^k)\\&=\nabla_{\overline{b}^{k+1}_l}\phi({\textbf{W}}^{k+1}_{l},{\textbf{b}}^{k+1}_{l-1},{\textbf{z}}^{k+1}_{l-1},{\textbf{a}}^{k+1}_{l-1},u^k)+\nu(b^{k+1}_l-\overline{b}^{k+1}_l)-\nu(b^{k+1}_l-\overline{b}^{k+1}_l)-\nabla_{\overline{b}^{k+1}_l}\phi({\textbf{W}}^{k+1}_{l},{\textbf{b}}^{k+1}_{l-1},{\textbf{z}}^{k+1}_{l-1},{\textbf{a}}^{k+1}_{l-1},u^k)\\&+\nabla_{b^{k+1}_l}\phi(\textbf{W}^{k+1}_l,\textbf{b}^{k+1}_{l},\textbf{z}^{k+1}_{l},\textbf{a}^{k+1}_{l},u^k)\\&=\nabla_{\overline{b}^{k+1}_l}\phi({\textbf{W}}^{k+1}_{l},{\textbf{b}}^{k+1}_{l-1},{\textbf{z}}^{k+1}_{l-1},{\textbf{a}}^{k+1}_{l-1},u^k)+\nu(b^{k+1}_l-\overline{b}^{k+1}_l)-\nu(b^{k+1}_l-\overline{b}^{k+1}_l)-\nu({W}^{k+1}_{l}a^{k+1}_{l-1}+\overline{b}^{k+1}_l-\overline{z}^{k+1}_l)\\&+\nu(W^{k+1}_la^{k+1}_{l-1}+b^{k+1}_l-z^{k+1}_l)\\&
    =\nabla_{\overline{b}^{k+1}_l}\phi({\textbf{W}}^{k+1}_{l},{\textbf{b}}^{k+1}_{l-1},{\textbf{z}}^{k+1}_{l-1},{\textbf{a}}^{k+1}_{l-1},u^k)+\nu(b^{k+1}_l-\overline{b}^{k+1}_l)+\nu(\overline{z}^{k+1}_l-z^{k+1}_l)
    \end{align*}
    The optimality condition of Equation \eqref{eq:update b} yields
    \begin{align*}
        0\in \nabla_{\overline{b}^{k+1}_l}\phi({\textbf{W}}^{k+1}_{l},{\textbf{b}}^{k+1}_{l-1},{\textbf{z}}^{k+1}_{l-1},{\textbf{a}}^{k+1}_{l-1},u^k)+\nu(b^{k+1}_l-\overline{b}^{k+1}_l)
    \end{align*}
    Therefore, $\Vert\nabla_{b^{k+1}_l} L_\rho\Vert$ is linearly independent on $\Vert z^{k+1}_l-\overline{z}^{k+1}_l\Vert$. \\
    \indent For $b^{k+1}_{L}$,
    \begin{align*}
    \nabla _{b^{k+1}_L}L_\rho&=\nabla_{b^{k+1}_L}\phi(\textbf{W}^{k+1},\textbf{b}^{k+1},\textbf{z}^{k+1},\textbf{a}^{k+1},u^{k+1})\\&=\nabla_{\overline{b}^{k+1}_L}\phi({\textbf{W}}^{k+1},{\textbf{b}}^{k+1}_{L-1},{\textbf{z}}^{k+1}_{L-1},{\textbf{a}}^{k+1},u^k)+\rho(b^{k+1}_L-\overline{b}^{k+1}_L)-\rho(b^{k+1}_L-\overline{b}^{k+1}_L)-\nabla_{\overline{b}^{k+1}_L}\phi({\textbf{W}}^{k+1},{\textbf{b}}^{k+1}_{L-1},{\textbf{z}}^{k+1}_{L-1},{\textbf{a}}^{k+1},u^k)\\&+\nabla_{b^{k+1}_L}\phi(\textbf{W}^{k+1},\textbf{b}^{k+1},\textbf{z}^{k+1},\textbf{a}^{k+1},u^{k+1})\\&=\nabla_{\overline{b}^{k+1}_L}\phi({\textbf{W}}^{k+1},{\textbf{b}}^{k+1}_{L-1},{\textbf{z}}^{k+1}_{L-1},{\textbf{a}}^{k+1},u^k)+\rho(b^{k+1}_L-\overline{b}^{k+1}_L)-\rho(b^{k+1}_L-\overline{b}^{k+1}_L)-\rho(W^{k+1}_La^{k+1}_L+\overline{b}^{k+1}_L-\overline{z}^{k+1}_L-u^k/\rho)\\&+\rho(W^{k+1}_La^{k+1}_L+{b}^{k+1}_L-{z}^{k+1}_L-u^{k+1}/\rho)\\&=\nabla_{\overline{b}^{k+1}_L}\phi({\textbf{W}}^{k+1},{\textbf{b}}^{k+1}_{L-1},{\textbf{z}}^{k+1}_{L-1},{\textbf{a}}^{k+1},u^k)+\rho(b^{k+1}_L-\overline{b}^{k+1}_L)+\rho(\overline{z}^{k+1}_L-z^{k+1}_L)+u^k-u^{k+1}
    \end{align*}
    Because
    \begin{align*}
        &\Vert\rho(\overline{z}^{k+1}_L-z^{k+1}_L)+u^k-u^{k+1}\Vert\\&\leq \rho\Vert\overline{z}^{k+1}_L-z^{k+1}_L\Vert+\Vert u^k-u^{k+1}\Vert\text{(triangle inequality)}\\&=\rho\Vert\overline{z}^{k+1}_L-z^{k+1}_L\Vert+\Vert \nabla_{z_L^{k}}R(z_L^{k};y)-\nabla_{z_L^{k+1}}R(z_L^{k+1};y)\Vert\text{(Lemma \ref{lemma:z_l optimality})}\\&\leq \rho \Vert\overline{z}^{k+1}_L-z^{k+1}_L\Vert+H\Vert z_L^{k}-z_L^{k+1}\Vert\text{($R(z_L;y)$ is Lipschitz differentiable)}
    \end{align*}
    and the optimality condition of Equation \eqref{eq:update bl} yields
    \begin{align*}
        0\in \nabla_{\overline{b}^{k+1}_L}\phi({\textbf{W}}^{k+1},{\textbf{b}}^{k+1}_{L-1},{\textbf{z}}^{k+1}_{L-1},{\textbf{a}}^{k+1},u^k)+\rho(b^{k+1}_L-\overline{b}^{k+1}_L)
    \end{align*}
    Therefore, $\Vert\nabla_{b^{k+1}_L} L_\rho\Vert$ is upper bounded by a combination of $\Vert z^{k+1}_L-\overline{z}^{k+1}_L\Vert$ and $\Vert z^{k+1}_L-{z}^{k}_L\Vert$. \\
    \indent For $z^{k+1}_l(l< L)$,
    \begin{align*}
        \partial_{z^{k+1}_l} L_\rho&=\partial_{z^{k+1}_l}\phi(\textbf{W}^{k+1},\textbf{b}^{k+1},\textbf{z}^{k+1},\textbf{a}^{k+1},u^{k+1})\\&=\nu(z^{k+1}_l-W^{k+1}_la^{k+1}_{l-1}-b^{k+1}_l)+\nu \partial f_l(z^{k+1}_l)\circ(f(z^{k+1}_l)-a^{k+1}_l)\\&=\partial_{z^{k+1}_l}\phi(\textbf{W}_l^{k+1},\textbf{b}_l^{k+1},\textbf{z}_l^{k+1},\textbf{a}_l^{k+1},u^{k})\\&=\partial_{z^{k+1}_l}\phi(\textbf{W}_l^{k+1},\textbf{b}_l^{k+1},\textbf{z}_l^{k+1},\textbf{a}_l^{k+1},u^{k})-\partial_{z^{k+1}_l}\phi(\textbf{W}_l^{k+1},\textbf{b}_l^{k+1},\textbf{z}_l^{k+1},\textbf{a}_{l-1}^{k+1},u^{k})+\partial_{z^{k+1}_l}\phi(\textbf{W}_l^{k+1},\textbf{b}_l^{k+1},\textbf{z}_l^{k+1},\textbf{a}_{l-1}^{k+1},u^{k})\\&=\nu\partial f_l(z^{k+1}_l)\circ(\overline{a}_l^{k+1}-a^{k+1}_l)+\partial_{z^{k+1}_l}\phi(\textbf{W}_l^{k+1},\textbf{b}_l^{k+1},\textbf{z}_l^{k+1},\textbf{a}_{l-1}^{k+1},u^{k})
    \end{align*}
    because $z^{k+1}_l$ is bounded and $f_l(z_l)$ is continuous and hence $f_l(z^{k+1}_l)$ is bounded, and the optimality condition of Equation \eqref{eq:update z} yields
    \begin{align*}
        0\in \partial_{z^{k+1}_l}\phi(\textbf{W}_l^{k+1},\textbf{b}_l^{k+1},\textbf{z}_l^{k+1},\textbf{a}_{l-1}^{k+1},u^{k})
    \end{align*}
    \indent Therefore, $\Vert \partial_{z^{k+1}_l}L_\rho\Vert$ is upper bounded by $\Vert a^{k+1}_l-\overline{a}^{k+1}_l\Vert$.\\
    \indent For $z^{k+1}_L$,
    \begin{align*}
        \nabla_{z^{k+1}_L}L_\rho&=\nabla_{z^{k+1}_L} R(z^{k+1}_L;y)+\nabla_{z^{k+1}_L}\phi(\textbf{W}^{k+1},\textbf{b}^{k+1},\textbf{z}^{k+1},\textbf{a}^{k+1},u^{k+1})\\&=\nabla_{z^{k+1}_L} R(z^{k+1}_L;y)+\nabla_{z^{k+1}_L}\phi(\textbf{W}^{k+1},\textbf{b}^{k+1},\textbf{z}^{k+1},\textbf{a}^{k+1},u^{k})-\nabla_{z^{k+1}_L}\phi(\textbf{W}^{k+1},\textbf{b}^{k+1},\textbf{z}^{k+1},\textbf{a}^{k+1},u^{k})\\&+\nabla_{z^{k+1}_L}\phi(\textbf{W}^{k+1},\textbf{b}^{k+1},\textbf{z}^{k+1},\textbf{a}^{k+1},u^{k+1})\\&=\nabla_{z^{k+1}_L} R(z^{k+1}_L;y)+\nabla_{z^{k+1}_L}\phi(\textbf{W}^{k+1},\textbf{b}^{k+1},\textbf{z}^{k+1},\textbf{a}^{k+1},u^{k})+u^{k+1}-u^k
    \end{align*}
    The optimality condition of Equation \eqref{eq:update zl} yields
    \begin{align*}
        0\in\nabla_{z^{k+1}_L} R(z^{k+1}_L;y)+\nabla_{z^{k+1}_L}\phi(\textbf{W}^{k+1},\textbf{b}^{k+1},\textbf{z}^{k+1},\textbf{a}^{k+1},u^{k})
    \end{align*}
     and
     \begin{align*}
             &\Vert u^{k+1}-u^k\Vert= \Vert \nabla_{z^{k+1}_L}R(z^{k+1}_L;y)-\nabla_{z^{k}_L}R(z^{k}_L;y)\Vert \ \text{(Lemma \ref{lemma:z_l optimality})}\leq H \Vert z^{k+1}_L-z^k_L\Vert\\&\text{(Cauchy-Schwarz inequality, $R(z_L;y)$ is Lipschitz differentiable)}
     \end{align*}
 Therefore $\Vert\nabla_{z^{k+1}_L}L_\rho\Vert$ is upper bounded by $\Vert z^{k+1}_L-z^{k}_L\Vert$.\\
\indent For $a^{k+1}_l(l<L-1)$,
\begin{align*}
    &\nabla_{a^{k+1}_l}\phi(\textbf{W}^{k+1},\textbf{b}^{k+1},\textbf{z}^{k+1},\textbf{a}^{k+1},u^{k+1})\\&=\nu (W^{k+1}_{l+1})^T(W^{k+1}_{l+1}a^{k+1}_l+b^{k+1}_{l+1}-z^{k+1}_{l+1})+\nu(a^{k+1}_l-f_l(z^{k+1}_l))\\&=\nabla_{a^{k+1}_l}\phi(\textbf{W}_{l+1}^{k+1},\textbf{b}^{k+1}_{l+1},\textbf{z}^{k+1}_{l+1},\textbf{a}^{k+1}_{l},u^{k})\\&=\tau^{k+1}_l\circ(a^{k+1}_l-\overline{a}^{k+1}_l)+\nabla_{\overline{a}^{k+1}_l}\phi(\textbf{W}_{l}^{k+1},\textbf{b}^{k+1}_{l},\textbf{z}^{k+1}_{l},\textbf{a}^{k+1}_{l-1},u^{k})-\tau^{k+1}_l\circ(a^{k+1}_l-\overline{a}^{k+1}_l)-\nabla_{\overline{a}^{k+1}_l}\phi(\textbf{W}_{l}^{k+1},\textbf{b}^{k+1}_{l},\textbf{z}^{k+1}_{l},\textbf{a}^{k+1}_{l-1},u^{k})\\&+\nabla_{a^{k+1}_l}\phi(\textbf{W}_{l+1}^{k+1},\textbf{b}^{k+1}_{l+1},\textbf{z}^{k+1}_{l+1},\textbf{a}^{k+1}_{l},u^{k})\\&=\tau^{k+1}_l\circ(a^{k+1}_l-\overline{a}^{k+1}_l)+\nabla_{\overline{a}^{k+1}_l}\phi(\textbf{W}_{l}^{k+1},\textbf{b}^{k+1}_{l},\textbf{z}^{k+1}_{l},\textbf{a}^{k+1}_{l-1},u^{k})-\tau^{k+1}_l\circ(a^{k+1}_l-\overline{a}^{k+1}_l)-\nu (\overline{W}^{k+1}_{l+1})^T(\overline{W}^{k+1}_{l+1}\overline{a}^{k+1}_l+\overline{b}^{k+1}_{l+1}-\overline{z}^{k+1}_{l+1})\\&-\nu(\overline{a}^{k+1}_l-f_l(z^{k+1}_l))+\nu (W^{k+1}_{l+1})^T(W^{k+1}_{l+1}a^{k+1}_l+b^{k+1}_{l+1}-z^{k+1}_{l+1})+\nu(a^{k+1}_l-f_l(z^{k+1}_l))
\end{align*}
Because
\begin{align*}
    &\Vert-\tau^{k+1}_l\circ(a^{k+1}_l-\overline{a}^{k+1}_l)-\nu (\overline{W}^{k+1}_{l+1})^T(\overline{W}^{k+1}_{l+1}\overline{a}^{k+1}_l+\overline{b}^{k+1}_{l+1}-\overline{z}^{k+1}_{l+1})-\nu(\overline{a}^{k+1}_l-f_l(z^{k+1}_l))\\&+\nu (W^{k+1}_{l+1})^T(W^{k+1}_{l+1}a^{k+1}_l+b^{k+1}_{l+1}-z^{k+1}_{l+1})+\nu(a^{k+1}_l-f_l(z^{k+1}_l))\Vert\\&\leq \Vert\tau^{k+1}_l\circ(a^{k+1}_l-\overline{a}^{k+1}_l)\Vert+\nu\Vert a^{k+1}_l-\overline{a}^{k+1}_l\Vert+\nu\Vert (W^{k+1}_{l+1})^Tb^{k+1}_{l+1}-(\overline{W}^{k+1}_{l+1})^T\overline{b}^{k+1}_{l+1} \Vert\\&+\nu\Vert (W^{k+1}_{l+1})^Tz^{k+1}_{l+1}-(\overline{W}^{k+1}_{l+1})^T\overline{z}^{k+1}_{l+1}\Vert +\nu\Vert (W^{k+1}_{l+1})^TW^{k+1}_{l+1}a^{k+1}_l-(\overline{W}^{k+1}_{l+1})^T\overline{W}^{k+1}_{l+1}\overline{a}^{k+1}_l\Vert\text{(triangle inequality)}
\end{align*}
Then we need to show that $\nu\Vert (W^{k+1}_{l+1})^Tb^{k+1}_{l+1}-(\overline{W}^{k+1}_{l+1})^T\overline{b}^{k+1}_{l+1} \Vert$, $\nu\Vert (W^{k+1}_{l+1})^Tz^{k+1}_{l+1}-(\overline{W}^{k+1}_{l+1})^T\overline{z}^{k+1}_{l+1} \Vert$, and $\nu\Vert (W^{k+1}_{l+1})^TW^{k+1}_{l+1}a^{k+1}_l-(\overline{W}^{k+1}_{l+1})^T\overline{W}^{k+1}_{l+1}\overline{a}^{k+1}_l\Vert$ are upper bounded by $\Vert \textbf{W}^{k+1}-\overline{\textbf{W}}^{k+1}\Vert$,  $\Vert\textbf{b}^{k+1}-\overline{\textbf{b}}^{k+1}\Vert$, $\Vert\textbf{z}^{k+1}-\overline{\textbf{z}}^{k+1}\Vert$, and $\Vert\textbf{a}^{k+1}-\overline{\textbf{a}}^{k+1}\Vert$.
\begin{align*}
    &\nu\Vert (W^{k+1}_{l+1})^Tb^{k+1}_{l+1}-(\overline{W}^{k+1}_{l+1})^T\overline{b}^{k+1}_{l+1} \Vert\\&=\nu\Vert (W^{k+1}_{l+1})^Tb^{k+1}_{l+1}-(\overline{W}^{k+1}_{l+1})^Tb^{k+1}_{l+1}+(\overline{W}^{k+1}_{l+1})^Tb^{k+1}_{l+1}-(\overline{W}^{k+1}_{l+1})^T\overline{b}^{k+1}_{l+1} \Vert\\&\leq \nu\Vert b^{k+1}_{l+1}\Vert\Vert W^{k+1}_{l+1}-\overline{W}^{k+1}_{l+1}\Vert+\nu\Vert \overline{W}^{k+1}_{l+1}\Vert\Vert b^{k+1}_{l+1}-\overline{b}^{k+1}_{l+1}\Vert\\&\text{(triangle inequality, Cauchy-Schwarz inequality)}
\end{align*}
Because $\Vert b^{k+1}_{l+1}\Vert$ and $\Vert\overline{W}^{k+1}_{l+1}\Vert$ are upper bounded, $\nu\Vert (W^{k+1}_{l+1})^Tb^{k+1}_{l+1}-(\overline{W}^{k+1}_{l+1})^T\overline{b}^{k+1}_{l+1} \Vert$ is therefore upper bounded by a combination of $\Vert W^{k+1}_{l+1}-\overline{W}^{k+1}_{l+1}\Vert$ and $\Vert b^{k+1}_{l+1}-\overline{b}^{k+1}_{l+1}\Vert$.\\
Similarly, $\nu\Vert (W^{k+1}_{l+1})^Tz^{k+1}_{l+1}-(\overline{W}^{k+1}_{l+1})^T\overline{z}^{k+1}_{l+1} \Vert$ is  upper bounded by a combination of $\Vert W^{k+1}_{l+1}-\overline{W}^{k+1}_{l+1}\Vert$ and $\Vert z^{k+1}_{l+1}-\overline{z}^{k+1}_{l+1}\Vert$.\\
\begin{align*}
    &\nu\Vert (W^{k+1}_{l+1})^TW^{k+1}_{l+1}a^{k+1}_l-(\overline{W}^{k+1}_{l+1})^T\overline{W}^{k+1}_{l+1}\overline{a}^{k+1}_l\Vert\\&=\nu\Vert(W^{k+1}_{l+1})^TW^{k+1}_{l+1}a^{k+1}_l-(W^{k+1}_{l+1})^TW^{k+1}_{l+1}\overline{a}^{k+1}_l+(W^{k+1}_{l+1})^TW^{k+1}_{l+1}\overline{a}^{k+1}_l-(W^{k+1}_{l+1})^T\overline{W}^{k+1}_{l+1}\overline{a}^{k+1}_l+(W^{k+1}_{l+1})^T\overline{W}^{k+1}_{l+1}\overline{a}^{k+1}_l-(\overline{W}^{k+1}_{l+1})^T\overline{W}^{k+1}_{l+1}\overline{a}^{k+1}_l\Vert\\&\leq \nu\Vert(W^{k+1}_{l+1})^TW^{k+1}_{l+1}(a^{k+1}_l-\overline{a}^{k+1}_l)\Vert+\nu\Vert (W^{k+1}_{l+1})^T(W^{k+1}_{l+1}-\overline{W}^{k+1}_{l+1})\overline{a}^{k+1}_l\Vert+\nu\Vert(W^{k+1}_{l+1}-\overline{W}^{k+1}_{l+1})^T\overline{W}^{k+1}_{l+1}\overline{a}^{k+1}_l\Vert\text{(triangle inequality)}\\&\leq \nu\Vert W^{k+1}_{l+1}\Vert\Vert W^{k+1}_{l+1}\Vert\Vert a^{k+1}_l-\overline{a}^{k+1}_l\Vert+\nu\Vert W^{k+1}_{l+1}\Vert\Vert W^{k+1}_{l+1}-\overline{W}^{k+1}_{l+1}\Vert\Vert\overline{a}^{k+1}_l\Vert+\nu\Vert W^{k+1}_{l+1}-\overline{W}^{k+1}_{l+1}\Vert\Vert\overline{W}^{k+1}_{l+1}\Vert\Vert\overline{a}^{k+1}_l\Vert\text{(Cauchy-Schwarz inequality)}
\end{align*}
Because $\Vert W^{k+1}_{l+1}\Vert$, $\Vert \overline{W}^{k+1}_{l+1}\Vert$  and $\Vert \overline{a}^{k+1}_l\Vert$ are upper bounded, $\nu\Vert (W^{k+1}_{l+1})^TW^{k+1}_{l+1}a^{k+1}_l-(\overline{W}^{k+1}_{l+1})^T\overline{W}^{k+1}_{l+1}\overline{a}^{k+1}_l\Vert$ is therefore upper bounded by a combination of $\Vert W^{k+1}_{l+1}-\overline{W}^{k+1}_{l+1}\Vert$ and $\Vert a^{k+1}_{l+1}-\overline{a}^{k+1}_{l+1}\Vert$.\\
\indent For $a^{k+1}_{L-1}$, 
\begin{align*}
    &\nabla_{a^{k+1}_{L-1}}\phi(\textbf{W}^{k+1},\textbf{b}^{k+1},\textbf{z}^{k+1},\textbf{a}^{k+1},u^{k+1})\\&=\tau^{k+1}_{L-1}\circ(a^{k+1}_{L-1}-\overline{a}^{k+1}_{L-1})+\nabla_{\overline{a}^{k+1}_{L-1}}\phi(\textbf{W}_{L-1}^{k+1},\textbf{b}^{k+1}_{L-1},\textbf{z}^{k+1}_{L-1},\textbf{a}^{k+1}_{L-2},u^{k})-\tau^{k+1}_{L-1}\circ(a^{k+1}_{L-1}-\overline{a}^{k+1}_{L-1})-\nabla_{\overline{a}^{k+1}_{L-1}}\phi(\textbf{W}_{L-1}^{k+1},\textbf{b}^{k+1}_{L-1},\textbf{z}^{k+1}_{L-1},\textbf{a}^{k+1}_{L-2},u^{k})\\&+\nabla_{a^{k+1}_{L-1}}\phi(\textbf{W}^{k+1},\textbf{b}^{k+1},\textbf{z}^{k+1},\textbf{a}^{k+1},u^{k+1})\\&=\tau^{k+1}_{L-1}\circ(a^{k+1}_{L-1}-\overline{a}^{k+1}_{L-1})+\nabla_{\overline{a}^{k+1}_{L-1}}\phi(\textbf{W}_{L-1}^{k+1},\textbf{b}^{k+1}_{L-1},\textbf{z}^{k+1}_{L-1},\textbf{a}^{k+1}_{L-2},u^{k})-\tau^{k+1}_{L-1}\circ(a^{k+1}_{L-1}-\overline{a}^{k+1}_{L-1})-\rho (\overline{W}^{k+1}_{L})^T(\overline{W}^{k+1}_{L}\overline{a}^{k+1}_{L-1}+\overline{b}^{k+1}_{L}-\overline{z}^{k+1}_{L}-u^k/\rho)\\&-\nu(\overline{a}^{k+1}_{L-1}-f_{L-1}(z^{k+1}_{L-1}))+\rho (W^{k+1}_{L})^T(W^{k+1}_{L}a^{k+1}_{L-1}+b^{k+1}_{L}-z^{k+1}_{L}-u^{k+1}/\rho)+\nu(a^{k+1}_{L-1}-f_{L-1}(z^{k+1}_{L-1}))
\end{align*}
Because
\begin{align*}
    &\Vert-\tau^{k+1}_{L-1}\circ(a^{k+1}_{L-1}-\overline{a}^{k+1}_{L-1})-\rho (\overline{W}^{k+1}_{L})^T(\overline{W}^{k+1}_{L}\overline{a}^{k+1}_{L-1}+\overline{b}^{k+1}_{L}-\overline{z}^{k+1}_{L}-u^k/\rho)-\nu(\overline{a}^{k+1}_{L-1}-f_{L-1}(z^{k+1}_{L-1}))\\&+\rho (W^{k+1}_{L})^T(W^{k+1}_{L}a^{k+1}_{L-1}+b^{k+1}_{L}-z^{k+1}_{L}-u^{k+1}/\rho)+\nu(a^{k+1}_{L-1}-f_{L-1}(z^{k+1}_{L-1}))\Vert\\&\leq \Vert\tau^{k+1}_{L-1}\circ(a^{k+1}_{L-1}-\overline{a}^{k+1}_{L-1})\Vert+\nu\Vert a^{k+1}_{L-1}-\overline{a}^{k+1}_{L-1}\Vert+\rho\Vert (W^{k+1}_{L})^Tb^{k+1}_{L}-(\overline{W}^{k+1}_{L})^T\overline{b}^{k+1}_{L} \Vert\\&+\rho\Vert (W^{k+1}_{L})^Tz^{k+1}_{L}-(\overline{W}^{k+1}_{L})^T\overline{z}^{k+1}_{L}\Vert+\Vert (W^{k+1}_{L})^Tu^{k+1}-(\overline{W}^{k+1}_{L})^T{u}^{k}\Vert \\&+\rho\Vert (W^{k+1}_{L})^TW^{k+1}_{L}a^{k+1}_{L-1}-(\overline{W}^{k+1}_{L})^T\overline{W}^{k+1}_{L}\overline{a}^{k+1}_{L-1}\Vert\text{(triangle inequality)}
\end{align*}
Then we need to show that $\rho\Vert (W^{k+1}_{L})^Tb^{k+1}_{L}-(\overline{W}^{k+1}_{L})^T\overline{b}^{k+1}_{L} \Vert$, $\rho\Vert (W^{k+1}_{L})^Tz^{k+1}_{L}-(\overline{W}^{k+1}_{L})^T\overline{z}^{k+1}_{L}\Vert$, $\Vert (W^{k+1}_{L})^Tu^{k+1}-(\overline{W}^{k+1}_{L})^T{u}^{k}\Vert$ and $\rho\Vert (W^{k+1}_{L})^TW^{k+1}_{L}a^{k+1}_{L-1}-(\overline{W}^{k+1}_{L})^T\overline{W}^{k+1}_{L}\overline{a}^{k+1}_{L-1}\Vert$  are upper bounded by $\Vert \textbf{W}^{k+1}-\overline{\textbf{W}}^{k+1}\Vert$,  $\Vert\textbf{b}^{k+1}-\overline{\textbf{b}}^{k+1}\Vert$, $\Vert\textbf{z}^{k+1}-\overline{\textbf{z}}^{k+1}\Vert$, $\Vert\textbf{a}^{k+1}-\overline{\textbf{a}}^{k+1}\Vert$ and $\Vert\textbf{z}^{k+1}-\textbf{z}^{k}\Vert$.
\begin{align*}
    &\rho\Vert (W^{k+1}_{L})^Tb^{k+1}_{L}-(\overline{W}^{k+1}_{L})^T\overline{b}^{k+1}_{L} \Vert\\&=\rho\Vert (W^{k+1}_{L})^Tb^{k+1}_{L}-(\overline{W}^{k+1}_{L})^Tb^{k+1}_{L}+(\overline{W}^{k+1}_{L})^Tb^{k+1}_{L}-(\overline{W}^{k+1}_{L})^T\overline{b}^{k+1}_{L} \Vert\\&\leq \rho\Vert b^{k+1}_{L}\Vert\Vert W^{k+1}_{L}-\overline{W}^{k+1}_{L}\Vert+\rho\Vert \overline{W}^{k+1}_{L}\Vert\Vert b^{k+1}_{L}-\overline{b}^{k+1}_{L}\Vert\\&\text{(triangle inequality, Cauchy-Schwarz inequality)}
\end{align*}
Because $\Vert b^{k+1}_{L}\Vert$ and $\Vert\overline{W}^{k+1}_{L}\Vert$ are upper bounded, $\rho\Vert (W^{k+1}_{L})^Tb^{k+1}_{L}-(\overline{W}^{k+1}_{L})^T\overline{b}^{k+1}_{L} \Vert$ is therefore upper bounded by a combination of $\Vert W^{k+1}_{L}-\overline{W}^{k+1}_{L}\Vert$ and $\Vert b^{k+1}_{L}-\overline{b}^{k+1}_{L}\Vert$.\\
Similarly, $\rho\Vert (W^{k+1}_{L})^Tz^{k+1}_{L}-(\overline{W}^{k+1}_{L})^T\overline{z}^{k+1}_{L} \Vert$ is  upper bounded by a combination of $\Vert W^{k+1}_{L}-\overline{W}^{k+1}_{L}\Vert$ and $\Vert z^{k+1}_{L}-\overline{z}^{k+1}_{L}\Vert$.\\
\begin{align*}
    &\Vert (W^{k+1}_L)^Tu^{k+1}-(\overline{W}^{k+1}_L)^Tu^{k}\Vert\\&=\Vert(W^{k+1}_L)^Tu^{k+1}-(W^{k+1}_L)^Tu^{k}+(W^{k+1}_L)^Tu^{k}-(\overline{W}^{k+1}_L)^Tu^{k}\Vert\\&\leq \Vert(W^{k+1}_L)^T(u^{k+1}-u^{k})\Vert+\Vert(W^{k+1}_L-\overline{W}^{k+1}_L)^Tu^{k}\Vert\text{(triangle inequality)}\\&\leq\Vert W^{k+1}_L\Vert\Vert u^{k+1}-u^{k}\Vert+\Vert W^{k+1}_L-\overline{W}^{k+1}_L\Vert \Vert u^{k}\Vert\text{(Cauthy-Schwarz inequality)}\\&=\Vert W^{k+1}_L\Vert\Vert \nabla_{z^{k+1}_L} R(z^{k+1}_L;y)-\nabla_{z^{k}_L}R(z^{k}_L;y)\Vert+\Vert W^{k+1}_L-\overline{W}^{k+1}_L\Vert \Vert u^{k}\Vert \ \text{(Lemma \ref{lemma:z_l optimality})}\\&\leq H\Vert W^{k+1}_L\Vert\Vert z^{k+1}_L-z^{k}_L\Vert+\Vert W^{k+1}_L-\overline{W}^{k+1}_L\Vert\Vert u^{k}\Vert\\&\text{($R(z_L;y)$ is Lipschitz differentiable)}
\end{align*}
Because $\Vert W^{k+1}_L\Vert$ and $\Vert u^k\Vert$ are bounded, $\Vert (W^{k+1}_L)^Tu^{k+1}-(\overline{W}^{k+1}_L)^Tu^{k}\Vert$ is upper bounded by a combination of $\Vert z^{k+1}_L-z^{k}_L\Vert$ and $\Vert W^{k+1}_L-\overline{W}^{k+1}_L\Vert$.
\begin{align*}
    &\rho\Vert (W^{k+1}_{L})^TW^{k+1}_{L}a^{k+1}_{L-1}-(\overline{W}^{k+1}_{L})^T\overline{W}^{k+1}_{L}\overline{a}^{k+1}_{L-1}\Vert\\&=\rho\Vert(W^{k+1}_{L})^TW^{k+1}_{L}a^{k+1}_{L-1}-(W^{k+1}_{L})^TW^{k+1}_{L}\overline{a}^{k+1}_{L-1}+(W^{k+1}_{L})^TW^{k+1}_{L}\overline{a}^{k+1}_{L-1}\\&-(W^{k+1}_{L})^T\overline{W}^{k+1}_{L}\overline{a}^{k+1}_{L-1}+(W^{k+1}_{L})^T\overline{W}^{k+1}_{L}\overline{a}^{k+1}_{L-1}-(\overline{W}^{k+1}_{L})^T\overline{W}^{k+1}_{L}\overline{a}^{k+1}_{L-1}\Vert\\&\leq \rho\Vert(W^{k+1}_{L})^TW^{k+1}_{L}(a^{k+1}_{L-1}-\overline{a}^{k+1}_{L-1})\Vert+\rho\Vert (W^{k+1}_{L})^T(W^{k+1}_{L}-\overline{W}^{k+1}_{L})\overline{a}^{k+1}_{L-1}\Vert+\rho\Vert(W^{k+1}_{L}-\overline{W}^{k+1}_{L})^T\overline{W}^{k+1}_{L}\overline{a}^{k+1}_{L-1}\Vert\text{(triangle inequality)}\\&\leq \rho\Vert W^{k+1}_{L}\Vert\Vert W^{k+1}_{L}\Vert\Vert a^{k+1}_{L-1}-\overline{a}^{k+1}_{L-1}\Vert+\rho\Vert W^{k+1}_{L}\Vert\Vert W^{k+1}_{L}-\overline{W}^{k+1}_{L}\Vert\Vert\overline{a}^{k+1}_{L-1}\Vert+\rho\Vert W^{k+1}_{L}-\overline{W}^{k+1}_{L}\Vert\Vert\overline{W}^{k+1}_{L}\Vert\Vert\overline{a}^{k+1}_{L-1}\Vert\text{(Cauchy-Schwarz inequality)}
\end{align*}
Because $\Vert W^{k+1}_{L}\Vert$, $\Vert \overline{W}^{k+1}_{L}\Vert$  and $\Vert \overline{a}^{k+1}_{L-1}\Vert$ are upper bounded, $\rho\Vert (W^{k+1}_{L})^TW^{k+1}_{L}a^{k+1}_{L-1}-(\overline{W}^{k+1}_{L})^T\overline{W}^{k+1}_{L}\overline{a}^{k+1}_{L-1}\Vert$ is therefore upper bounded by a combination of $\Vert W^{k+1}_{L}-\overline{W}^{k+1}_{L}\Vert$ and $\Vert a^{k+1}_{L}-\overline{a}^{k+1}_{L}\Vert$.\\
\indent For $u^{k+1}$,
\begin{align*}
    \nabla_{u^{k+1}_l} L_\rho&=\nabla_{u^{k+1}_l}\phi(\textbf{W}^{k+1},\textbf{b}^{k+1},\textbf{z}^{k+1},\textbf{a}^{k+1},{u}^{k+1})\\&=z^{k+1}_L-W^{k+1}_La^{k+1}_L-b^{k+1}_L\\&=(1/\rho) (u^{k+1}-u^k)\\&=(1/\rho)(\nabla_{z^{k}_L}R(z^{k}_L;y)-\nabla_{z^{k+1}_L}R(z^{k+1}_L;y))\text{(Lemma \ref{lemma:z_l optimality})}
\end{align*}
Because 
\begin{align*}
    &\Vert(1/\rho)(\nabla_{z^{k}_L}R(z^{k}_L;y)-\nabla_{z^{k+1}_L}R(z^{k+1}_L;y))\Vert\leq (H/\rho)\Vert z^{k+1}_L-z^{k}_L\Vert\text{($R(z_L;y)$ is Lipschitz differentiable)}
\end{align*}
Therefore, $\Vert \nabla_{u^{k+1}_l} L_\rho\Vert$ is upper bounded by $\Vert z^{k+1}_L-z^{k}_L\Vert$.
\end{proof}
\section*{Proof of Theorem \ref{thero: theorem 3}}
\label{sec:convergence rate}
\begin{proof}
To prove this theorem, we will first show that $c_k$ satisfies two conditions: (1). $c_k\geq c_{k+1}$. (2). $\sum\nolimits_{k=0}^\infty c_k$ is bounded.  We then conclude the convergence rate of $o(1/k)$ based on these two conditions. Specifically, first, we have
\begin{align*}
    c_k&=\min\nolimits_{0\leq i\leq k}(\sum\nolimits_{l=1}^L(\Vert \overline{W}_l^{i+1}-W_l^i\Vert^2_2+\Vert {W}_l^{i+1}-\overline{W}_l^{i+1}\Vert^2_2+\Vert \overline{b}_l^{i+1}-b_l^i\Vert^2_2+\Vert {b}_l^{i+1}-\overline{b}_l^{i+1}\Vert^2_2) +\sum\nolimits_{l=1}^{L-1}(\Vert \overline{a}_l^{i+1}-a_l^i\Vert^2_2+\Vert {a}_l^{i+1}-\overline{a}_l^{i+1}\Vert^2_2)\\&+\Vert \overline{z}^{i+1}_L-{z}^{i}_L\Vert^2_2+\Vert z^{i+1}_L-\overline{z}^{i+1}_L\Vert^2_2) \\&\geq\min\nolimits_{0\leq i\leq k+1}(\sum\nolimits_{l=1}^L\Vert \overline{W}_l^{i+1}-W_l^i\Vert^2_2+\Vert {W}_l^{i+1}-\overline{W}_l^{i+1}\Vert^2_2+\Vert \overline{b}_l^{i+1}-b_l^i\Vert^2_2+\Vert {b}_l^{i+1}-\overline{b}_l^{i+1}\Vert^2_2) +\sum\nolimits_{l=1}^{L-1}(\Vert \overline{a}_l^{i+1}-a_l^i\Vert^2_2+\Vert {a}_l^{i+1}-\overline{a}_l^{i+1}\Vert^2_2)\\&+\Vert \overline{z}^{i+1}_L-{z}^{i}_L\Vert^2_2+\Vert z^{i+1}_L-\overline{z}^{i+1}_L\Vert^2_2)\\&= c_{k+1}
\end{align*}
Therefore $c_k$ satisfies the first condition. Second,
\begin{align*}
    &\sum\nolimits_{k=0}^\infty c_k\\&=\sum\nolimits_{k=0}^\infty\min\nolimits_{0\leq i\leq k}(\sum\nolimits_{l=1}^L(\Vert \overline{W}_l^{i+1}-W_l^i\Vert^2_2+\Vert {W}_l^{i+1}-\overline{W}_l^{i+1}\Vert^2_2+\Vert \overline{b}_l^{i+1}-b_l^i\Vert^2_2+\Vert {b}_l^{i+1}-\overline{b}_l^{i+1}\Vert^2_2) +\sum\nolimits_{l=1}^{L-1}(\Vert \overline{a}_l^{i+1}-a_l^i\Vert^2_2+\Vert {a}_l^{i+1}-\overline{a}_l^{i+1}\Vert^2_2)\\&+\Vert \overline{z}^{i+1}_L-{z}^{i}_L\Vert^2_2+\Vert z^{i+1}_L-\overline{z}^{i+1}_L\Vert^2_2)\\&\leq \sum\nolimits_{k=0}^\infty(\sum\nolimits_{l=1}^L(\Vert \overline{W}_l^{k+1}-W_l^k\Vert^2_2+\Vert {W}_l^{k+1}-\overline{W}_l^{k+1}\Vert^2_2+\Vert \overline{b}_l^{k+1}-b_l^k\Vert^2_2+\Vert {b}_l^{k+1}-\overline{b}_l^{k+1}\Vert^2_2) +\sum\nolimits_{l=1}^{L-1}(\Vert \overline{a}_l^{k+1}-a_l^k\Vert^2_2+\Vert {a}_l^{k+1}-\overline{a}_l^{k+1}\Vert^2_2)\\&+\Vert \overline{z}^{k+1}_L-{z}^{k}_L\Vert^2_2+\Vert z^{k+1}_L-\overline{z}^{k+1}_L\Vert^2_2) \\&\leq (L_\rho(\textbf{W}^0,\textbf{b}^0,\textbf{z}^0,\textbf{a}^0,u^0)-L_\rho(\textbf{W}^*,\textbf{b}^*,\textbf{z}^*,\textbf{a}^*,u^*))/C_3\text{(Property \ref{pro:property 2})}
\end{align*}
So $\sum\nolimits_{k=0}^\infty c_k$ is bounded and $c_{k}$ satisfies the second condition. Finally, it has been proved that the sufficient conditions of convergence rate $o(1/k)$ are: (1) $c_k\geq c_{k+1}$, and (2) $\sum\nolimits_{k=0}^\infty c_k$ is bounded, and (3) $c_k\geq0$ (Lemma 1.2 in \cite{deng2017parallel}). Since we have proved the first two conditions and the third one $c_k \geq 0$ is obvious, the convergence rate of $o(1/k)$ is proven. 
\end{proof}
\end{appendix}

\end{document}